\documentclass[twoside,11pt]{article}

\usepackage{jmlr2e}

\usepackage{algorithm}
\usepackage{algorithmic}
\usepackage{wrapfig}
\usepackage[caption=false]{subfig}
\usepackage{amsmath}
\usepackage{amssymb}
\usepackage{hyperref}
\usepackage{qtree}
\usepackage{booktabs}
\usepackage{multirow}
\usepackage{bm}
\usepackage{tabularx}
\usepackage{enumerate}   
\usepackage{enumitem}
\usepackage[textsize=tiny]{todonotes}
\usepackage{blindtext}
\usepackage{makecell}

\usepackage{arydshln}
\usepackage{mathtools}
\mathtoolsset{showonlyrefs}

\usepackage{Definitions}

\jmlrheading{1}{2019}{-}{0/00}{00/00}{R. Fathony, K. Asif, A. Liu, M. Bashiri, W. Xing, S. Behpour, X. Zhang, and B. D. Ziebart}

\ShortHeadings{Consistent Robust Adversarial Prediction for General Multiclass Classification}{Fathony, Asif, Liu, Bashiri, Xing, Behpour, Zhang, and Ziebart}
\firstpageno{1}

\begin{document}

\title{Consistent Robust Adversarial Prediction \\for General Multiclass Classification}

\author{\name Rizal Fathony \email rfathony@cs.cmu.edu\\ 
\addr School of Computer Science, Carnegie Mellon University
\AND 
\name Kaiser Asif \email kasif@conversantmedia.com\\ 
\addr Conversant LLC
\AND 
\name Anqi Liu \email anqiliu@caltech.edu\\ 
\addr Department of Computing and Mathematical Sciences, California Institute of Technology
\AND 
\name Mohammad Ali Bashiri \email mbashi4@uic.edu\\ 
\addr Department of Computer Science, University of Illinois at Chicago
\AND 
\name Wei Xing \email wxing3@uic.edu\\ 
\addr Department of Computer Science, University of Illinois at Chicago
\AND 
\name Sima Behpour \email sbehpour@seas.upenn.edu\\ 
\addr Department of Computer and Information Science, University of Pennsylvania
\AND 
\name Xinhua Zhang \email zhangx@uic.edu\\ 
\addr Department of Computer Science, University of Illinois at Chicago
\AND 
\name Brian D. Ziebart \email bziebart@uic.edu\\ 
\addr Department of Computer Science, University of Illinois at Chicago
}

\editor{}

\maketitle

\begin{abstract}%   <- trailing '%' for backward compatibility of .sty file
We propose a \emph{robust adversarial prediction framework} for general multiclass classification.
Our method seeks predictive distributions that robustly optimize non-convex and non-continuous multiclass loss metrics against the worst-case conditional label distributions (the adversarial distributions) that (approximately) match 
the statistics of the training data.
Although the optimized loss metrics are non-convex and non-continuous, the dual formulation of the framework is a convex optimization problem that can be recast as 
a risk minimization model with a prescribed convex surrogate loss we call \emph{the adversarial surrogate loss}.
We show that the adversarial surrogate losses 
fill an existing gap in surrogate loss construction for general multiclass classification problems, by simultaneously aligning better with the original multiclass loss, guaranteeing Fisher consistency, enabling a way to incorporate rich feature spaces via the kernel trick, and providing competitive performance in practice.
\end{abstract}


\section{Introduction}

Multiclass classification is a canonical machine learning task in which a predictor chooses a predicted label from a finite number of possible class labels. For many application domains, the penalty for 
making an incorrect prediction is defined by a loss function that depends on the value of the predicted label and the true label. 
Some example of the task are the
zero-one loss classification where the predictor suffers a loss of one when making incorrect prediction and zero otherwise as well as the ordinal classification (also known as ordinal regression) where the predictor suffers a loss that increases as the prediction moves away from the true label.

Empirical risk minimization (ERM) \citep{vapnik1992principles} is a standard approach for solving general multiclass classification problems by finding the classifier that minimizes a loss metric over the training data. However, since directly minimizing this loss over training data within the
ERM framework is generally 
NP-hard \citep{Steinwart2008SVM}, convex surrogate losses that can be efficiently optimized are
employed to approximate the loss.
Constructing surrogate losses for binary classification has been well studied, resulting in surrogate losses that enjoy desirable theoretical properties and good performance in practice. Among the popular examples are the logarithmic loss, which is minimized by the logistic regression classifier 
\citep{mccullagh1989generalized}, and the hinge loss, which is minimized by the support 
vector machine (SVM) 
\citep{boser1992training,cortes1995support}. 
Both of these surrogate losses are Fisher consistent \citep{lin2002support,bartlett05} for binary 
classification, 
meaning they minimize the zero-one loss and yield the Bayes optimal decision 
when they learn from any true distribution of data using a sufficiently rich feature 
representation.
SVMs provide the additional advantage that 
when combined with kernel methods, extremely rich feature representations
can be efficiently incorporated.

Unfortunately, generalizing the hinge loss to 
multiclass
classification tasks with more than two labels
in a theoretically-sound manner
is challenging.
In the case of multiclass zero-one loss for example,
existing extensions of the hinge loss to multiclass convex surrogates 
\citep{crammer2002algorithmic,weston1999support,lee2004multicategory}
tend to lose their Fisher consistency guarantees 
\citep{tewari2007consistency,liu2007fisher} or produce low 
accuracy predictions in practice \citep{dogan2016unified}.
In the case of multiclass ordinal classification,
surrogate losses are usually constructed by transforming the binary hinge loss to take into account the different penalties of the ordinal regression problem using  thresholding methods \citep{shashua2003ranking,chu2005new,lin2006large,rennie2005lossfunctions,li2007ordinal}, or sample re-weighting methods \citep{li2007ordinal}.
Many methods for other general multiclass problems also rely on similar transformations of the binary hinge loss to construct convex surrogates \citep{binder2012taxonomies,ramaswamy2018consistent,lin2014reduction}. 
Empirical evaluations have compared the appropriateness of different 
surrogate losses for general multiclass classification,
but these still leave
the possibility of undiscovered surrogates 
that align better with the original multiclass classification loss.

To address these limitations, we propose a \emph{robust adversarial prediction framework} that seeks the most robust \citep{grunwald2004game,delage2010distributionally} prediction distribution
that minimizes the loss metric in the worst-case given statistical summaries of the empirical distributions. 
We replace the empirical training data for evaluating our predictor with an adversary that is free to choose an evaluating distribution from the set of distributions that (approximately) match the statistical summaries of empirical training data via moment matching constraints of the features.
Although the optimized loss metrics are non-convex and non-continuous, we show that the dual formulation of the framework is a convex empirical risk minimization model with a prescribed convex surrogate loss that we call the \emph{adversarial surrogate loss}.

We develop algorithms to compute the adversarial surrogate losses efficiently: linear time for ordinal classification with the absolute loss metric, quasilinear time for the zero-one loss metric, and linear program-based algorithm for more general loss metrics.
We show that the adversarial surrogate losses 
fill the existing gap in surrogate loss construction for general multiclass classification problems
by simultaneously: (1) aligning better with the original multiclass loss metric, since optimizing the surrogate loss is equivalent with optimizing the original loss metric in the primal adversarial prediction formulation; (2) guaranteeing Fisher consistency; (3) enabling computational efficiency in a rich feature representation via the kernel trick; 
and (4) providing competitive performance in practice.

\subsection{Contributions of the Paper}

Some of the contents in this paper have previously appeared in machine learning conferences: the adversarial prediction formulation for general loss matrices \citep{asif2015adversarial}, the adversarial surrogate loss for the multiclass zero-one loss metric \citep{fathony2016adversarial}, the adversarial surrogate loss for ordinal classification with the absolute loss metric \citep{fathony2017adversarial}, and the Fisher consistency proof in the case of symmetric loss metrics \citep{fathony2018efficient}. 
This paper also contains distinct elements to provide a more general view of the adversarial prediction framework for general multiclass classification that have not previously been presented in the conference papers. 
The following is a summary of the new contributions included in this paper:
\begin{enumerate}
    \item A general view of adversarial surrogate losses for general multiclass classification problems (Section 3);
    \item A new proof technique for deriving the corresponding surrogate loss for a given loss metrics to optimize, based on the extreme points enumeration of the convex polytope (proofs in Section 3); 
    \item An extension to the ordinal classification problem using the squared loss rather than the absolute loss (Section 3.3);
    \item An analysis of the adversarial surrogate loss for the weighted loss metrics (Section 3.4);
    \item The loss formulation and prediction scheme of the adversarial surrogate loss for the task of classification with abstention (Section 3.5, Section 4.3);
    \item A Fisher consistency analysis for non-symmetric loss metrics under potential-based prediction schemes (Section 5.1);
    \item A Fisher consistency analysis for the case where the set of the predictor's options are different from the set of ground truth labels (Section 5.2); 
    \item A primal optimization algorithm to incorporate rich feature spaces via the kernel trick based on the PEGASOS algorithm (Section 6.2); and
    \item Additional experiments for the classification with abstention tasks (Section 7.3).
\end{enumerate}

\subsection{Paper Organization}

This article is organized as follows. The next section formulates the general multiclass classification problem, demonstrates some example problems, and discusses related techniques that solve these problems. 
Section 3 presents our adversarial prediction framework formulation, and the adversarial surrogate losses constructed from the dual formulation of the framework for several loss metrics including the zero-one loss, absolute loss, squared loss, and abstention-based loss metrics. 
Section 4 presents two different schemes for making predictions, probabilistic and non-probabilistic schemes.
Section 5 establishes the Fisher consistency property of adversarial surrogate losses. 
Section 6 presents algorithms for optimizing the adversarial surrogate losses as well as the technique to incorporate the kernel trick into the algorithm. 
Finally, Section 7 discusses experimental evaluations and the empirical advantages of the adversarial surrogate losses compared to the state-of-the-art techniques
that can be viewed as risk minimization methods with piece-wise convex surrogates. This includes the generalization of hinge loss and SVM to general multiclass classification problems.

% include other files
\section{Preliminaries and Related Works}

In multiclass classification problems, 
the predictor needs to predict a variable by choosing one class 
% the ground truth label are chosen 
from a finite set of possible class labels. 
The most popular form of multiclass classification uses zero-one loss metric minimization as the objective. This loss metric penalizes all mistakes equally with a loss of one for incorrect predictions and zero loss otherwise. In fact, the term ``multiclass classification'' itself, is widely used to refer to this specific variant that uses the zero-one loss as the objective. We refer to ``general multiclass classification'' as the multiclass classification task that can use any loss metric defined based on the predictor's label prediction and the true label in this work.

\subsection{General Multiclass Classification}

In a general multiclass classification problem, the predictor is provided with training examples that are pairs of training data and labels $\{({\bf x}_1, y_{1}), \hdots , ({\bf x}_n, y_{n}) \}$ drawn i.i.d. from a distribution $D$ on $\mathcal{X} \times \mathcal{Y}$, where $\mathcal{X}$ is the feature space and $\mathcal{Y} = [k] \triangleq \{1,\hdots,k\}$ is a finite set of class labels. 
For a given data point ${\bf x}$, the predictor has to provide a class label prediction $\hat{y} \in \mathcal{T} = [l] \triangleq  \{1,\hdots,l\}$. 
Although the set of prediction labels $\Tcal$ is usually the same as the set of ground truth labels $\Ycal$, we also consider settings in which they differ. 
A multiclass loss metric $\text{loss}(\hat{y}, y) : \mathcal{T} \times \mathcal{Y} \rightarrow [0,\infty)$, denotes the loss incurred by predicting $\hat{y}$ when the true label is $y$. 
The loss metric, $\text{loss}(\hat{y}, y)$, is also commonly written as a loss matrix ${\bf L} \in \mathbb{R}_+^{l \times k}$ (in this case, $\mathbb{R}_+$ refers to $[0,\infty)$), where the value of a matrix cell in $i$-th row and $j$-th column corresponds to the value of $\text{loss}(\hat{y}, y)$ when $\hat{y} = i$ and $y=j$. 
Some examples of the loss metrics for general multiclass classification problems are:
\begin{enumerate}
\item \textbf{Zero-one loss metric}. The predictor suffers one loss if its prediction is not the same as the true label, otherwise it suffers zero loss, $\text{loss}^{\text{0-1}}(\hat{y}, y) = I(\hat{y} \neq y)$. 
\item \textbf{Ordinal classification with absolute loss metric}. The predictor suffers a loss that increases as the prediction moves farther away from the true label. A canonical example for ordinal classification loss metric is the absolute loss, $\text{loss}^{\text{ord}}(\hat{y}, y) = |\hat{y} - y|$.
\item \textbf{Ordinal classification with squared loss metric}. The squared loss metric,
$\text{loss}^{\text{sq}}(\hat{y}, y) = (\hat{y} - y)^2$, is also popular for evaluating ordinal classification predictions.

\item \textbf{Classification with abstention}. In this prediction setting, a standard zero-one loss metric is used. However, the predictor has an additional prediction option to abstain from making a label prediction. Hence, $\Tcal \neq \Ycal$ in this setting. A constant penalty $\alpha$ is incurred whenever the predictor chooses to use the abstain option.

\end{enumerate}

Example loss matrices for these classification problems are shown in Figure \ref{fig:loss-matrix}.
\begin{figure*}[h]
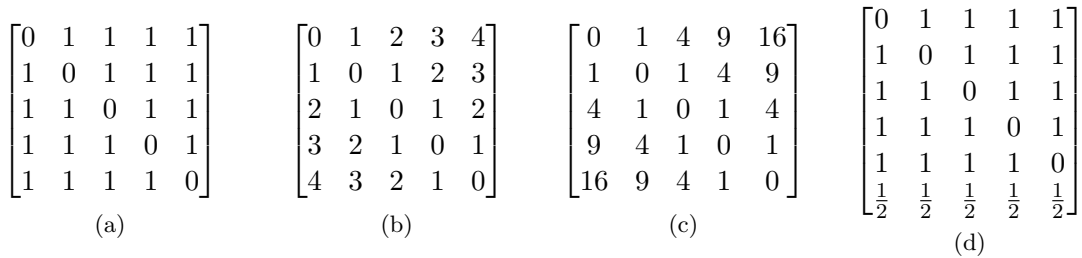

	\begin{minipage}{.25\linewidth}
		\centering
		\subfloat[]{
        \label{fig:lm-zo}
        $
        \begin{bmatrix}
        0 & 1 & 1 & 1 & 1 \\
        1 & 0 & 1 & 1 & 1 \\
        1 & 1 & 0 & 1 & 1 \\
        1 & 1 & 1 & 0 & 1 \\
        1 & 1 & 1 & 1 & 0
        \end{bmatrix}
        $
        }
	\end{minipage}%
	\begin{minipage}{.25\linewidth}
		\centering
		\subfloat[]{
        \label{fig:lm-abs}
        $
        \begin{bmatrix}
        0 & 1 & 2 & 3 & 4 \\
        1 & 0 & 1 & 2 & 3 \\
        2 & 1 & 0 & 1 & 2 \\
        3 & 2 & 1 & 0 & 1 \\
        4 & 3 & 2 & 1 & 0
        \end{bmatrix}
        $
        }
	\end{minipage}%
	\begin{minipage}{.25\linewidth}
		\centering
		\subfloat[]{
        \label{fig:lm-sq}
        $
        \begin{bmatrix}
        0 & 1 & 4 & 9 & 16 \\
        1 & 0 & 1 & 4 & 9 \\
        4 & 1 & 0 & 1 & 4 \\
        9 & 4 & 1 & 0 & 1 \\
        16 & 9 & 4 & 1 & 0
        \end{bmatrix}
        $
        }
	\end{minipage}%
	\begin{minipage}{.25\linewidth}
		\centering
		\subfloat[]{
        \label{fig:lm-abstain}
        $
        \begin{bmatrix}
        0 & 1 & 1 & 1 & 1 \\
        1 & 0 & 1 & 1 & 1 \\
        1 & 1 & 0 & 1 & 1 \\
        1 & 1 & 1 & 0 & 1 \\
        1 & 1 & 1 & 1 & 0 \\
        \frac12 & \frac12 & \frac12 & \frac12 & \frac12 
        \end{bmatrix}
        $
        }
	\end{minipage}
    
	\caption{
    Examples of the loss matrices for general multiclass classification when the number of class labels is 5 and the loss metric is: the zero-one loss (a), ordinal regression with the absolute loss (b), ordinal regression with the squared loss (c), and classification with abstention and $\alpha = \frac12$ (d).
		}
	\label{fig:loss-matrix}
\end{figure*}

\subsection{Empirical Risk Minimization and Fisher Consistency}
\label{sec:erm}

A standard approach to parametric classification 
%given a loss function $\text{loss}(\cdot,\cdot)$ 
is to assume some functional
form for the classifier (e.g., a linear discriminant function,
${\yhat}_{\theta}({\bf x}) = \operatornamewithlimits{argmax}_{y} 
\theta^\intercal \phi({\bf x},y)$, where 
$\phi({\bf x},y) \in \mathbb{R}^m$ is a feature
function) 
and then select model parameters $\theta$ that minimize the
empirical risk, 
\begin{align}
\argmin_\theta 
\mathbb{E}_{{\bf X},Y \sim \tilde{P}}\left[\text{loss}\left({\yhat}_{\theta}({\bf X}),Y\right)\right] + \lambda ||\theta||,
\end{align}
with a regularization penalty $\lambda||\theta||$ often added to avoid
overfitting to available training data\footnote{Lowercase non-bold, $x$,
and bold, ${\bf x}$, denote scalar and vector values, and capitals, 
$X$ or ${\bf X}$, denote random variables.}\!\!.
%We denote scalar values 
%and vector values lowercase non-bold, $x$, and bold, ${\bf x}$, 
%non-bold lowercase $x$, multivariate values in bold lowercase ${\bf x}$,
%and random variables in capital $X$ or ${\bf X}$.}.
Unfortunately, many combinations of classification functions,
$\yhat_\theta({\bf x})$, and loss metrics, %$\text{loss}(\cdot,\cdot)$, 
do not
lend themselves to efficient parameter optimization
under the empirical risk minimization (ERM) formulation.  For example, the
zero-one loss measuring the misclassification rate will generally lead to a
non-convex empirical risk minimization problem that is NP-hard to solve
\citep{hoffgen1995robust}.

\begin{figure}[ht]
% \vspace{-1mm}
\centering
\includegraphics[width=0.40\textwidth]{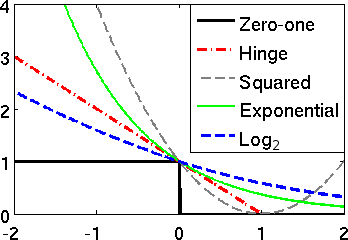}
\caption{Convex surrogates for the zero-one loss.}
\label{fig:losses}
\end{figure}

To avoid these intractabilities,
convex surrogate loss functions %s of the zero-one loss 
(Figure \ref{fig:losses}) that 
serve as upper bounds on the desired loss metric are 
often used to create tractable optimization objectives.
The popular support vector machine (SVM) classifier
\citep{cortes1995support}, for example,
employs the hinge-loss---an upper bound on the zero-one loss---to 
avoid the often intractable empirical risk minimization problem.
The logistic regression classifier \citep{mccullagh1989generalized} performs a probabilistic prediction
by minimizing the logarithmic loss, whereas
Adaboost \citep{freund1997decision} incrementally minimizes the exponential
loss.

There are many ways to construct convex surrogate loss functions for a given loss metric that we want to optimize. An important property for theoretically guaranteeing optimal prediction is Fisher consistency. It requires a learning method to produce Bayes optimal predictions which minimize the expected loss of this distribution, 
$\hat{y} \in \argmax_{y'} \mathbb{E}_{Y\sim P}[\text{loss}({y'},Y)]$ under ideal learning conditions (trained from the true data distribution $P(Y|\Xbf)$ using a fully expressive feature representation). Fisher consistency property guarantees that a learning algorithm (i.e. surrogate loss) reaches the optimal prediction under the original loss metric in the limit.
\citet{tewari2007consistency} presented 
techniques to characterize the Fisher consistency of surrogate losses for the multiclass zero-one loss metric, which then is extended by \citet{ramaswamy2012classification,ramaswamy2016convex}
to general multiclass loss metrics.

\subsection{General Multiclass Classification Methods}

A variety of methods have been proposed to address the general multiclass classification problem. 
Most of the methods can be viewed as optimizing surrogate losses that come from the extension of binary surrogate loss, e.g., hinge loss (used by SVM), logistic loss (used by logistic regression) and exponential loss (used by AdaBoost), 
to general multiclass cases.
We narrow our focus over this broad range of methods found in the related work to those that can be viewed as empirical risk minimization methods with piece-wise convex surrogates (i.e. generalized hinge loss / generalized SVM), which are more closely related to our approach. 

\subsubsection{Multiclass Zero-one Classification}

% The multiclass support vector machine (SVM) seeks class-based potentials 
% $f_y({\bf x})$ for each input vector ${\bf x} \in \boldsymbol{\mathcal{X}}$ 
% and class $y \in \mathcal{Y}$ so that the discriminant function,
% $\hat{y}_{\bf f}({\bf x}) = \argmax_y f_y({\bf x})$, 
% minimizes misclassification errors,
% $\text{loss}_{\bf f}({\bf x}, y) = I(y \neq \hat{y}_{\bf f}({\bf x}))$.  
Many methods have been proposed to generalize SVM to the multiclass  classification setting. 
Apart from the one-vs-all and one-vs-one decomposed formulations 
\citep{deng2012support},
there are three main joint formulations: 
\begin{enumerate}
\item 
The WW model by
\citet{weston1999support}, which incorporates the sum of hinge losses for all 
alternative labels,
\[
\text{loss}_{\text{WW}}({\bf x}, y) = 
\textstyle\sum_{j \neq y} \, [1 + (f_j({\bf x}) - f_{y}({\bf x}))]_+ ;
\]
\item The CS model by \citet{crammer2002algorithmic}, which uses the hinge
loss of only the largest alternative label,
\[
\text{loss}_{\text{CS}}({\bf x}, y) = 
\textstyle\max_{j \neq y} \left[1 + 
\left(f_{j}({\bf x}) - f_{y}({\bf x})\right)\right]_+ ; \text{ and}
\]
\item The LLW model by 
\citet{lee2004multicategory}, which employs an absolute hinge loss,
\[
\text{loss}_{\text{LLW}}({\bf x}, y) = \textstyle\sum_{j \neq y}  \left[1 + f_j({\bf x})\right]_+ ,
\]
and a constraint that $\sum_{j}  f_j({\bf x}) = 0$.
\end{enumerate}
% the WW model by
% \citet{weston1999support}, which incorporates the sum of hinge losses for all 
% alternative labels,
% $ \text{loss}_{\text{WW}}({\bf x}_i, y_i) = 
% \sum_{j \neq y_i} \, [1 - (f_{y_i}({\bf x}_i) - f_j({\bf x}_i))]_+
% $;
% the CS model by \citet{crammer2002algorithmic}, which uses the hinge
% loss of only the largest alternative label,
% $ \text{loss}_{\text{CS}}({\bf x}_i, y_i) = 
% \max_{j \neq y_i} \left[1- 
% \left(f_{y_i}({\bf x}_i) - f_{j}({\bf x}_i)\right)\right]_+
% $; 
% and the LLW model by 
% \citet{lee2004multicategory}, which employs an absolute hinge loss,
% $\text{loss}_{\text{LLW}}({\bf x}_i, y_i) = \sum_{j \neq y_i}  \left[1 + f_j({\bf x}_i)\right]_+$, and a constraint that $\sum_{j}  f_j({\bf x}_i) = 0$.
The former two models (CS and WW) both utilize the pairwise class-based potential differences $f_j({\bf x}) - f_{y}({\bf x})$ and are therefore categorized as relative margin methods. LLW, on the other hand, is an absolute margin method that only relates to $f_j({\bf x})$ \citep{dogan2016unified}.

Fisher consistency, or Bayes consistency  \citep{bartlett05,tewari2007consistency}, guarantees that minimization of a surrogate loss under the true distribution
provides the Bayes-optimal classifier, i.e., minimizes the  zero-one loss. 
%If given any possible distribution of data, 
%a classifier is Bayes-optimal, it is called universally consistent.
Among these methods, only the LLW method is Fisher consistent % and universally consistent
%It has been shown that neither the WW model nor the CS model are Fisher consistent 
\citep{lee2004multicategory,tewari2007consistency,liu2007fisher}.
%The LLW model provides a different approach by construction a loss function with Fisher consistency in mind, .
However, as pointed out by \citet{dogan2016unified}, LLW's use of an absolute 
margin in the loss (rather than the relative margin of WW and CS) often 
causes it to perform poorly for datasets with low dimensional feature spaces.
From the opposite direction, the requirements for Fisher consistency have been 
well-characterized \citep{tewari2007consistency}, yet this has not led to a
multiclass classifier that is Fisher consistent and performs well in
practice.

\subsubsection{Multiclass Ordinal Classification}

Existing techniques for ordinal classification that optimize piece-wise convex surrogates can be categorized into three groups as follows.

\begin{enumerate}

\item \underline{Threshold methods for ordinal classification}. \\
Threshold methods treat the ordinal response variable, $\hat{f} \triangleq {\bf w}\cdot{\bf x}$, as a continuous real-valued variable and introduce $k-1$ thresholds $\eta_1, \eta_2,..., \eta_{k-1}$ that partition the real line into $k$ segments: $\eta_0 = -\infty < \eta_1 < \eta_2 < ... < \eta_{k-1} < \eta_{k} =\infty $. Each segment corresponds to a label with $\hat{y}_i$  assigned label $j$ if $\eta_{j-1}< \hat{f} \leq \eta_{j}$.
There are two different approaches for constructing surrogate losses based on the threshold methods to optimize the choice of ${\bf w}$ and $\eta_1, \hdots, \eta_{k-1}$ \citep{shashua2003ranking,chu2005new,rennie2005lossfunctions}. 
\emph{All thresholds} method (also called SVORIM) 
penalizes all thresholds involved when a mistake is made.
\emph{Immediate thresholds} (also called SVOREX) only penalizes the most immediate thresholds.

\item \underline{A reduction framework from ordinal classification to binary classification.} \\
\citet{li2007ordinal} proposed a reduction framework to convert ordinal regression problems to binary classification problems by extending training examples. For each training sample $({\bf x}, y)$, the reduction framework creates $k-1$ extended samples $({\bf x}^{(j)}, y^{(j)})$ and assigns weight $w_{y,j}$ to each extended sample. The binary label associated with the extended sample is equivalent to the answer of the question: ``is the rank of ${\bf x}$ greater than $j$?'' The reduction framework %provides options to 
allows a choice for how extended samples ${\bf x}^{(j)}$ are constructed from original samples ${\bf x}$ and how to perform binary classification. 

\item \underline{Cost-sensitive classification methods for ordinal classification.} \\
Rather than using thresholding or the reduction framework, ordinal regression can also be cast as a special case of cost-sensitive multiclass classification.
Two of the most popular classification-based ordinal regression techniques are extensions of one-versus-one (OVO) and one-versus-all (OVA) 
cost-sensitive classification \citep{lin2008ordinal,lin2014reduction}.
Both algorithms leverage a transformation that converts a 
cost-sensitive classification problem to a set of weighted binary classification problems. 
Rather than reducing to binary classification, \citet{tu2010one} reduce cost-sensitive classification to one-sided regression (OSR), which can be viewed as an extension of the one-versus-all (OVA) technique. \\

\end{enumerate}

% In terms of Fisher consistency of surrogate losses for ordinal classification,
A recent analysis by \citet{pedregosa2017consistency} shows that many surrogate losses for ordinal classification enjoy Fisher consistency. For example, the \emph{all thresholds} and \emph{immediate thresholds} methods are Fisher consistent provided that the base binary surrogate losses they use are convex with differentiability and a negative derivative at zero.

\subsubsection{Multiclass Classification with Abstention}

In the classification with abstention setting, a standard zero-one loss is used to evaluate the prediction. However, the predictor has an additional option to abstain from making a label prediction and suffer a constant penalty $\alpha$. 
In the literature, this type of prediction setting is also called  %``\textit{classification with abstention}'' or
``\textit{classification with reject option}''.

Most of the early papers on classification with abstention focused on the binary prediction case. \citet{bartlett2008classification} proposed a consistent surrogate loss based on the SVM's hinge loss for binary classification with abstention where the value of $\alpha$ is restricted to the interval $[0, \frac12]$. \citet{grandvalet2009support} extended the approach to the case where the abstention penalty between the positive class $\alpha_+$ and negative class $\alpha_{-}$ is non-symmetric. A recent study by \citet{cortes2016boosting} proposed a modification of the boosting algorithm \citep{freund1997decision} that incorporate the abstention setting into the prediction. They also proposed a base weak classifier, \textit{abstention stump}, which is a modification from the popular weak classifier for the standard boosting algorithm (decision stump).

For the multiclass classification setting, a recent paper by \citet{ramaswamy2018consistent} proposed several algorithms that extend the binary hinge loss to the case of multiclass classification with abstention. 
They extended the definition of SVM's one-versus-all (OVA) and Crammer-Singer (CS) models to incorporate the abstention penalty. They also proposed a consistent algorithm for multiclass classification with abstention in the case of $\alpha \in [0, \frac12]$, by encoding the prediction classes in binary number representation and formulate a binary encoded prediction (BEP) surrogate.

\section{Adversarial Prediction Formulation}

In a general multiclass classification problem, 
the predictor needs to make a label prediction $\hat{y} \in \mathcal{T} = \{1,\hdots,l\}$ for a given data point ${\bf x}$.
To evaluate the performance of the prediction, we compute the multiclass loss metric $\text{loss}(\hat{y}, y)$ by comparing the prediction to the ground truth label $y$.
The predictor is also allowed to make a probabilistic prediction by outputting a conditional probability $\hat{P}(\hat{Y}|{\bf x})$. In this case, the expected loss $\Ebb_{\hat{Y} | {\bf x} \sim \hat{P}} \; \text{loss}(\hat{Y}, y) = \sum_{i=1}^l \hat{P}(\hat{Y}=i|{\bf x}) \; \text{loss}(i, y)$ is measured. Note that in our notation, the upper case $Y$ and ${\bf X}$ refer to random variables (of a scalar and vector respectively) while lower case $y$ and ${\bf x}$ refer to the observed variables.

Our approach seeks a predictor that robustly minimizes a multiclass loss metric against the worst-case distribution that (approximately) matches the statistics of the training data.
In this setting, a predictor makes a probabilistic prediction over the set of all possible labels (denoted as $\hat{P}(\hat{Y}|{\bf X})$). Instead of evaluating the predictor with the empirical distribution, the predictor is pitted against an adversary that also makes a probabilistic prediction (denoted as $\check{P}(\check{Y}|{\bf X})$). 
The predictor's objective is to minimize the expected loss metric calculated from the predictor's and adversary's probabilistic predictions, while the adversary seeks to maximize the loss. The adversary is constrained to select a probabilistic prediction that matches the statistical summaries of the empirical training distribution (denoted as $\tilde{P}$) via moment-matching constraints on the features $\phi({\bf x}, y)$. 

\begin{definition}  \label{def:adv}
	In the {\bf adversarial prediction framework} for general multiclass classification,
	the predictor player first selects a predictive distribution,
    $\hat{P}(\hat{Y}|{\bf X})$,
	for each input ${\bf x}$,
	from the conditional probability simplex, % $\Delta$,
	and then the adversarial player selects an evaluation distribution,
    $\check{P}(\check{Y}|{\bf X})$,
	for each input ${\bf x}$
	from the set %$\Xi$ 
	of distributions consistent with the known statistics:
	\begin{align}
    \min_{\hat{P}(\hat{Y}|{\bf X})} \; \max_{\check{P}(\check{Y}|{\bf X}) %\in \Xi
    } \; &
    \mathbb{E}_{{\bf X} \sim \tilde{P}; \hat{Y}|{\bf X}\sim\hat{P};
    	\check{Y}|{\bf X}\sim \check{P}} \left[\text{loss}(\hat{Y}, \check{Y}) \right]  \label{eq:def} \\ 
    \text{subject to:  } & \mathbb{E}_{{\bf X} \sim 
    \tilde{P};
    	\check{Y}|{\bf X}\sim \check{P}}[\phi({\bf X},\check{Y})]
    = \mathbb{E}_{{\bf X},{Y} \sim \tilde{P}}\left[\phi({\bf X},{Y}) \right]. \nonumber
    \end{align}
	Here, the statistics $\mathbb{E}_{{\bf X},{Y} \sim \tilde{P}}\left[\phi({\bf X},{Y}) \right]$ are a vector of provided feature moments measured from
	training data.
\end{definition}

For the purpose of establishing efficient learning algorithms, we 
use the method of Lagrangian multipliers and strong duality for convex-concave saddle point problems \citep{von1945theory,sion1958general} to formulate the equivalent dual optimization as stated in Theorem \ref{thm:dual}.

\begin{theorem}  \label{thm:dual}
	Determining the value of the \emph{constrained} adversarial prediction minimax game
	reduces to a minimization over the empirical average of the value of many
	unconstrained minimax games:
	\begin{align}
	\min_{\theta}
     \mathbb{E}_{{\bf X},{Y}\sim \tilde{P}} \left[
     \max_{\check{P}(\check{Y}|{\bf X})}
     \min_{\hat{P}(\hat{Y}|{\bf X})}
    \mathbb{E}_{\hat{Y}|{\bf X}\sim\hat{P};
    \check{Y}|{\bf X}\sim{\check{P}}}\left[ \text{loss}(\hat{Y},\check{Y})
    + \theta^\intercal \left( \phi({\bf X},\check{Y}) - \phi({\bf X},{Y}) \right)
    \right]
    \right], \label{eq:dual}
	\end{align}
	where $\theta$ is the Lagrange dual variable for the moment matching constraints.
\end{theorem}
\begin{proof}
	\begin{align}
    & \min_{\hat{P}(\hat{Y}|{\bf X})}
     \max_{\check{P}(\check{Y}|{\bf X})}
    \mathbb{E}_{{\bf X}\sim \tilde{P};\hat{Y}|{\bf X}\sim\hat{P};
    \check{Y}|{\bf X}\sim{\check{P}}}\left[ \text{loss}(\hat{Y},\check{Y})\right] \\
    & \qquad \quad \text{subject to: } 
    \mathbb{E}_{{\bf X}\sim \tilde{P};\check{Y}|{\bf X}\sim \check{P}}\left[\phi({\bf X},\check{Y}) \right] = \mathbb{E}_{{\bf X},{Y} \sim \tilde{P}}\left[\phi({\bf X},{Y}) \right] \notag \\
    \overset{(a)}{=}&
     \max_{\check{P}(\check{Y}|{\bf X})}
     \min_{\hat{P}(\hat{Y}|{\bf X})}
    \mathbb{E}_{{\bf X}\sim \tilde{P};\hat{Y}|{\bf X}\sim\hat{P};
    \check{Y}|{\bf X}\sim{\check{P}}}\left[ \text{loss}(\hat{Y},\check{Y})\right] \\
    & \qquad \quad \text{subject to: } 
    \mathbb{E}_{{\bf X}\sim \tilde{P};\check{Y}|{\bf X}\sim \check{P}}\left[\phi({\bf X},\check{Y}) \right] = \mathbb{E}_{{\bf X},{Y} \sim \tilde{P}}\left[\phi({\bf X},{Y}) \right] \notag \\
    \overset{(b)}{=}&
     \max_{\check{P}(\check{Y}|{\bf X})}
     \min_{\theta}
     \min_{\hat{P}(\hat{Y}|{\bf X})}
    \mathbb{E}_{{\bf X},{Y}\sim \tilde{P};\hat{Y}|{\bf X}\sim\hat{P};
    \check{Y}|{\bf X}\sim{\check{P}}}\left[ \text{loss}(\hat{Y},\check{Y})
    + \theta^\intercal \left( \phi({\bf X},\check{Y}) - \phi({\bf X},{Y}) \right)
    \right] \\
    \overset{(c)}{=}&
     \min_{\theta}
     \max_{\check{P}(\check{Y}|{\bf X})}
     \min_{\hat{P}(\hat{Y}|{\bf X})}
    \mathbb{E}_{{\bf X},{Y}\sim \tilde{P};\hat{Y}|{\bf X}\sim\hat{P};
    \check{Y}|{\bf X}\sim{\check{P}}}\left[ \text{loss}(\hat{Y},\check{Y})
    + \theta^\intercal \left( \phi({\bf X},\check{Y}) - \phi({\bf X},{Y}) \right)
    \right] \\
    \overset{(d)}{=}&
     \min_{\theta}
     \mathbb{E}_{{\bf X},{Y}\sim \tilde{P}}
     \left[
     \max_{\check{P}(\check{Y}|{\bf X})}
     \min_{\hat{P}(\hat{Y}|{\bf X})}
    \mathbb{E}_{\hat{Y}|{\bf X}\sim\hat{P};
    \check{Y}|{\bf X}\sim{\check{P}}}\left[ \text{loss}(\hat{Y},\check{Y})
    + \theta^\intercal \left( \phi({\bf X},\check{Y}) - \phi({\bf X},{Y}) \right)
    \right]
    \right].
\end{align}
The transformation steps above are described as follows:
\begin{enumerate}[label=(\alph*), itemsep=1pt]
    \item We flip the min and max order using minimax duality \citep{von1945theory}. The domains of $\hat{P}(\hat{Y}|{\bf X})$ and $\check{P}(\check{Y}|{\bf X})$ are both compact convex sets and the objective function is bilinear, therefore, strong duality holds.
    \item We introduce the Lagrange dual variable $\theta$ to directly incorporate the equality constraints into the objective function.
    \item The domain of $\check{P}(\check{Y}|{\bf X})$ is a compact convex subset of $\mathbb{R}^n$, while the domain of $\theta$ is $\mathbb{R}^m$. The objective is concave on $\check{P}(\check{Y}|{\bf X})$ for all $\theta$ (a non-negative linear combination of minimums of affine functions is concave), while it is convex on $\theta$ for all $\check{P}(\check{Y}|{\bf X})$. Based on Sion's minimax theorem \citep{sion1958general}, strong duality holds, and thus we can flip the optimization order of $\check{P}(\check{Y}|{\bf X})$ and $\theta$.
    \item Since the expression is additive in terms of $\check{P}(\check{Y}|{\bf X})$ and $\hat{P}(\hat{Y}|{\bf X})$, we can push the expectation over the empirical distribution ${\bf X},{Y} \sim \tilde{P}$ outside and independently optimize each $\check{P}(\check{Y}|{\bf x})$ and $\hat{P}(\hat{Y}|{\bf x})$.
\end{enumerate} 
\end{proof}

The dual problem (Eq. \eqref{eq:dual}) possesses the important property of being a \textbf{convex} optimization problem in $\theta$. The objective of Eq. \eqref{eq:dual} consists of the function $\text{loss}(\hat{Y},\check{Y})
+ \theta^\intercal \left( \phi({\bf X},\check{Y}) - \phi({\bf X},{Y}) \right)$ which is an affine function with respect to $\theta$, followed by operations that preserve convexity \citep{boyd2004convex}: 
(1) the non-negative weighted sum (the expectations in the objective), 
(2) the minimization in the predictor $\hat{P}(\hat{Y}|X)$  over a non-empty convex set out of a function that is jointly convex in $\theta$ and $\hat{P}(\hat{Y}|X)$,
and (3) the point-wise maximum in the adversary distribution $\check{P}(\check{Y}|X)$ over an infinite set of convex functions. 
Therefore, the overall objective is convex with respect to $\theta$. This property is important since we can use gradient-based optimization in our learning algorithm and guarantee  convergence to the global optimum of the objective despite the fact that the original loss metrics we want to optimize in the primal formulation of the adversarial prediction (Eq. \eqref{eq:def}) are non-convex and non-continuous.

Despite the different motivations between our adversarial prediction framework and the empirical risk minimization framework, the dual optimization formulation (Eq. \eqref{eq:dual}) resembles a risk minimization problem with the surrogate loss defined as:
\begin{align}
    AL({\bf x}, y, \theta) = 
    \max_{\check{P}(\check{Y}|{\bf x})}
     \min_{\hat{P}(\hat{Y}|{\bf x})}
    \mathbb{E}_{\hat{Y}|{\bf x}\sim\hat{P};
    \check{Y}|{\bf x}\sim{\check{P}}}\left[ \text{loss}(\hat{Y},\check{Y})
    + \theta^\intercal \left( \phi({\bf x},\check{Y}) - \phi({\bf x},{y}) \right)
    \right]. \label{eq:al-hat-check}
\end{align}
We call this surrogate loss the ``adversarial surrogate loss'' or in short ``AL''. In the next subsections, we will analyze more about this surrogate loss for different instances of general multiclass classification problems. 

Let us first simplify the notation used in our surrogate loss. We construct a vector ${\bf p}$ to compactly represent the predictor's conditional probability $\hat{P}(\hat{Y}|{\bf x})$, where the value of its $i$-th index is $p_i = \hat{P}(\hat{Y}=i|{\bf x})$. Similarly, we construct a vector ${\bf q}$ for the adversary's conditional probability, i.e., $q_i = \check{P}(\check{Y}=i|{\bf x})$. We also define a potential vector ${\bf f}$ whose $i$-th index stores the potential for the $i$-th class, i.e., $f_i = \theta^\intercal \phi({\bf x}, i)$. Finally, we use a matrix $\Lbf$ to represent the loss function introduced at the beginning of this section. Using these notations we can rewrite our adversarial surrogate loss as:
\begin{align}
    AL({\bf f}, y) = 
     \max_{\qvec \in \Delta}
    \min_{\pvec \in \Delta}
    \pvec^\intercal \Lbf \qvec 
    + \fvec^\intercal \qvec
    - f_{y}, 
\end{align}
where $\Delta$ denotes the conditional probability simplex. 
The maximin formulation above can be converted to a linear program as follows:
\begin{align}
    AL({\bf f}, y) = 
    \max_{\qvec, v} & \;
    v + \fvec^\intercal \qvec - f_y \label{eq:al-lp} \\
    \text{s.t.:} & \;  \Lbf_{(i,:)} \qvec \ge v \quad \forall i \in [k] \nonumber \\
    & \; q_i \ge 0 \quad \qquad \forall i \in [k] \nonumber \\
    & \; \qvec^\intercal {\bf 1}  = 1,   \nonumber % \\
%    & \; v \ge 0,
\end{align}
where $v$ is a slack variable for converting the inner minimization into sets of linear inequality constraints, and $\Lbf_{(i,:)}$ denote the $i$-th row of matrix $\Lbf$.
We will analyze the solution of this linear program for several different types of loss metrics to construct a simpler closed-form formulation of the surrogate loss.

\subsection{Multiclass Zero-One Classification}

The multiclass zero-one loss metric is one of the most popular metrics used in multiclass classification. The loss metric penalizes an incorrect prediction with a loss of one and zero otherwise, i.e., $\text{loss}(\hat{y},y) = I(\hat{y} \neq y)$. An example of zero-one loss matrix for classification with five classes can be seen in Figure \ref{fig:lm-zo}. 

We focus on analyzing the solution of the maximization in Eq. \eqref{eq:al-lp} for the case where $\Lbf$ is the zero-one loss matrix. 
Since the objective in Eq. \eqref{eq:al-lp} is linear and the constraints form a convex polytope $\Cbb$ over the space of $\begin{bmatrix} \qvec \\ v \end{bmatrix}$,
there is always an optimal solution that is an extreme point of the domain \citep[Theorem 32.2 of][]{Rockafellar70}.
The only catch is that $\Cbb$ is not bounded, 
but this can be easily addressed by adding a nominal constraint $v \ge -1$ (see Proposition \ref{prop:extreme}). 
Our strategy is to first characterize the extreme points of $\Cbb$ that may possibly solve Eq. \eqref{eq:al-lp},
and then the evaluation of adversarial loss ($AL$) becomes equivalent to finding an extreme point that maximizes the objective in Eq. \eqref{eq:al-lp}.

The polytope $\Cbb$ can be defined in its canonical form by using the half-space representation of a polytope as follows:
\begin{align}
    \Cbb =
    \left\{
    \begin{bmatrix} \qvec \\ v \end{bmatrix} \,\middle\vert\,
    \Abf \begin{bmatrix} \qvec \\ v \end{bmatrix} \ge \bvec,
    \where
    \Abf =
    \begin{bmatrix} \Lbf & -{\bf 1} \\ \Ibf & {\bf 0} \\ {\bf 1}^\intercal & 0 \\ -{\bf 1}^\intercal & 0 % \\
%    {\bf 0} & 1
    \end{bmatrix},
    \ \ 
    \bvec = 
    \begin{bmatrix} {\bf 0}\\ {\bf 0} \\ 1 \\ -1 %\\ 0 
    \end{bmatrix}
    \right\}.
    \label{eq:polytope-half-space}
\end{align}
Here $\Lbf$ is a $k$-by-$k$ loss matrix, $\Ibf$ is a $k$-by-$k$ identity matrix, ${\bf 1}$ and ${\bf 0}$ are vectors with length $k$ that contain all 1 and or all 0 respectively.
$\Abf$ has $2k + 2$ rows and $k+1$ columns. 
Below is an example of this half-space representation for a four-class classification with zero-one loss metric:
\begin{align}
\begin{array}{c} 
\phantom{0} \\ \phantom{0} \\ \text{1st block} \\ \phantom{0} \\\hdashline[2pt/2pt] \phantom{0}  \\ \phantom{0} \\ \text{2nd block} \\ \phantom{0} \\ \hdashline[2pt/2pt] \phantom{0} \\ \text{3rd block} \end{array}     
\left[
\begin{array}{ccccc}
0 & 1 & 1 & 1 & -1 \\ 1 & 0 & 1 & 1 & -1 \\ 1 & 1 & 0 & 1 & -1 \\ 1 & 1 & 1 & 0 & -1 \\ \hdashline[2pt/2pt]
1 & 0 & 0 & 0 & 0 \\ 0 & 1 & 0 & 0 & 0 \\ 0 & 0 & 1 & 0 & 0 \\ 0 & 0 & 0 & 1 & 0 \\ \hdashline[2pt/2pt]
1 & 1 & 1 & 1 & 0 \\ -1 & -1 & -1 & -1 & 0 
\end{array} 
\right]
\begin{bmatrix} q_1 \\ q_2 \\ q_3 \\ q_4 \\ v \end{bmatrix}
\ge 
\begin{bmatrix} 0 \\ 0 \\ 0 \\ 0 \\ 0 \\ 0 \\ 0 \\ 0 \\ 1 \\ -1 \end{bmatrix} .   
\end{align}

For simplicity, we divide $\Abf$ into 3 blocks of rows. The first block contains $k$ rows defining the constraints that relate the loss matrix with the slack variable $v$, the second  block also contains $k$ rows for non-negativity constraints, and the third block is for the sum-to-one constraints.

To characterize the extreme points of $\Cbb$ that solve Eq. \eqref{eq:al-lp}, 
we utilize the algebraic characterization of extreme points in a bounded polytope given by Theorem 3.17 from \citet{andreasson2005introduction}. 
For convenience, we quote it here.

\begin{proposition}[Theorem 3.17 from \citet{andreasson2005introduction}]
%\label{prop:extreme}
Let $\Pbb \triangleq \{ \cvec \in \Rbb^n \mid \Abf \cvec \geq \bvec \}$ be a bounded polytope, where $\Abf \in \Rbb^{m\times n}$ has $rank(\Abf) = n$ and $\bvec \in \Rbb^m$. 
For any $\bar{\cvec} \in \Pbb$,
let $\Ical(\bar{\cvec})$ be the set of row index $i$ such that $\Abf_{(i,:)} \bar{\cvec} = b_i$.
Let $\Abf_{\bar{\cvec}}$ and $\bvec_{\bar{\cvec}}$ be the submatrix and subvector of $\Abf$ and $\bvec$ that extract the rows in $\Ical(\bar{\cvec})$, respectively.
Then $\Abf_{\bar{\cvec}} \cvec = \bvec_{\bar{\cvec}}$ is called the \emph{equality subsystem} for $\bar{\cvec}$,
and
$\bar{\cvec} \in \Pbb$ is an extreme point if and only if $rank(\Abf_{\bar{\cvec}}) = n$.
\end{proposition}

Since $\Cbb$ is not bounded ($v$ can diverge to $-\infty$),
we now further characterize a subset of $\Cbb$ that must include an optimal solution to Eq. \eqref{eq:al-lp}.
\begin{proposition}
	\label{prop:extreme}
	Let $\ext \Cbb = \{\cvec \in \Cbb | \rank(\Abf_\cvec) = k+1\}$.
  %, \Ical(\cvec)$ contains at least one row from the first block$\,\}$.
	Then $\ext \Cbb$ must contain an optimal solution to Eq. \eqref{eq:al-lp}.
\end{proposition}

\begin{proof}
	Let us add a nominal constraint of $v \ge -1$ to the definition of $\Cbb$,
	and denote the new polytope as $\bar{\Cbb} := \cbr{\cvec: \Gbf \cvec \ge \begin{bmatrix}
	\bvec \\
	-1
	\end{bmatrix}}$, 
	where 
	$\Gbf = 
		\begin{bmatrix} \Abf \\ 
		\zero^\intercal \; 1 
	\end{bmatrix}$.
	It does not change the solution to Eq. \eqref{eq:al-lp} because $v$ appears in the objective only as $v$,
	and $\Lbf_{(i,:)} \qvec \ge 0$.
	However, this additional constraint makes $\bar{\Cbb}$ compact,
	allowing us to apply Theorem 3.17 of \citep{andreasson2005introduction} and conclude that any $\cvec = \begin{bmatrix}
	\qvec \\
	v
	\end{bmatrix}$ is an extreme point of $\bar{\Cbb}$ if and only if $\rank(\Gbf_{\cvec}) = k+1$.
	But all optimal solutions must have $v \ge 0$, 
	hence the last row of $\Gbf$ cannot be in $\Gbf_\cvec$.
	So it suffices to consider $\cvec$ with $\Gbf_\cvec = \Abf_\cvec$,
	whence $\rank(\Abf_\cvec) = k+1$.
\end{proof}

Obviously $\Abf_\cvec$ must include the third block of $\Abf$ for all $\cvec \in \Cbb$ in Eq. \eqref{eq:polytope-half-space}.
The rank condition also enforces that at least one row from the first block is selected.

For convenience, we will refer to $\ext \Cbb$ as the extreme point of $\Cbb$.%
\footnote{Indeed, it is the bona fide extreme point set of $\Cbb$ under the standard definition which does not require compactness \cite[Section 18,][]{Rockafellar70}.
But the guarantee of attaining optimality at an extreme point does require boundedness.}
By analyzing $\ext \Cbb$ in the case of multiclass zero-one classification, we simplify the adversarial surrogate loss (Eq. \eqref{eq:al-lp}) as stated in the following Theorem $\ref{thm:ermloss-zo}$.

\begin{theorem}
	\label{thm:ermloss-zo}
	The predictive function
	for the multiclass zero-one adversarial classification is equivalently obtained from empirical risk minimization under 
the adversarial zero-one loss function: 
	\begin{align}
	& \text{AL}^{\text{0-1}}
	({\bf f}, y) = \, 
	\max_{
    {S \subseteq
			[k], %\\ 
            \; S \neq \emptyset } }
	\frac{\sum_{i \in S} f_i
		+|S|-1}{|S|} - f_y, \label{eq:al-zo}
	\end{align}
	where 
$S$ is any non-empty subset of the $k$ classes.
\end{theorem}

\begin{proof}
	The $\text{AL}^{\text{0-1}}$ above corresponds to the set of ``extreme points"%
	\footnote{We add a quotation mark here because our proof will only show, as it suffices to show, that $D$ contains all the extreme points of $\Cbb$ and $D \subseteq \Cbb$.
	We do not need to show that $D$ is exactly the extreme point set of $\Cbb$, although that fact is not hard to show either.}
	\begin{align}
	D = \cbr{\begin{bmatrix}
		\qvec \\
		v
		\end{bmatrix}
		=
		\frac{1}{|S|} 
		\begin{bmatrix}
		\sum_{i \in S} \evec_i  \\
		|S|-1
		\end{bmatrix} 
		\ \middle\vert \ \emptyset \neq S \subseteq [k]},
	\end{align}
	where $\evec_i \in \RR^k$ is the $i$-th canonical vector with a single 1 at the $i$-th coordinate and 0 elsewhere.
	That means $\qvec$ first picks a nonempty support $S \subseteq [k]$,
	then places uniform probability of $\frac{1}{|S|}$ on these coordinates,
	and finally sets $v$ to $\frac{|S|-1}{|S|}$.

	By Proposition \ref{prop:extreme},
	it now suffices to prove that $D \subseteq \Cbb$ and $D \supseteq \ext \Cbb = \{\cvec \in \Cbb : \rank(\Abf_\cvec) = k+1\}$,
	i.e.,
	any $\cvec \in \Cbb$ whose equality system satisfies $\rank(\Abf_\cvec) = k+1$ must be in $D$.
	$D \subseteq \Cbb$ is trivial, so we focus on $D \supseteq \ext \Cbb$.
	
	Given $\cvec \in \ext \Cbb$,
	suppose the set of rows that $\Abf_\cvec$ selected from the first and second block of $\Abf$ are $R$ and $T$, respectively.
	Both $R$ and $T$ are subsets of $[k]$,
	indexed against $\Abf$.
	We first observe that $R$ and $T$ must be disjoint because if $i \in R \cap T$,
	then $q_i = 0$ and $v = \Lbf_{(i,:)} \qvec = \sum_{j\neq i} q_j = 1-q_i=1$.
	But then for all $j$, $\Lbf_{(j,:)} \qvec \ge v$ implies $1 \le \sum_{l \neq j} q_l = 1 - q_j$. This is impossible as it means $\qvec = \zero$.
	
	Now that $R$ and $T$ are disjoint, $\rank(\Abf_\cvec) = k+1$ implies that $R = [k] \backslash T$.
	Since $q_i=0$ for all $i \in T$,
	solving $|R|$ linear equalities with respect to $|R|$ unknowns yields $q_j = 1 / |R|$ for all $j \in R$.
	Such a tuple of $\qvec$ and $v$ is clearly in $D$.	
	Obviously $R$ cannot be empty because then $T = [k]$ and $\qvec = \zero$.
\end{proof}

We denote the potential differences $\psi_{i,y} = f_i - f_y$, then  
Eq \eqref{eq:al-zo}, can be equivalently written as:  
\[ 
\text{AL}^{\text{0-1}}
	({\bf f}, y) = \,  \max_{
    {S \subseteq [k], %\\ 
            \; S \neq \emptyset } }
            \frac{\sum_{i \in S} \psi_{i,y} +|S|-1}{|S|} .
\]
Thus, AL$^{\text{0-1}}$ is 
the maximum value over $2^k-1$ linear hyperplanes. 
For binary prediction tasks, there are three linear hyperplanes:
$\psi_{1,y}, \psi_{2,y}$ and 
$\frac{\psi_{1,y}+\psi_{2,y}+1}{2}$. 
Figure \ref{fig:adv2} shows the loss function in potential difference space $\psi$ when the true label is $y = 1$. Note that $\text{AL}^{\text{0-1}}$ 
combines two hinge functions
at $\psi_{2,y} = -1$ and 
$\psi_{2,y} = 1$, rather than SVM's single hinge at 
$\psi_{2,y} = -1$.
This difference from the hinge loss corresponds to the loss that is realized 
by randomizing label predictions of $\hat{P}({\hat{Y}|\xvec})$ in Eq. \eqref{eq:al-hat-check}.

\begin{figure}[ht]
\centering	
\includegraphics[width=0.42\linewidth]{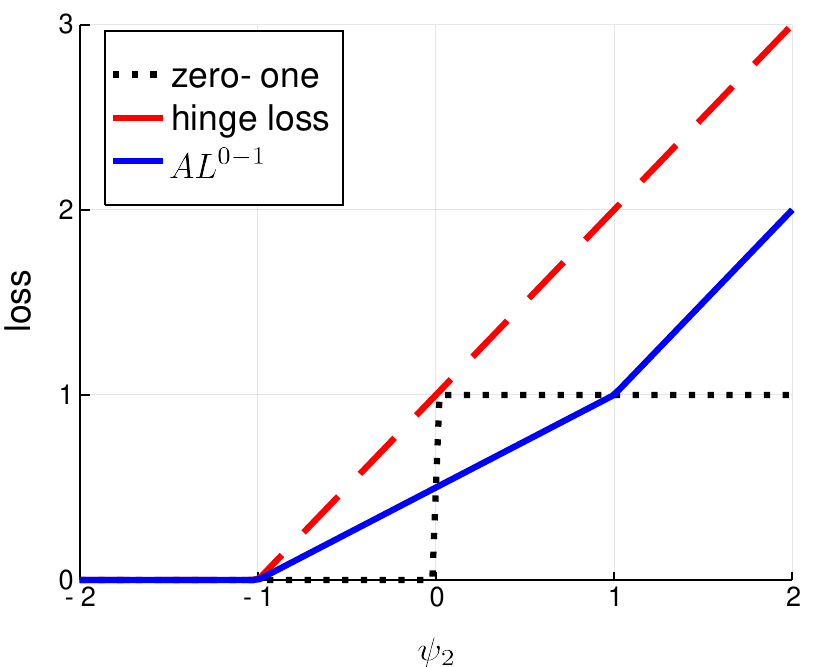}	
	\caption{$\text{AL}^{\text{0-1}}$ evaluated over the space of potential 
differences 
($\psi_{i, y} = f_{i} - f_{y}$; 
and $\psi_{i,i} = 0$) 
for binary prediction tasks when the true 
label is $y = 1$.}
	\label{fig:adv2}
\end{figure}

For three classes, the loss function has seven facets as shown in 
Figure \ref{fig:adv}. Figures \ref{fig:adv},
\ref{fig:ww}, and \ref{fig:cs}
show
the similarities and differences between $\text{AL}^{\text{0-1}}$ and the 
multiclass SVM surrogate losses based on class potential differences. Note that $\text{AL}^{\text{0-1}}$ is a relative margin loss function that utilizes the pairwise potential \emph{difference} $\psi_{i, y}$. This avoids the surrogate loss construction pitfall pointed out by \citet{dogan2016unified} that states that surrogate losses based on the absolute margin (rather than relative margin) may suffer from low performance for datasets with low dimenional feature spaces.

\begin{figure}[ht]
	\begin{minipage}{.33\linewidth}
		\centering
		\subfloat[]{\label{fig:adv}\includegraphics[width=1.0\linewidth]{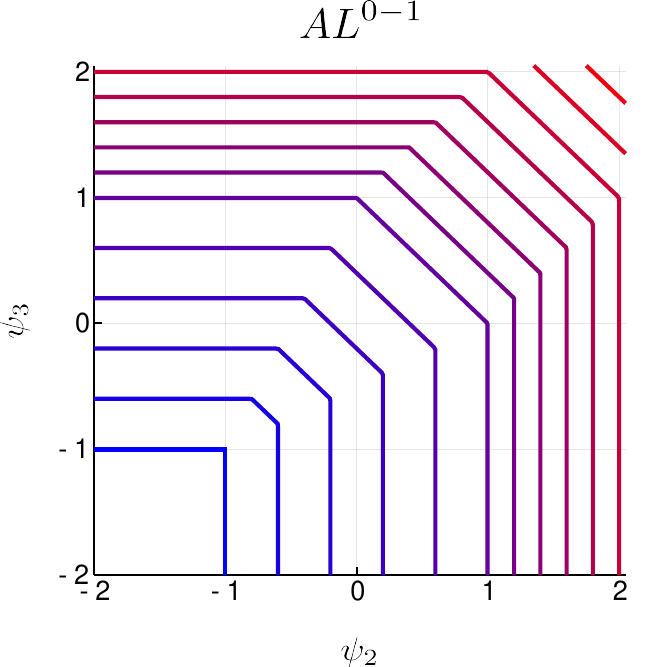}}
	\end{minipage}%
	\begin{minipage}{.33\linewidth}
		\centering
		\subfloat[]{\label{fig:ww}\includegraphics[width=1.0\linewidth]{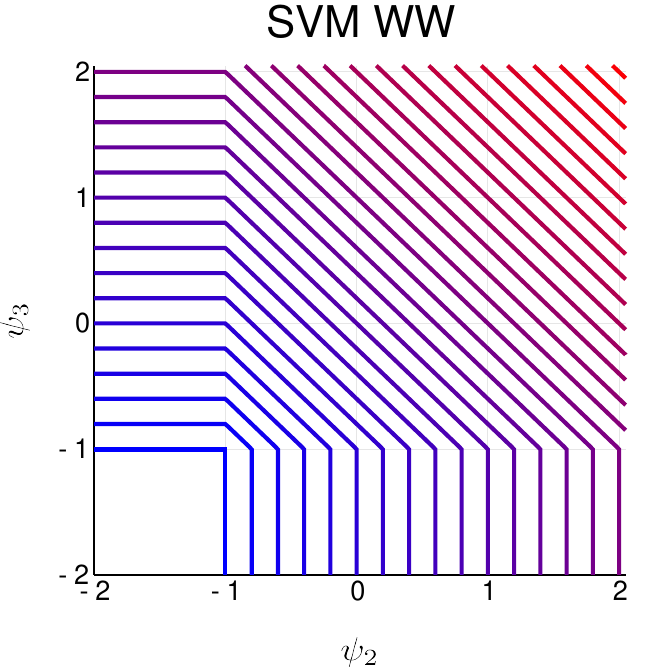}}
	\end{minipage}%
	\begin{minipage}{.33\linewidth}
		\centering
		\subfloat[]{\label{fig:cs}\includegraphics[width=1.0\linewidth]{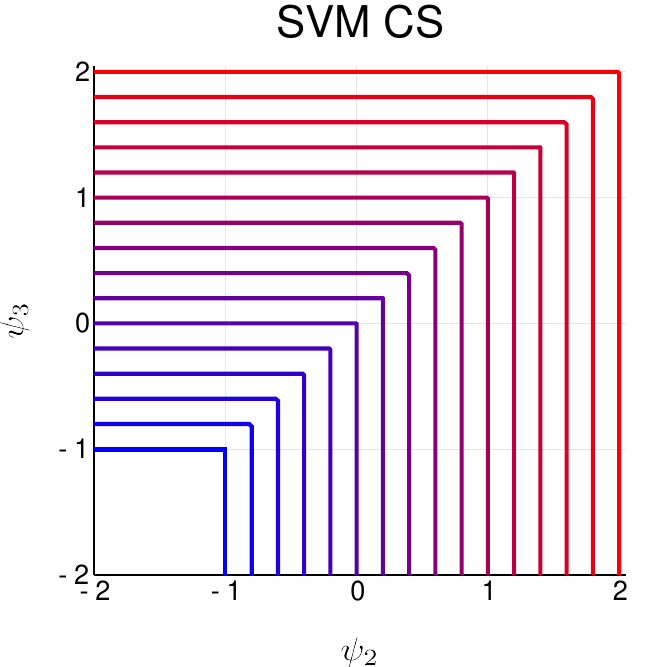}}
	\end{minipage}		
	\caption{
Loss function contour plots over the space of potential differences  for the prediction task with three classes when the true label is $y = 1$ under $\text{AL}^{\text{0-1}}$ (a),
the WW loss (b), and the CS loss (c).
(Note that $\psi_i$ in the plots refers to $\psi_{i, y} = f_{i} - f_{y}$; and $\psi_{i,i} = 0$.)
		}

	\label{fig:advloss3}
\end{figure}

Even though AL$^{\text{0-1}}$ is 
the maximization over $2^k-1$ possible values, it can be efficiently computed as follows. First we need to sort the potential for all labels $\{f_i : i \in [k]\}$ in non-increasing order. The set $S^*$ that maximize AL$^{\text{0-1}}$ must include the first $j$ labels in the sorted order, for some value of $j$. Therefore, to compute AL$^{\text{0-1}}$, we can incrementally add the label in the sorted order to the set $S^*$ until adding an additinal label would decrease the value of the loss.\footnote{We refer the reader to the Appendix C of \citep{fathony2016adversarial} for the optimality proof of this algorithm.} This results in an algorithm with a runtime complexity of $\Ocal(k \log k)$, which is much faster than enumerating all possible values in the maximization.

\subsection{Ordinal Classification with Absolute Loss}

In multiclass ordinal classification (also known as ordinal regression), the discrete class labels being predicted have an inherent order (e.g., \emph{poor}, \emph{fair}, \emph{good}, \emph{very good}, and \emph{excellent} labels).
The absolute error, $\text{loss}(\hat{y},y) = |\hat{y}-y|$ between label prediction ($\hat{y} \in \mathcal{Y}$) and actual label ($y \in \mathcal{Y}$) is a canonical ordinal regression loss metric. 
The adversarial surrogate loss for ordinal classification using the absolute loss metric is defined in Eq. \eqref{eq:al-lp}, where $\Lbf$ is the absolute loss matrix (e.g., Figure \ref{fig:lm-abs} for a five class ordinal classification).
The constraints in Eq. \eqref{eq:al-lp} form a convex polytope $\Cbb$. Below is an example of the half-space representation of $\Cbb$ for a four-class ordinal classification problem.
\begin{align}
\begin{array}{c} 
\phantom{0} \\ \phantom{0} \\ \text{1st block} \\ \phantom{0} \\\hdashline[2pt/2pt] \phantom{0}  \\ \phantom{0} \\ \text{2nd block} \\ \phantom{0} \\ \hdashline[2pt/2pt] \phantom{0} \\ \text{3rd block} \end{array}     
\left[
\begin{array}{ccccc}
0 & 1 & 2 & 3 & -1 \\ 1 & 0 & 1 & 2 & -1 \\ 2 & 1 & 0 & 1 & -1 \\ 3 & 2 & 1 & 0 & -1 \\ \hdashline[2pt/2pt]
1 & 0 & 0 & 0 & 0 \\ 0 & 1 & 0 & 0 & 0 \\ 0 & 0 & 1 & 0 & 0 \\ 0 & 0 & 0 & 1 & 0 \\ \hdashline[2pt/2pt]
1 & 1 & 1 & 1 & 0 \\ -1 & -1 & -1 & -1 & 0 
\end{array} 
\right]
\begin{bmatrix} q_1 \\ q_2 \\ q_3 \\ q_4 \\ v \end{bmatrix}
\ge 
\begin{bmatrix} 0 \\ 0 \\ 0 \\ 0 \\ 0 \\ 0 \\ 0 \\ 0 \\ 1 \\ -1 \end{bmatrix} .   
\end{align}
By analyzing the extreme points of $\Cbb$, we define the adversarial surrogate loss for ordinal classification with absolute loss $\text{AL}^{\text{ord}}$ as stated in Theorem \ref{thm:ermloss-abs}.
\begin{theorem}
	\label{thm:ermloss-abs}
An adversarial ordinal classification predictor with absolute loss is obtained by choosing 
%parameters ${\theta}$
a predictive function
that minimize the empirical risk of the surrogate loss function:
\begin{align}
	& \text{AL}^{\text{ord}}
	({\bf f}, y) = 
	  \max_{i,j \in [k] }  \frac{f_i + f_j + j - i}{2} 
      - f_y \label{eq:al-abs}.
\end{align}
\end{theorem}

\begin{proof}
    The $\text{AL}^{\text{ord}}$ above corresponds to the set of ``extreme points" 
    \begin{align}
    	D = \cbr{\begin{bmatrix}
    		\qvec \\
    		v
    		\end{bmatrix}
    		=
    		\frac{1}{2} 
    		\begin{bmatrix}
    			\evec_i + \evec_j \\
    			j - i
    		\end{bmatrix} 
    		\ \middle \vert \  i, j \in [k]}.
    \end{align}

    This means $\qvec$ can only have one or two non-zero elements (note that $i$ and $j$ can be equal) with uniform probability of $\frac{1}{2}$ and the value of $v$ is $\frac{j-i}{2}$. 
    
    Similar to the proof of Theorem \ref{thm:ermloss-zo}, 
    we next prove that $D \supseteq \ext \Cbb = \{\cvec \in \Cbb : \rank(\Abf_\cvec) = k+1\}$.    
    Given $\cvec \in \ext \Cbb$,
    suppose the set of rows that $\Abf_\cvec$ selected from the first and second block of $\Abf$ are $S$ and $T$, respectively.
    Both $S$ and $T$ are subsets of $[k]$,
    indexed against $\Abf$.
    Denote $s_{\max} = \max (S)$ and $s_{\min} = \min (S)$.    
    We consider two cases:
    \begin{enumerate}%[noitemsep,topsep=0pt]
        \item $S \cap T = \emptyset$: the indices selected from the first and second blocks are disjoint. \\

        It is easy to check that
        $\cvec$ must be 
        $\begin{bmatrix}
        \qvec \\
        v
        \end{bmatrix}
        :=
        \frac{1}{2} 
        \begin{bmatrix}
        \evec_\smax + \evec_\smin \\
        \smax - \smin
        \end{bmatrix} 
        $.
        Obviously it satisfies (being equal) the rows in $\Abf_\cvec$ extracted from the first and third blocks of $\Abf$, 
        because $|l - \smax| + |l - \smin| = \smax - \smin$ for all $l \in S$.
        Since $S \cap T = \emptyset$, $\cvec$ must also satisfy those rows from the second block.
        Finally notice that only one vector in $\RR^{k+1}$ can meet all the equalities encoded by $\Abf_\cvec$ because $\rank(\Abf_\cvec) = k+1$.
        Obviously $\cvec \in D$.
        
        \item $S \cap T \neq \emptyset$: the indices from the first block overlap with those from the second block.\\
				Including in $\Abf_\cvec$ the $i$-th row of the second block means setting $q_i$ to 0.
        Denote the set of remaining indices as 
        $R = [k] \backslash T$,
        and let $\rmax = \max(R)$ and $\rmin = \min(R)$.
        Now consider two sub-cases:
        \begin{enumerate}[label=\alph*)] %.,noitemsep,topsep=0pt]
            \item $\rmin \le \smin$ and $\rmax \ge \smax$. \\
            
            One may check that
            $\cvec$ must be 
            $\begin{bmatrix}
            \qvec \\
            v
            \end{bmatrix}
            :=
            \frac{1}{2} 
            \begin{bmatrix}
            \evec_\rmax + \evec_\rmin \\
            \rmax - \rmin
            \end{bmatrix} 
            $.
            Obviously it satisfies (being equal) the rows in $\Abf_\cvec$ extracted from the first and third blocks of $\Abf$, 
            because for all $l \in S$,
            $l \ge \smin \ge \rmin$ and $l \le \smax \le \rmax$,
            implying $|l - \rmax| + |l - \rmin| = \rmax - \rmin$.
            Since by definition $\rmax$ and $\rmin$ are not among the rows selected from the second block,
            the equalities from the second block must also be satisfied.
            As in case 1, only one vector in $\RR^{k+1}$ can meet all the equalities encoded by $\Abf_\cvec$ because $\rank(\Abf_\cvec) = k+1$.
            Obviously $\cvec \in D$.            
            
            \item $\rmin > \smin$ or $\rmax < \smax$. \\
            We first show $\rmin > \smin$ is impossible.
            By definition of $R$, $q_l = 0$ for all $l < \rmin$.
            For all $l \ge \rmin$ ($> \smin$),
            it follows that $\Lbf_{(\smin,l)} = l - \smin > l - \rmin = \Lbf_{(\rmin,l)}$.
            Noting that at least one $q_l$ must be positive for $l \ge \rmin$ (because of the sum-to-one constraint),
            we conclude that $\Lbf_{(\smin,:)} \qvec > \Lbf_{(\rmin,:)} \qvec$.
            But this contradicts with $\Lbf_{(\smin,:)} \qvec = v \le \Lbf_{(\rmin,:)} \qvec$,
            where the equality is because $\smin \in S$.
            
            Similarly, $\rmax < \smax$ is also impossible.
        \end{enumerate}
    \end{enumerate}

    Therefore, in all possible cases, we have shown that any $\cvec$ in $\ext \Cbb$ must be in $D$.
    Further noticing the obvious fact that $D \subseteq \Cbb$,
    we conclude our proof.
\end{proof}

We note that the AL$^{\text{ord}}$ surrogate is the maximization over pairs of different potential functions associated with each class (including pairs of identical class labels) added to the distance between the pair.
To compute the loss more efficiently, we make use of the fact that maximization over each element of the pair can be independently realized:
\begin{align}
	  \max_{i,j \in [k] }  \frac{f_i + f_j + j - i}{2} 
      - f_y = \tfrac12 \max_i \left( f_i - i \right) + \tfrac12 \max_j  \left( f_j + j \right) - f_y. \label{eq:al-abs-dec}
\end{align}
We derive two different versions of AL$^{\text{ord}}$ based on different feature representations used for constraining the adversary's probability distribution.

\subsubsection{Feature Representations}

We consider two feature representations corresponding to different training data summaries:
\begin{align}
\phi_{th}({\bf x},y) = \left(
\begin{array}{c}
y {\bf x}\\
I(y \leq 1)\\
I(y \leq 2)\\
\vdots\\
I(y\leq k-1)
\end{array} \right); \text{ and }
\quad \phi_{mc}({\bf x},y) = \left(
\begin{array}{c}
I(y=1) {\bf x}\\
I(y=2) {\bf x}\\
I(y=3) {\bf x}\\
\vdots\\
I(y=k) {\bf x}
\end{array}\right).
\end{align}

The first, which we call the  {\bf thresholded regression representation}, has size $m + k-1$, where $m$ is the dimension of our input space.
It induces a single shared vector of feature weights and a set of thresholds.
If we denote the weight vector associated with the $y {\bf x}$ term as ${\bf w}$ and the terms associated with the cumulative sum of class indicator functions as 
$\eta_1$, $\eta_2$, $\ldots$, $\eta_{k-1}$, 
then thresholds for switching between class $i$ and $i+1$ (ignoring other classes) occur when ${\bf w} \cdot {\bf x} = \eta_{j}$. 

The second feature representation, $\phi_{mc}$, which we call the {\bf multiclass  representation}, has size $mk$
and can be equivalently interpreted as inducing a set of class-specific feature weights, $f_i = {\bf w}_i \cdot {\bf x}$. 
This feature representation is useful when ordered labels cannot be thresholded according to any single direction in the input space, as shown in the example dataset of Figure \ref{fig:mc-data}.
\begin{figure}[ht]
\centering
% \vspace{-4mm}
\includegraphics[width=0.30\linewidth]{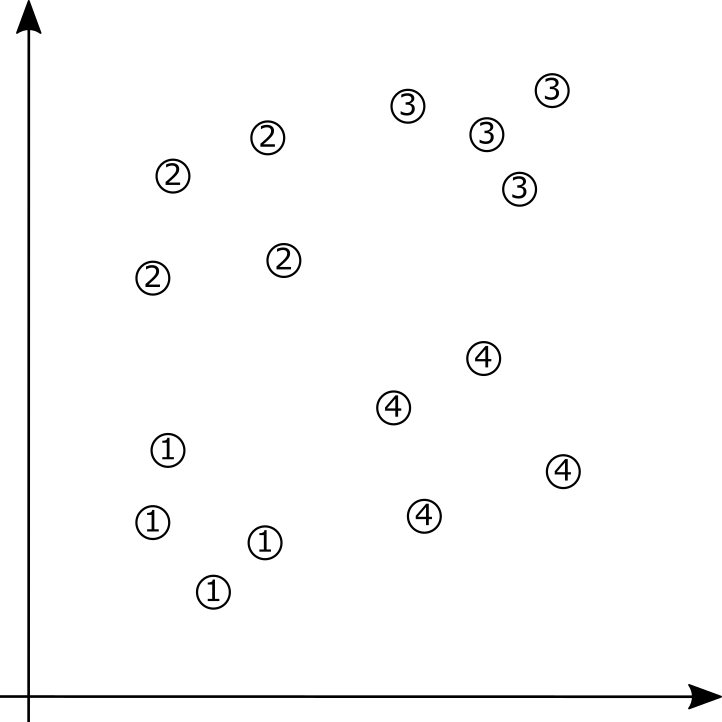}
\caption{Example where multiple weight vectors are useful.}
\label{fig:mc-data}
\end{figure}

\subsubsection{Thresholded regression surrogate loss}
In the thresholded regression feature representation, the parameter contains a single shared vector of feature weights ${\bf w}$ and $k-1$ terms $\eta_k$ associated with thresholds. Following Eq. \eqref{eq:al-abs-dec}, the adversarial ordinal regression surrogate loss for this feature representation can be written as:
{\small
\begin{align}
	& \text{AL}^{\text{ord-th}} 
	({\bf x}, y) = 
      \max_i \frac{i({\bf w}\cdot {\bf x}-1)+ \sum_{l \geq i}\eta_l}{2} 
      +
      \max_j \frac{j({\bf w}\cdot {\bf x}+1)+ \sum_{l \geq j}\eta_l}{2} 
      - y {\bf w}\cdot{\bf x}-\sum_{l \geq y} \eta_l.
      \label{eq:th}
\end{align}}%

This loss has a straight-forward interpretation in terms of the thresholded regression perspective, as shown in Figure \ref{fig:thresholded}: it is based on averaging the thresholded label predictions for potentials ${\bf w}\cdot{\bf x}-1$ and ${\bf w}\cdot{\bf x}+1$.  This penalization of the pair of thresholds differs from the thresholded surrogate losses of related work, which either penalize all violated thresholds or penalize only the thresholds adjacent to the actual class label. 
\begin{figure}[ht]
\centering
\includegraphics[width=0.60\linewidth]{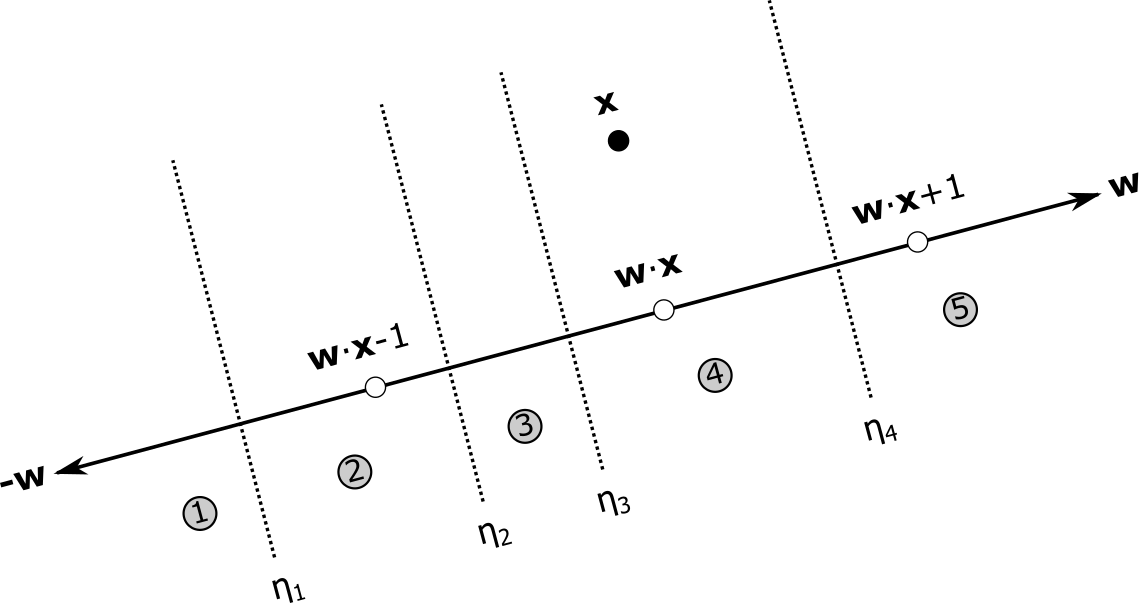}
\caption{Surrogate loss calculation for datapoint ${\bf x}$ (projected to ${\bf w}\cdot {\bf x}$) with a label prediction of $4$ for predictive purposes, the surrogate loss is instead obtained using potentials for the classes based on ${\bf w}\cdot{\bf x}-1$ (label $2$) and ${\bf w}\cdot{\bf x}+1$ (label $5$) averaged together.} 
\label{fig:thresholded}
\end{figure}

Using a binary search procedure over $\eta_1, \hdots, \eta_{k-1}$, the largest lower bounding threshold for each of these potentials can be obtained in $\mathcal{O}(\log k)$ time.

\subsubsection{Multiclass ordinal surrogate loss}
In the multiclass feature representation, we have a set of feature weights ${\bf w}_i$ for each label and the adversarial multiclass ordinal surrogate loss can be written as:
\[
	\text{AL}^{\text{ord-mc}} 
	({\bf x}, y) = 
    \max_{i,j \in [k] } \frac{{\bf w}_{i} \cdot {\bf x} + {\bf w}_{j} \cdot {\bf x} + j - i }{2}
      - {\bf w}_{y} \cdot {\bf x}. \label{eq:mc}
\]
We can also view this as the maximization over ${k(k+1)}/{2}$ linear hyperplanes. 
For an ordinal regression problem with three classes, the loss has six facets with different shapes for each true label value, as shown in Figure \ref{fig:advloss-ord}.
In contrast with $\text{AL}^{\text{ord-th}}$, the class potentials for $\text{AL}^{\text{ord-mc}}$ may differ from one another in more-or-less arbitrary ways.  Thus, searching for the maximal $i$ and $j$ class labels requires $\mathcal{O}(k)$ time.

\begin{figure*}[ht]
\centering
	\begin{minipage}{.33\linewidth}
		\centering
		\subfloat[]{\label{fig:y1}\includegraphics[width=1.0\linewidth]{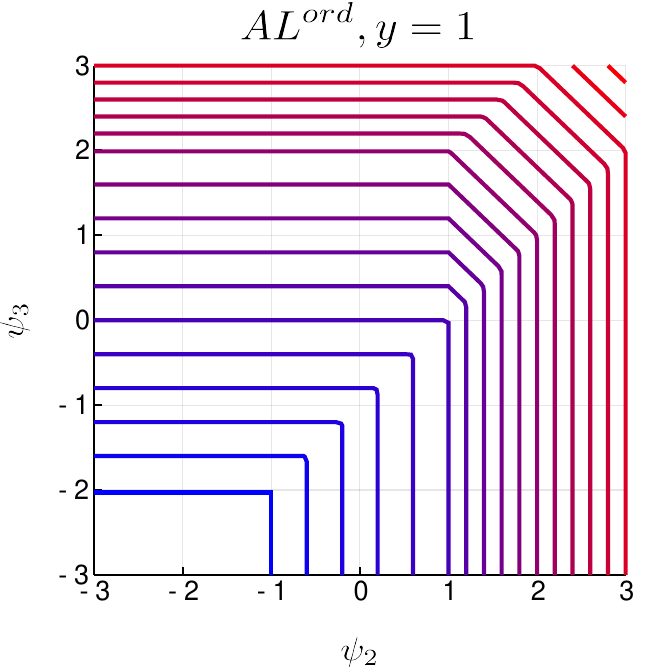}}
	\end{minipage}%
	\begin{minipage}{.33\linewidth}
		\centering
		\subfloat[]{\label{fig:y2}\includegraphics[width=1.0\linewidth]{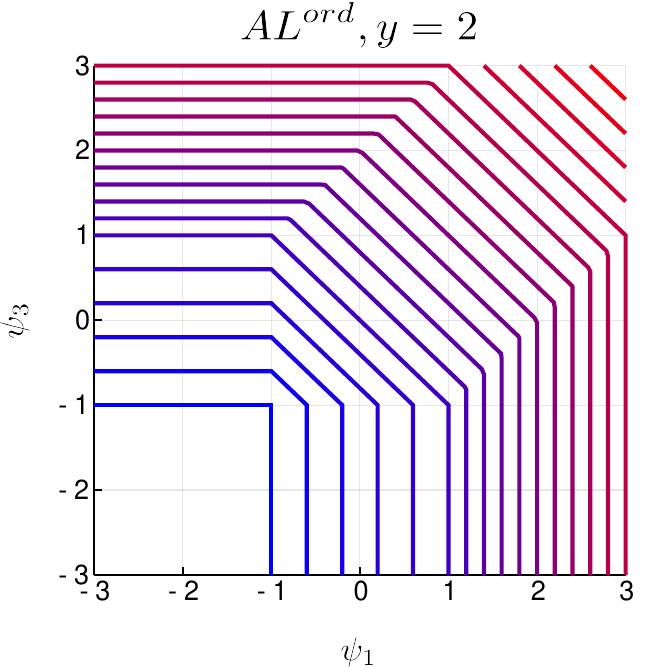}}
	\end{minipage}%
	\begin{minipage}{0.33\linewidth}
		\centering
		\subfloat[]{\label{fig:y3}\includegraphics[width=1.0\linewidth]{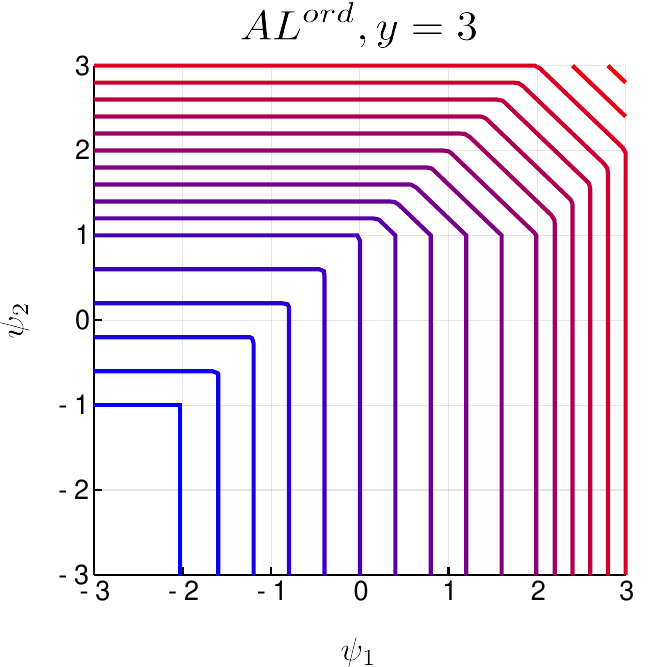}}
	\end{minipage}		
	\caption{
Loss function contour plots of $\text{AL}^{\text{ord}}$ over the space of potential differences $\psi_j \triangleq f_j - f_{y}$ for the prediction task with three classes when the true label is $y = 1$ (a), $y = 2$ (b), and $y = 3$ (c).
		}
	\label{fig:advloss-ord}
\end{figure*}

\subsection{Ordinal Classification with Squared Loss}

In some prediction tasks, the squared loss is the preferred metric for ordinal classification to enforce larger penalty as the difference between the predicted and true label increases \citep{baccianella2009evaluation, pedregosa2017consistency}. The loss is calculated using the squared difference between label prediction ($\hat{y} \in \mathcal{Y}$) and ground truth label ($y \in \mathcal{Y}$), that is: $\text{loss}(\hat{y},y) = (\hat{y}-y)^2$. 
The adversarial surrogate loss for ordinal classification using the squared loss metric is defined in Eq. \eqref{eq:al-lp}, where $\Lbf$ is the squared loss matrix (e.g. Figure \ref{fig:lm-sq} for a five classes ordinal classification).
The constraints in Eq. \eqref{eq:al-lp} form a convex polytope $\Cbb$. Below is an example of the half-space representation of $\Cbb$ for a four-class ordinal classification problem with squared loss metric.
\begin{align}
\begin{array}{c} 
\phantom{0} \\ \phantom{0} \\ \text{1st block} \\ \phantom{0} \\\hdashline[2pt/2pt] \phantom{0}  \\ \phantom{0} \\ \text{2nd block} \\ \phantom{0} \\ \hdashline[2pt/2pt] \phantom{0} \\ \text{3rd block} \end{array}     
\left[
\begin{array}{ccccc}
0 & 1 & 4 & 9 & -1 \\ 1 & 0 & 1 & 4 & -1 \\ 4 & 1 & 0 & 1 & -1 \\ 9 & 4 & 1 & 0 & -1 \\ \hdashline[2pt/2pt]
1 & 0 & 0 & 0 & 0 \\ 0 & 1 & 0 & 0 & 0 \\ 0 & 0 & 1 & 0 & 0 \\ 0 & 0 & 0 & 1 & 0 \\ \hdashline[2pt/2pt]
1 & 1 & 1 & 1 & 0 \\ -1 & -1 & -1 & -1 & 0 
\end{array} 
\right]
\begin{bmatrix} q_1 \\ q_2 \\ q_3 \\ q_4 \\ v \end{bmatrix}
\ge 
\begin{bmatrix} 0 \\ 0 \\ 0 \\ 0 \\ 0 \\ 0 \\ 0 \\ 0 \\ 1 \\ -1 \end{bmatrix} .   
\end{align}

We define the adversarial surrogate loss for ordinal classification with squared loss $\text{AL}^{\text{sq}}$ as stated in Theorem \ref{thm:ermloss-sq}.

\begin{theorem}
	\label{thm:ermloss-sq}
An adversarial ordinal classification predictor with squared loss is obtained by choosing 
%parameters ${\theta}$ 
a predictive function
that minimize the empirical risk of the surrogate loss function:
\begin{align}
	\text{AL}^{\text{sq}}
	({\bf f}, y) = 
	  \max \Bigg\{
      \max_{\substack{i,j,l \in [k] \\ i < l \leq j }}  
	  \tfrac{ \left( 2(j - l) + 1 \right) \left[ f_i + \left( l - i \right)^2 \right] + 
	  \left( 2(l - i) - 1 \right) \left[ f_j + \left( j - l \right)^2 \right] }
	  {2 \left(j - i \right)},\; \max_i f_i\; 
	  \Bigg\}
      - f_y \label{eq:al-quad}.
\end{align}
\end{theorem}

\begin{proof}
% 	\textbf{(New proof)}
    The $\text{AL}^{\text{sq}}$ above corresponds to the set of extreme points 
    \begin{align}
    	D = \cbr{\begin{bmatrix}
    		\qvec \\
    		v
    		\end{bmatrix}
    		=
    		\tfrac{ 2(j - l) + 1}{2 \left(j - i \right)} 
    		\begin{bmatrix}
    			\evec_i \\
    			(l-i)^2
    		\end{bmatrix} +
    		\tfrac{ 2(l - i) - 1}{2 \left(j - i \right)} 
    		\begin{bmatrix}
    			\evec_j \\
    			(j-l)^2
    		\end{bmatrix} 
    		\ \middle \vert \  \begin{matrix}i, j, l \in [k]\\  i < l \leq j\end{matrix} } 
    		\cup \cbr{
    		    \begin{bmatrix}
    		\qvec \\
    		v
    		\end{bmatrix}
    		= \begin{bmatrix}
    			\evec_i \\
    			0
    		\end{bmatrix} 
    		 \middle \vert \ i \in [k]
    		}.
    \end{align}
%    where $\evec_i \in \RR^k$ is the $i$-th canonical vector with a single 1 at the $i$-th coordinate and 0 elsewhere.
    This means $\qvec$ can either have one non-zero element with a probability of one or two non-zero elements with the probability specified above. 
    
    Similar to the proof of Theorem \ref{thm:ermloss-zo}, 
    we next prove that $D \supseteq \ext \Cbb = \{\cvec \in \Cbb : \rank(\Abf_\cvec) = k+1\}$,
    as $D \subseteq \Cbb$ is again obvious.    
    Given $\cvec \in \ext \Cbb$,
    suppose the set of rows that $\Abf_\cvec$ selected from the first and second block of $\Abf$ are $S$ and $T$, respectively.
    Both $S$ and $T$ are subsets of $[k]$,
    indexed against $\Abf$. We also denote the set of remaining indices as 
        $R = [k] \backslash T$.
        
    In the case of the squared loss metric, we observe that every row in the first block of $\Abf$ 
    can be written as a linear combination of two other rows in the first block and the sum-to-one row from the third block. This follows the corresponding relation in continuous squared functions:
    \[
    (x - a)^2 = x^2 -2ax - a^2 = \alpha (x^2 -2bx + b^2) + \beta (x^2 -2cx + c^2) + \gamma = \alpha (x- b)^2 + \beta (x-c)^2 + \gamma,
    \]
    for some value of $\alpha, \beta$, and $\gamma$. Therefore, $S$ can only include one or two elements. This means that $R$ must also contain one or two elements.
    We consider these two cases:
    \begin{enumerate}%[noitemsep,topsep=0pt]
        \item $S$ contains a single element $\{i\}$.\\
        In this case, $R$ must also be $\{i\}$. If $R = \{j\}$ where $j \neq i$, the equation subsystem requires $v = \Lbf_{(i,:)} \qvec = (i-j)^2 \ge 1$, since by definition of $R$, $q_j = 1$ and $q_l = 0$ for all $l \in [k] \backslash j$. However, this contradicts with the requirement of the $j$-th row of $\Abf$ that $v \le \Lbf_{(j,:)} \qvec = 0$. Finally, it is easy to check that the vector in $\Rbb^{k+1}$ that meet all the equalities encoded in this $\Abf_\cvec$ is $\cvec = \begin{bmatrix}
    			\evec_i \\
    			0
    		\end{bmatrix} $. Obviously $\cvec \in D$.
        
        \item $S$ contains two elements.\\
        The rank condition requires that $R$ must also contains two elements $\{i,j\}$. Consider these following sub-cases:
        \begin{enumerate}[label=\alph*)] %.,noitemsep,topsep=0pt]
            \item $S = \{l-1, l\}$, where $i < l \le j$. \\
            Let $\cvec = \begin{bmatrix}
    			\qvec \\
    			v
    		\end{bmatrix} $ be the solution of the equalities encoded in this $\Abf_\cvec$. By definition of $R$, $q_l = 0$ for all $q \in [k] \backslash \{i,j\}$. The value of $q_i$ and $q_j$ can be calculated by solving $\Lbf_{(l-1,:)} \qvec = \Lbf_{(l,:)} \qvec$ or equivalently $\Lbf_{(l-1,i)} q_i + \Lbf_{(l-1,j)} q_j = \Lbf_{(l,i)} q_i + \Lbf_{(l,j)} q_j$, with the constraint that $q_i + q_j = 1$ and the non-negativity constraints. 
    		Solving for this equation resulting in the following $q_i$, $q_j$, and $v$:
            \begin{align}
            q_i &= \frac{  \Lbf_{(l-1,j)} -  \Lbf_{(l,j)}   }
        	  { \Lbf_{(l,i)} -  \Lbf_{(l-1,i)} + \Lbf_{(l-1,j)} -  \Lbf_{(l,j)} }   \\
        	  &=  \frac{  (j - l +1)^2 -  (j-l)^2   }
        	  { (l-i)^2 -  (l-1-i)^2 + (j-l+1)^2 -  (j-l)^2 } =
        	  \frac{ 2(j - l) + 1}{2 \left(j - i \right)}, \\
        	 q_j &= \frac{  \Lbf_{(l,i)} -  \Lbf_{(l-1,i)}    }
        	  { \Lbf_{(l,i)} -  \Lbf_{(l-1,i)} + \Lbf_{(l-1,j)} -  \Lbf_{(l,j)} }   \\
        	  &=  \frac{  (l - i)^2 -  (l-1-i)^2   }
        	  { (l-i)^2 -  (l-1-i)^2 + (j-l+1)^2 -  (j-l)^2 } =
        	  \frac{ 2(l - i) - 1}{2 \left(j - i \right)}, \\
        	  v &= \frac{ \left( \Lbf_{(l-1,j)} -  \Lbf_{(l,j)} \right) \Lbf_{(l,i)}  + 
        	  \left( \Lbf_{(l,i)} -  \Lbf_{(l-1,i)} \right)  \Lbf_{(l,j)}  }
        	  { \Lbf_{(l,i)} -  \Lbf_{(l-1,i)} + \Lbf_{(l-1,j)} -  \Lbf_{(l,j)} } \\
        	  &=  \frac{  \left( 2(j - l) + 1 \right) (l-i)^2 + \left( 2(l - i) - 1 \right) (j-l)^2 }
        	  { 2 \left(j - i \right) } .
            \end{align}
            It is obvious that $\cvec \in D$.
            
            \item $S = \{m, l\}$, where $i \le m < l \le j$ and $m \neq l-1$. \\
            We want to show that this case is impossible. 
            Solving for the $m$-th and the $l$-th equality, $v= \Lbf_{(m,i)} q_i + \Lbf_{(m,j)} q_j = \Lbf_{(l,i)} q_i + \Lbf_{(l,j)} q_j$ resulting in $q_i = \frac{1}{z} [(j-m)^2 - (j-l)^2]$, 
            $q_j = \frac{1}{z} [(l-i)^2 - (m-i)^2]$, and 
            \[v = \tfrac{1}{z} \cbr{  (l-i)^2 [(j-m)^2 - (j-l)^2] + (j-l)^2 [(l-i)^2 - (m-i)^2] }, \]
            where $z = [(j-m)^2 - (j-l)^2] + [(l-i)^2 - (m-i)^2]$. 
            
            Let $o$ be an index such that $m<o<l$. This row must exist since $m \neq l-1$ and $m<l$. Applying the solution above to the $o$-th row, we define:
            \[ w \triangleq \Lbf_{(o,:)} \qvec = \tfrac{1}{z} \cbr{  (o-i)^2 [(j-m)^2 - (j-l)^2] + (j-o)^2 [(l-i)^2 - (m-i)^2] }.\]
            Then,
            \begin{align} v-w = &\tfrac{1}{z} \big\{  [(l-i)^2 - (o-i)^2] [(j-m)^2 - (j-l)^2] \\
                & \quad - [(j-o)^2 - (j-l)^2 ] [(l-i)^2 - (m-i)^2 ] \big\}.
            \end{align}
            This means that $v - w > 0$, since for all $i \le m <o < l \le j$, $i,j,l,m,o \in [k]$,
            \[ \frac{(l-i)^2 - (o-i)^2}{(l-i)^2 - (m-i)^2} > \frac{(j-o)^2 - (j-l)^2}{(j-m)^2 - (j-l)^2}. \]
            
            Thus, it contradicts with the requirement that $v \leq \Lbf_{(o,:)}$.
            
            % TODO: working on a concise way to show this. i.e. the $o$-th strict inequality cannot be satisfied, where $m<o<l$.
            
            \item $S = \{m, l\}$, where $m < i$ or $l > j$.\\
            We first show that $m < i$ is impossible. Note that for $m < i$, the loss value $\Lbf_{(m,i)} = (i-m)^2 > \Lbf_{(i,i)} = 0$ and $\Lbf_{(m,j)} = (j-m)^2 > \Lbf_{(i,j)} = (j-i)^2$. Noting that at least one of $q_i$ or $q_j$ must be positive due to sum-to-one constraint, we conclude that $\Lbf_{(m,:)} \qvec > \Lbf_{(i,:)} \qvec$. But this contradicts with $\Lbf_{(m,:)} \qvec = v \le \Lbf_{(i,:)} \qvec$ since the $m \in S$. Similarly, $l > j$ is also impossible.
            
        \end{enumerate}

    \end{enumerate}

    Therefore, in all possible cases, we have shown that any $\cvec$ in $\ext \Cbb$ must be in $D$, which concludes our proof.
\end{proof}

Note that AL$^\text{sq}$ contains two separate maximizations corresponding to the case where there are two non-zero elements of $\qvec$ and the case where only a single non-zero element of $\qvec$ is possible.
Unlike the surrogate for absolute loss, the maximization in AL$^\text{sq}$ cannot be realized independently. A $\Ocal(k^3)$ algorithm is needed to compute the maximization for the case that two non-zero elements of $\qvec$ are allowed, and a $\Ocal(k)$ algorithm is needed to find the maximum potential in the case of a single non-zero element of $\qvec$. Therefore, the total runtime of the algorithm for computing AL$^\text{sq}$ is $\Ocal(k^3)$. 
The loss surface of AL$^\text{sq}$ for the three classes classification is shown in Figure \ref{fig:advloss-sq}.

\begin{figure*}[ht]
\centering
	\begin{minipage}{.33\linewidth}
		\centering
		\subfloat[]{\label{fig:sqy1}\includegraphics[width=1.0\linewidth]{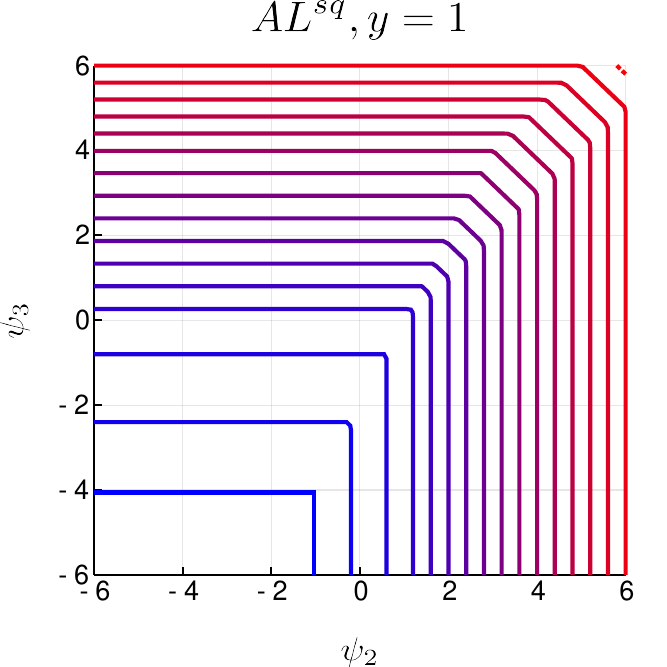}}
	\end{minipage}%
	\begin{minipage}{.33\linewidth}
		\centering
		\subfloat[]{\label{fig:sqy2}\includegraphics[width=1.0\linewidth]{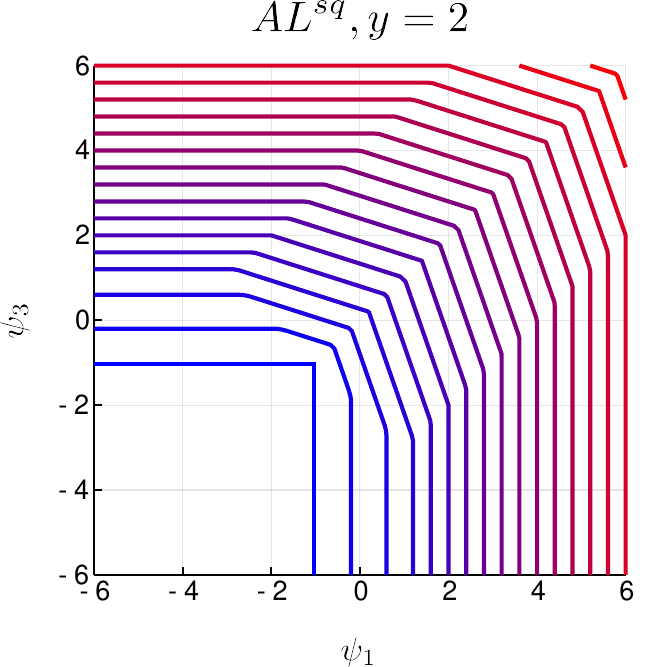}}
	\end{minipage}%
	\begin{minipage}{0.33\linewidth}
		\centering
		\subfloat[]{\label{fig:sqy3}\includegraphics[width=1.0\linewidth]{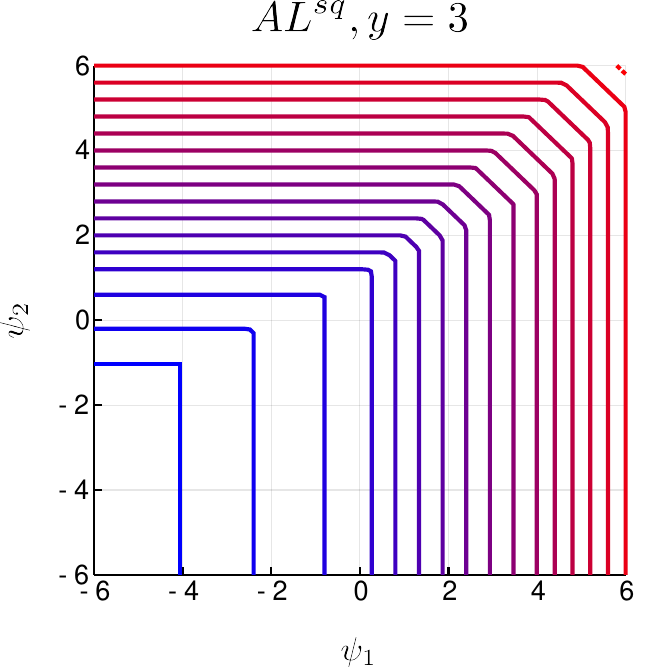}}
	\end{minipage}		
	\caption{
Loss function contour plots of $\text{AL}^{\text{sq}}$ over the space of potential differences $\psi_j \triangleq f_j - f_{y}$ for the prediction task with three classes when the true label is $y = 1$ (a), $y = 2$ (b), and $y = 3$ (c).
		}
	\label{fig:advloss-sq}
\end{figure*}

\subsection{Weighted Multiclass Loss}

In more general prediction tasks, the penalty metric for each sample may be different. For example, the predictor may need to prioritize samples with a particular characteristic. In this subsection, we study the adversarial surrogate loss for weighted multiclass loss,
and in particular, the setting with a standard loss metrics weighted by parameter $\alpha$ (for example, the weighted absolute loss: $\text{loss}(\hat{y},y) = \alpha  |\hat{y} - y |$). We next analyze in Theorem \ref{thm:ermloss-weighted} the extreme points of the polytope formed by the the constraints in Eq. \eqref{eq:al-lp} when $\Lbf$ is the weighted multiclass loss metric. 

\begin{theorem}
	\label{thm:ermloss-weighted}
Let ${\bf q}^*$, and $v^*$ be the solution of the adversarial maximin (Eq. \eqref{eq:al-lp}) with $\Lbf$ as the loss matrix, then if the loss matrix is $\alpha\Lbf$, the solution of (Eq. \eqref{eq:al-lp}) is ${\bf q}^\diamond ={\bf q}^*$, $v^\diamond = \alpha v^*$.
\end{theorem}
\begin{proof}
Multiplying both sides of the constraints $\Lbf_{(i,:)} \qvec \ge v$ in Eq. \eqref{eq:al-lp} and employing $\alpha \Lbf_{(i,:)} \qvec \ge \alpha v$, we arrive at an equivalent LP problem with the same solution. Therefore, if we replace the original loss metric with $\alpha \Lbf$, then the solution for $\qvec$ remain the same, and the optimum slack variable value is $\alpha v^*$.
\end{proof}

Using Theorem \ref{thm:ermloss-weighted}, we can derive the adversarial surrogate loss for weighted multiclass zero-one loss, absolute loss, and squared loss metrics as stated below.

\begin{corollary}
	\label{crl:ermloss-zo-weighted}
	An adversarial multiclass predictor with weighted zero-one loss is obtained by choosing 
% 	the parameter ${\theta}$
	a predictive function
that minimizes the empirical risk of the surrogate loss function:
	\[
	\text{AL}^{\text{0-1-w}}
	({\bf f}, y, \alpha) = \, 
	\max_{
    {S \subseteq
			[k], %\\ 
            \; S \neq \emptyset } }
	\frac{\sum_{i \in S} f_i
		+ \alpha \left( |S|-1 \right) }{|S|} - f_y.
	\]
\end{corollary}

\begin{corollary}
	\label{crl:ermloss-abs-weighted}
An adversarial ordinal classification predictor with weighted absolute loss is obtained by choosing %the parameter ${\theta}$
a predictive function
that minimizes the empirical risk of the surrogate loss function:
\[ \text{AL}^{\text{ord-w}}
	({\bf f}, y, \alpha) = 
	  \max_{i,j \in [k] }  \frac{f_i + f_j + \alpha \left( j - i \right)}{2} 
      - f_y .
\]
\end{corollary}

\begin{corollary}
	\label{crl:ermloss-sq-weighted}
An adversarial ordinal classification predictor with  weighted squared loss is obtained by choosing %the parameter ${\theta}$ 
a predictive function
that minimizes the empirical risk of the surrogate loss function:
\begin{align}
	\text{AL}^{\text{sq-w}}
	({\bf f}, y,\alpha) = 
	  \max \Bigg\{ 
      \max_{\substack{i,j,l \in [k] \\ i < l \leq j }}  \!
	  \tfrac{ \left( 2(j - l) + 1 \right) \left[ f_i + \alpha\left( l - i \right)^2 \right] + 
	  \left( 2(l - i) - 1 \right) \left[ f_j + \alpha\left( j - l \right)^2 \right] }
	  {2 \left(j - i \right)},\; \max_i f_i\; 
	  \Bigg\}
      - f_y .
\end{align}
\end{corollary}

The computational cost of calculating the adversarial surrogates for weighted multiclass loss metric above is the same as that for the non-weighted counterpart of the loss, i.e., $\Ocal(k \log k)$ for AL$^{\text{0-1-w}}$, $\Ocal(k)$ for AL$^{\text{ord-w}}$, and $\Ocal(k^3)$ for AL$^{\text{sq-w}}$. The weight constant $\alpha$ does not change the runtime complexity.

\subsection{Classification with Abstention}
\label{sec:class_abstain}

In some prediction tasks, it might be better for the predictor to abstain without making any prediction rather than making a prediction with high uncertainty for borderline samples. 
Under this setting, the standard zero-one loss is used for the evaluation metric with the addition that the predictor can choose an abstain option and suffer a penalty of $\alpha$.
The adversarial surrogate loss for classification with abstention is defined in Eq. \eqref{eq:al-lp}, where $\Lbf$ is the \textit{abstain} loss matrix (e.g. Figure \ref{fig:lm-abstain} for a five-class classification with $\alpha = \frac12$).
The constraints in Eq. \eqref{eq:al-lp} form a convex polytope $\Cbb$. Below is the example of the half-space representation of the polytope for a four-class classification problem with abstention.
\begin{align}
\begin{array}{c} 
\phantom{0} \\ \phantom{0} \\ \text{1st block} \\ \phantom{0}\\\phantom{0} \\\hdashline[2pt/2pt] \phantom{0}  \\ \phantom{0} \\ \text{2nd block} \\ \phantom{0} \\ \hdashline[2pt/2pt] \phantom{0} \\ \text{3rd block} \end{array}     
\left[
\begin{array}{ccccc}
0 & 1 & 1 & 1 & -1 \\ 1 & 0 & 1 & 1 & -1 \\ 1 & 1 & 0 & 1 & -1 \\ 1 & 1 & 1 & 0 & -1 \\ \alpha & \alpha & \alpha & \alpha & -1 \\ \hdashline[2pt/2pt]
1 & 0 & 0 & 0 & 0 \\ 0 & 1 & 0 & 0 & 0 \\ 0 & 0 & 1 & 0 & 0 \\ 0 & 0 & 0 & 1 & 0 \\ \hdashline[2pt/2pt]
1 & 1 & 1 & 1 & 0 \\ -1 & -1 & -1 & -1 & 0 
\end{array} 
\right]
\begin{bmatrix} q_1 \\ q_2 \\ q_3 \\ q_4 \\ v \end{bmatrix}
\ge 
\begin{bmatrix} 0 \\ 0 \\ 0 \\ 0 \\ 0 \\ 0 \\ 0 \\ 0 \\ 0 \\ 1 \\ -1 \end{bmatrix} .   
\end{align}
Note that the first block of the coefficient matrix $\Abf$ has $k+1$ rows (one additional row for the abstain option).

We design a convex surrogate loss that can be generalized to the case where $0 \leq \alpha \leq \frac12$. We define the adversarial surrogate loss for classification with abstention $\text{AL}^{\text{abstain}}$ as stated in Theorem \ref{thm:ermloss-abstain} below.

\begin{theorem}
	\label{thm:ermloss-abstain}
An adversarial predictor for classification with abstention with the penalty for abstain option is $\alpha$ where $0 \leq \alpha \leq \frac12$, is obtained by choosing 
%the parameter ${\theta}$ 
a predictive function
that minimizes the empirical risk of the surrogate loss function:
\begin{align}
	& \text{AL}^{\text{abstain}}
	({\bf f}, y, \alpha) = 
	  \max \left\{
	  \max_{i,j \in [k], i \neq j }  \left(1 - \alpha \right) f_i + \alpha f_j + \alpha, \; \max_i f_i\;  \right\}
      - f_y \label{eq:al-abstain}.
\end{align}
\end{theorem}

\begin{proof}
% 	\textbf{(New proof)}
    The $\text{AL}^{\text{abstain}}$ above corresponds to the set of extreme points 
    \begin{align}
    	D = \cbr{\begin{bmatrix}
    		\qvec \\
    		v
    		\end{bmatrix}
    		=
    		(1 - \alpha)
    		\begin{bmatrix}
    			\evec_i \\
    			0
    		\end{bmatrix} +
    		\alpha
    		\begin{bmatrix}
    			\evec_j \\
    			1
    		\end{bmatrix} 
    		\ \middle \vert \  \begin{matrix}i, j \in [k]\\  i \neq j\end{matrix} } 
    		\cup \cbr{
    		    \begin{bmatrix}
    		\qvec \\
    		v
    		\end{bmatrix}
    		= \begin{bmatrix}
    			\evec_i \\
    			0
    		\end{bmatrix} 
    		 \middle \vert \ i \in [k]
    		}.
    \end{align}
%    where $\evec_i \in \RR^k$ is the $i$-th canonical vector with a single 1 at the $i$-th coordinate and 0 elsewhere.
    This means $\qvec$ can only have one non-zero element with probability of one or two non-zero elements with the probability of $\alpha$ and $(1-\alpha)$. 
    
    Similar to the proof of Theorem \ref{thm:ermloss-zo}, 
    we next prove that $D \supseteq \ext \Cbb = \{\cvec \in \Cbb : \rank(\Abf_\cvec) = k+1\}$,
    as $D \subseteq \Cbb$ is again obvious.    
    Given $\cvec \in \ext \Cbb$,
    suppose the set of rows that $\Abf_\cvec$ selected from the first and second block of $\Abf$ are $S$ and $T$, respectively.
    Now $S$ is a subset of $[k+1]$ where the $(k+1)$-th index represents the abstain option, while 
    $T$ is a subset of $[k]$,
    indexed against $\Abf$. 
    Similar to the case of zero-one loss metric, 
    %as explained in the proof of Theorem 5, disregarding the abstain option, 
    $S$ and $T$ must be disjoint. % due to the same reason.
    We also denote the set of remaining indices as 
         $R = [k] \backslash T$.
    
    The abstain row in the first block of $\Abf$ implies that $v \le \alpha$, while including $j$ regular rows to $S$ implies that $v = \frac{j-1}{j}$. Therefore, only a single regular row can be in $S$ when $\alpha < \frac12$ or at most two regular rows can be in $S$ when $\alpha = \frac12$.
    
    We first consider $\alpha < \frac12$. Let $S = \{i,k+1\}$, i.e., one regular row and one abstain row. Due to rank requirement of $\Abf_\cvec$ and the disjointness of $S$ and $T$, $R$ must contain two elements with one of them be $i$, i.e. $R = \{i,j\}$. 
    To get the value of $q_i$ and $q_j$, we solve for the equation $\Lbf_{(i,:)} \qvec = \Lbf_{(k+1,:)} \qvec$ which can be simplified as $q_j = \alpha q_i + \alpha q_j$. The solution is to set $q_i = (1-\alpha)$, $q_j = \alpha$, and $v = \alpha$, which obviously in $D$.
    For the second case, let $S = \{i\}$, i.e., one regular row. In this case $R$ must be $\{i\}$ too. This yields $\cvec$ with $q_i = 1$, $q_j = 0, \forall j \in [k]\backslash i$, and $v=0$. Obviously $\cvec \in D$.  
    
    For the case where $\alpha = \frac12$, two cases above still apply with two additional cases. 
    First, $S = \{i,j\}$, i.e., two regular rows. In this case, $R$ must be $\{i,j\}$ too. The solution is to set $q_i = q_j = \frac12$, and $v = \frac12$. This satisfies $v = \Lbf_{(i,:)} \qvec = \Lbf_{(j,:)} \qvec = \frac12$ as well as $v \le \Lbf_{(k+1,:)} \qvec = \alpha = \frac12$. Obviously, this is in $D$.
    Second, $S = \{i,j,k+1\}$, i.e., two regular rows and one abstain row. Due to the rank requirement of $\Abf_\cvec$, and the disjointness of $S$ and $T$, $R$ must contain three elements:, $i,j,$ and another index $l \in [k]\backslash \{i,j\}$. It is easy to check that the solution in this case is also to set $q_i = q_j = \frac12$, and $v = \frac12$. This satisfies $v = \Lbf_{(i,:)} \qvec = \Lbf_{(j,:)} \qvec = \frac12$ as well as $v = \Lbf_{(k+1,:)} \qvec = \alpha = \frac12$.
    
    Therefore, in all possible cases, we have shown that any $\cvec$ in $\ext \Cbb$ must be in $D$.
\end{proof}

We can view the maximization in AL$^{\text{abstain}}$ as the maximization over $k^2$ linear hyperplanes, with $k$ hyperplanes are defined by the case where only a single element of $\qvec$ can be non zero and the rest $k(k-1)$ hyperplanes are defined by the case where two  elements of $\qvec$ are non zero. 
For the binary classification with abstention problem, the surrogate loss function has four facets. Figure \ref{fig:adv2-abstain} shows the loss function in the case where $\alpha = \frac13$ and $\alpha = \frac12$. 
Note that for $\alpha = \frac12$ the facet corresponds with the hyperplane of $(1-\alpha) f_1 + \alpha f_2 + \alpha$ collide with the facet corresponds with the hyperplane of $(1-\alpha) f_2 + \alpha f_1 + \alpha$, resulting in a loss function with only three facets.
For the three-class classification with abstention problem, the surrogate loss has nine facets with different shapes for each true label value, as shown in Figure \ref{fig:advloss-abstain} for $\alpha = \frac13$ and $\alpha = \frac12$. Similar to the binary classification case, for $\alpha = \frac12$, some facets in the surrogate loss surface collide resulting in a surrogate loss function with only six facets.

\begin{figure}[ht]
\centering	
\begin{minipage}{.5\linewidth}
		\centering
		\subfloat[]{\label{fig:abstain2-1}\includegraphics[width=0.8\linewidth]{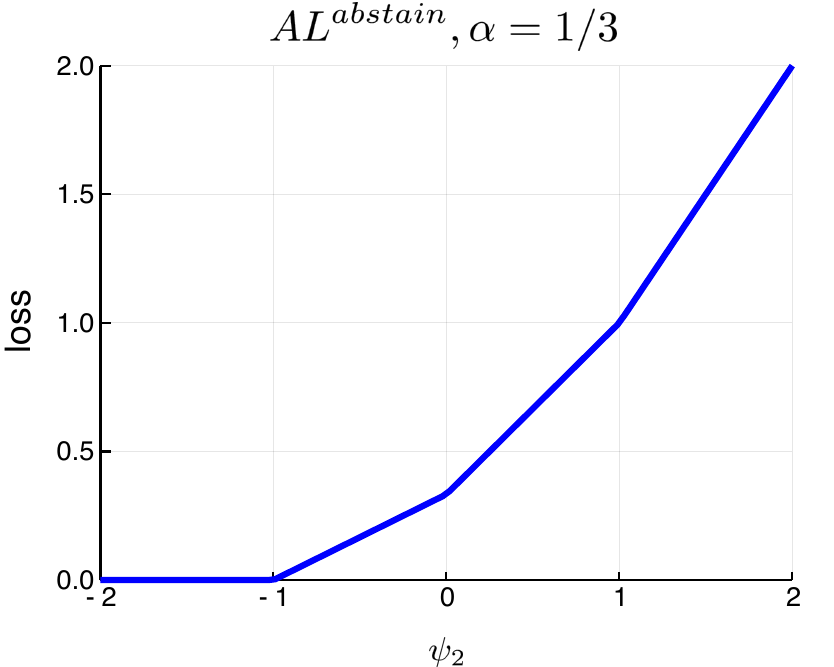}}
	\end{minipage}%
	\begin{minipage}{.5\linewidth}
		\centering
		\subfloat[]{\label{fig:abstain2-2}\includegraphics[width=0.8\linewidth]{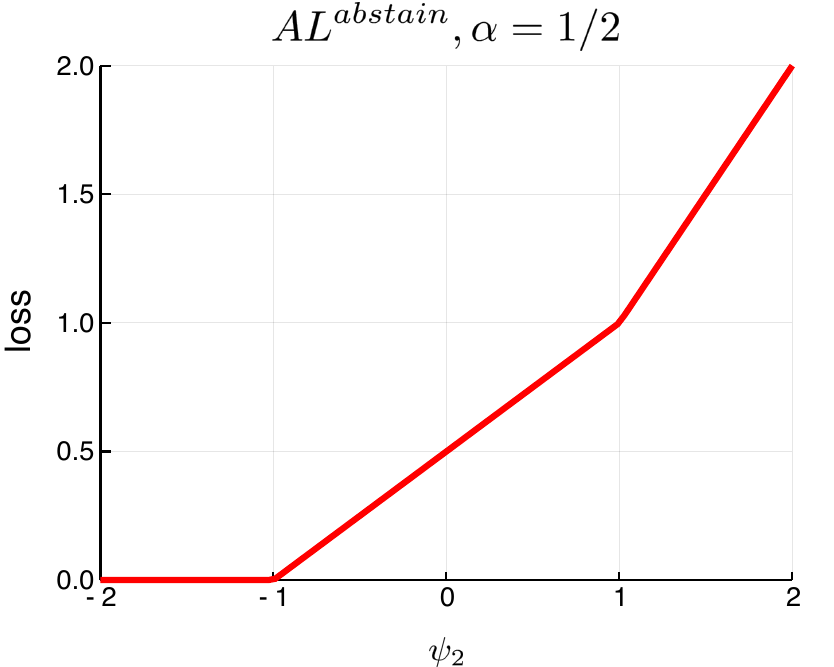}}
\end{minipage}
\caption{$\text{AL}^{\text{abstain}}$ evaluated over the space of potential 
differences 
($\psi_{i, y} = f_{i} - f_{y}$; 
and $\psi_{i,i} = 0$) 
for binary prediction tasks when the true 
label is $y = 1$, where $\alpha = \frac13$ (a), and $\alpha = \frac12$ (b).}
	\label{fig:adv2-abstain}
\end{figure}

\begin{figure*}[ht]
\centering
	\begin{minipage}{.5\linewidth}
		\centering
		\subfloat[]{\label{fig:abstain3-3}\includegraphics[width=0.75\linewidth]{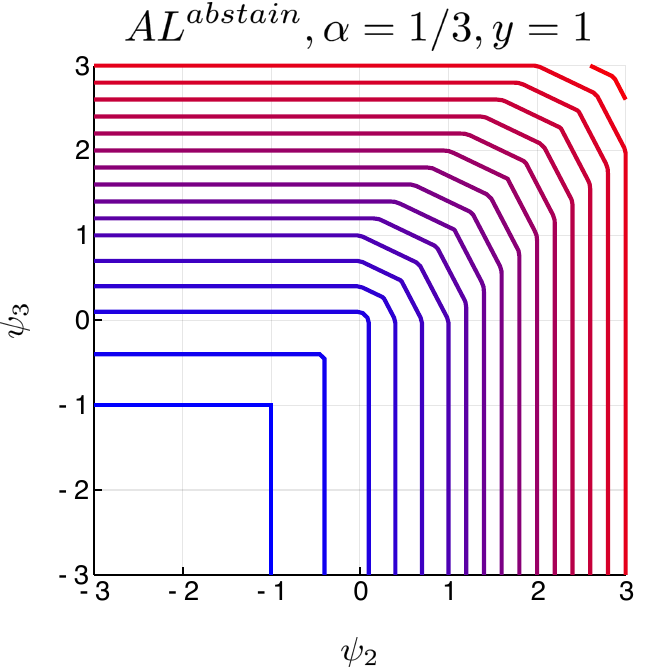}}
	\end{minipage}%
	\begin{minipage}{.5\linewidth}
		\centering
		\subfloat[]{\label{fig:abstain3-5}\includegraphics[width=0.75\linewidth]{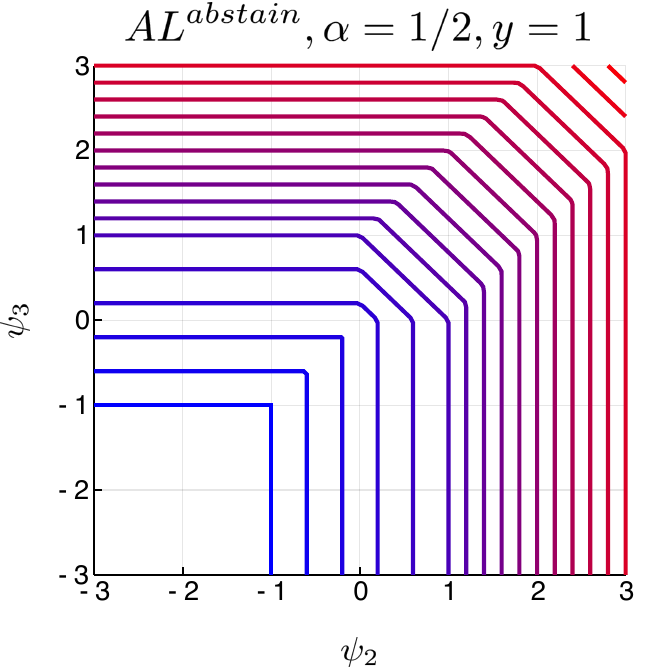}}
	\end{minipage}
	\caption{
Loss function contour plots of $\text{AL}^{\text{abstain}}$ over the space of potential differences $\psi_j \triangleq f_j - f_{y}$ for the prediction task with three classes when the true label is $y = 1$, where $\alpha = \frac13$ (a), and  $\alpha = \frac12$ (b).
		}
	\label{fig:advloss-abstain}
\end{figure*}

Even though the maximization in $\text{AL}^{\text{abstain}}$ is over $n^2$ different items, we construct a faster algorithm to compute the loss. The algorithm keeps track of the two largest potentials as it scans all $k$ potentials. Denote $i^*$ and $j^*$ as the index of the best and the second best potentials respectively. The algorithm then takes the maximum of two candidate solutions: (1) assigning all the probability to $f_{i*}$, resulting in the loss value of $f_{i*}$, or (2) assigning $1-\alpha$ probability to $f_{i*}$ and $\alpha$ probability to $f_{j*}$, resulting in the loss value of $(1-\alpha) f_{i^*} + \alpha f_{j^*} + \alpha$.
The runtime of this algorithm is $\Ocal(k)$ due to the need to scan all $k$ potentials once.

\subsection{General Multiclass Loss}

For a general multiclass loss matrix $\Lbf$, the extreme points of the polytope defined by the constraints in Eq. \eqref{eq:al-lp} may not be easily characterized. Nevertheless, since the maximization in Eq. \eqref{eq:al-lp} is in the form of a linear program (LP), some well-known algorithms for linear programming can be used to solve the problem. The techniques for solving LPs have been extensively studied, resulting in two major  algorithms:

\begin{enumerate}
    \item Simplex algorithm. \\
    The simplex algorithm \citep{dantzig1948programming,dantzig1963linear} cleverly visits the extreme points in the convex polytope until it reaches the one that maximizes the objective. This is  the most popular algorithm for solving LP problems. However, although the algorithm typically works well in practice, the worst case complexity of the algorithm is exponential in the problem size. 
    \item Interior point algorithm.\\
    \citet{karmarkar1984new} proposed an interior point algorithm for solving LPs with  polynomial worst case runtime complexity. The algorithm  finds the optimal solution by traversing the interior of the feasible region. The runtime complexity of Karmarkar's algorithm for solving the LP is $\Ocal(n^{3.5})$ where $n$ is the number of variables in the LP problem.
    In  Eq. \eqref{eq:al-lp}, $n = k+1$.
\end{enumerate}

Therefore, using Karmarkar's algorithm we can bound the worst-case runtime complexity of computing the adversarial surrogate for arbitrary loss matrix $\Lbf$ with $\Ocal(k^{3.5})$ where $k$ is the number of classes.

\section{Prediction Formulation}

The dual formulation of the adversarial prediction (Eq. \eqref{eq:dual}) provides a way to construct a learning algorithm for the framework. The learning step in the adversarial prediction is to find the optimal Lagrange dual variable $\theta^* = \min_{\theta} \; \mathbb{E}_{{\bf X},{Y}\sim \tilde{P}} \; \left[ AL({\bf X}, Y, \theta) \right]$.
%that minimizes Eq. \eqref{eq:dual}. 
In the prediction step, we use the optimal $\theta^*$ to make a label prediction given newly observed data.

\subsection{Probabilistic Prediction}

Given a new data point $\xvec$ and its label $y$, and the optimal $\theta^*$, we formulate the prediction minimax game based on Eq. \eqref{eq:dual} by flipping the optimization order between the predictor and the adversary player:
\begin{align}
	 \min_{\hat{P}(\hat{Y}|{\bf x})}
     \max_{\check{P}(\check{Y}|{\bf x})}
    \mathbb{E}_{\hat{Y}|{\bf x}\sim\hat{P};
    \check{Y}|{\bf x}\sim{\check{P}}}\left[ \text{loss}(\hat{Y},\check{Y})
    + {\theta^*}^\intercal \left( \phi({\bf x},\check{Y}) - \phi({\bf x},{y}) \right)
    \right].
\end{align}
This flipping is enabled by the strong minimax duality theorem \citep{von1945theory}. Denoting $f_i = {\theta^*}^\intercal \phi(\xvec, i)$, the prediction formulation can be written in our vector and matrix notation as:
\begin{align}
    \min_{\pvec \in \Delta}
    \max_{\qvec \in \Delta}
    \pvec^\intercal \Lbf \qvec 
    + \fvec^\intercal \qvec
    - f_{y}. \label{eq:predict-vector-y}
\end{align}
Even though the ground truth label $y$ serves an important role in the learning step (Eq. \eqref{eq:dual}), it is constant with respect to the predictor probability $\pvec$. Therefore, to get the optimal prediction probability $\pvec^*$, the term $f_y$ in Eq. \eqref{eq:predict-vector-y} can be removed, resulting in the following probabilistic prediction formulation:
\begin{align}
    \pvec^* = 
    \argmin_{\pvec \in \Delta}
    \max_{\qvec \in \Delta}
    \pvec^\intercal \Lbf \qvec 
    + \fvec^\intercal \qvec. \label{eq:predict-vector}
\end{align}

\subsection{Non-probabilistic Prediction}

In some prediction tasks, a learning algorithm needs to provide a single class label prediction rather than a probabilistic prediction. 
We propose two prediction schemes to get a non-probabilistic single label prediction $y^*$ from our formulation.
\begin{enumerate}
    \item The maximizer of the potential $\fvec$.\\
    This follows the standard prediction technique used by many ERM-based models, e.g., SVM. Given the best parameter $\theta^*$, the predicted label is computed by choosing the label that maximizes the potential value, i.e.,
    \begin{equation}
        y^* = \argmax_i f_i, \qquad \text{where: } f_i = {\theta^*}^\intercal \phi(\xvec, i).
    \end{equation}
    Note that this prediction scheme works for the prediction settings where the predictor employs the same set of class labels as the ground truth, i.e., $y^* \in \Ycal$ and $y \in \Ycal$ where $\Ycal = [k]$. If they are different such as in the classification task with abstention, this prediction scheme cannot be used. 
    The runtime complexity of this prediction scheme is $\Ocal(k)$ for $k$ classes.
    
    \item The maximizer of the predictor's optimal probability $\pvec^*$.\\
    This prediction scheme requires the predictor to first produce a probabilistic prediction by using Eq. \eqref{eq:predict-vector}. Then the algorithm chooses the label that maximizes the conditional probability, i.e.,
    \begin{align}
    y^* = \argmax_i p^*_i, \qquad \text{where: } \pvec^* = \argmin_{\pvec \in \Delta}
    \max_{\qvec \in \Delta}
    \pvec^\intercal \Lbf \qvec 
    + \fvec^\intercal \qvec.
    \end{align}
    This prediction scheme can be applied to more general problems, including the case where the predictor and ground truth class labels are chosen from different sets of labels. This is useful for  the classification task with abstention. However, for a general loss matrix $\Lbf$, this prediction scheme is more computation intensive than the potential-based prediction, i.e., $\Ocal(k^{3.5})$ due to the need of solving the minimax game by linear programming (Karmarkar's algorithm).
\end{enumerate}

\subsection{Prediction Algorithm for Classification with Abstention}
\label{sec:pred_abstain}

In the task of classification with abstention, the standard prediction scheme using the potential maximizer $\argmax_i f_i$ cannot be applied due to the additional abstain option of the predictor. In this subsection, we construct a fast prediction scheme that is based on the predictor's optimal probability in the minimax game (Eq. \eqref{eq:predict-vector}) without the need to use general purpose LP solver. The minimax game in Eq. \eqref{eq:predict-vector} can be equivalently written in the standard LP form as:
\begin{align}
    \min_{\pvec, v} & \;
    v \label{eq:lp-pred}\\
    \text{s.t.:} & \;  v \ge {\Lbf_{(:,i)}}^\intercal \pvec + f_i , \quad \forall i \in [k] \nonumber \\
    & \; \pvec \in \RR_+^{k+1}, \nonumber \\
    & \; \pvec^\intercal {\bf 1}  = 1 \nonumber ,
\end{align}
where $v$ is a slack variable to convert the inner maximization into linear constraints, and $\Lbf_{(:,i)}$ denotes the $i$-th column of the loss matrix $\Lbf$. We aim to analyze the optimal $\pvec$ and $v$ for the case where $\Lbf$ is the loss matrix for classification with abstention, e.g.,
\[
\Lbf = 
        \begin{bmatrix}
        0 & 1 & 1 & 1 \\
        1 & 0 & 1 & 1 \\
        1 & 1 & 0 & 1 \\
        1 & 1 & 1 & 0 \\
        \alpha & \alpha & \alpha & \alpha
        \end{bmatrix}
\]
in a four-class classification, where $\alpha$ is the penalty for abstaining (c.f. Section \ref{sec:class_abstain}). Similar to the case of the adversarial surrogate loss for classification with abstention, our analysis can be generalized to the case where $0 \le \alpha \le \frac12$.

\begin{theorem}
% \textbf{(New statement)}
Let $\alpha$ be the penalty for abstaining where $0 \leq \alpha \leq \frac12$, $\theta^*$ be the learned parameter, and $\fvec$ be the potential vector for all classes where $f_i = {\theta^*}^\intercal\phi(\xvec, i)$.
Given a new data point $\xvec$, 
let $i^* = \argmax_i f_i$ (break tie arbitrarily),
$j^* = \argmax_{j \neq i^*} f_j$,
and $\evec_{i^*} \in \RR^k$ be the $i^*$-th cannonical vector.
Then the predictor's optimal probability $\pvec^*$ of Eq. \eqref{eq:lp-pred} for the task of classification with abstention can be directly computed as:
\begin{align}
    \pvec^* = \begin{bmatrix}
      \evec_{i^*} \\
      0
    \end{bmatrix}
    \text{ if } f_{i^*} - f_{j^*} \ge 1 \qquad \text{and} \qquad
    \pvec^* = \begin{bmatrix}
    (f_{i^*} - f_{j^*})\evec_{i^*} \\
    1- f_{i^*} + f_{j^*}
    \end{bmatrix}
    \text{ if } f_{i^*} - f_{j^*} < 1.
\end{align}
\end{theorem}

\begin{proof}
% \textbf{(New proof)}
Based on Theorem \ref{thm:ermloss-abstain},
the optimal objective value of \eqref{eq:lp-pred} is exactly $\text{AL}^{\text{abstain}}(\fvec, y, \alpha)+f_y$,
which is $f_{i^*}$ when $f_{i^*} - f_{j^*} \ge 1$,
and $\alpha + (1-\alpha) f_{i^*} + \alpha f_{j^*}$ otherwise.
So we only need to verify that the $\pvec^*$ given in the theorem attains these two values,
or equivalently, $\max_i \cbr{{\Lbf_{(:,i)}}^\intercal \pvec^* + f_i}$ attains these two values.

\begin{enumerate}
    \item Case 1: $f_{i^*} - f_{j^*} \ge 1$.
Now $\pvec^* = \begin{bmatrix}
      \evec_{i^*} \\
      0
    \end{bmatrix}$
    renders 
    ${\Lbf_{(:,i^*)}}^\intercal \pvec^* + f_{i^*} = f_{i^*}$,
and 
    ${\Lbf_{(:,k)}}^\intercal \pvec^* + f_{k} = 1 + f_{k} \le f_{i^*}$ for all $k \neq i^*$.
So the objective of \eqref{eq:lp-pred} matches $\text{AL}^{\text{abstain}}+f_y$.

    \item Case 2: $f_{i^*} - f_{j^*} < 1$.
Now $\pvec^* = \begin{bmatrix}
    (f_{i^*} - f_{j^*})\evec_{i^*} \\
    1- f_{i^*} + f_{j^*}
    \end{bmatrix} \in \RR_+^{k+1}$ and $\one^\intercal \pvec^* = 1$.
    Furthermore,
\begin{align*}
    {\Lbf_{(:,i^*)}}^\intercal \pvec^* + f_{i^*} &= \alpha(1- f_{i^*} + f_{j^*}) + f_{i^*},\\
    {\Lbf_{(:,k)}}^\intercal \pvec^* + f_{k} &= 
    f_{i^*} - f_{j^*} + \alpha(1- f_{i^*} + f_{j^*}) + f_{k}
    \le \alpha(1- f_{i^*} + f_{j^*}) + f_{i^*} \quad (k \neq i^*).
\end{align*}
Therefore 
$\max_i \cbr{{\Lbf_{(:,i)}}^\intercal \pvec^* + f_i}
= \alpha(1- f_{i^*} + f_{j^*}) + f_{i^*}
$,
which matches $\text{AL}^{\text{abstain}}+f_y$.
\end{enumerate}
\end{proof}

From the theorem above, we derive a non-probabilistic prediction scheme based on the maximizer of the predictor's probability as follows.

\begin{corollary}
For $0 \leq \alpha \leq \frac12$, a non-probabilistic prediction of the adversarial prediction method for the classification with abstention task can be computed as:
\begin{align}
    y^* = \begin{cases} 
      i^* & f_{i^*} - f_{j^*} \geq \frac12 \\
      \text{abstain} & \text{otherwise}
   \end{cases}
\end{align}
where $i^*$ and $j^*$ are the indices of the largest and the second largest potentials respectively.
\label{col:abstain}
\end{corollary}

The runtime complexity of this prediction scheme is $\Ocal(k)$ since the algorithm needs to scan all $k$ potentials and maintain the two largest potentials. This is much faster than solving the minimax game in Eq. \eqref{eq:predict-vector}, which costs $\Ocal(k^{3.5})$.

\section{Theoretical Guarantees}

In this section, we study the theoretical properties of the adversarial prediction framework. We first analyze the property of our method in ideal learning settings, followed by the generalization guarantee of our method.

\subsection{Fisher Consistency}

The behavior of a prediction method in ideal learning settings---i.e., trained on the true evaluation distribution and given an arbitrarily rich feature representation, or, equivalently, considering the space of all measurable functions---provides a useful theoretical validation.  
Fisher consistency requires that the prediction model yields the Bayes optimal decision boundary in this setting \citep{tewari2007consistency,liu2007fisher,ramaswamy2012classification,pedregosa2017consistency}.
Suppose the potential scoring function $f(\xvec, y)$ is optimized over the space of all measurable functions. 
Given the true distribution $P(\Xbf,Y)$, a surrogate loss function $\delta$ % : \mathcal{X} \times \mathcal{Y} \rightarrow \mathbb{R}$
is said to be Fisher consistent with respect to the loss $\ell $ % : \mathcal{X} \times \mathcal{Y} \rightarrow \mathbb{R}$ 
if the minimizer ${f}^*$ of the surrogate loss reaches the Bayes optimal risk, i.e.:
\begin{align}
& f^* \in \argmin_{f} \mathbb{E}_{Y|{\bf x}\sim P} \left[ \delta_{f}({\bf x}, Y)\right] \;\;
\Rightarrow \;\;
\mathbb{E}_{Y|{\bf x}\sim P} \left[ \ell_{{f}^*}({\bf x}, Y)\right] = \min_{f} \mathbb{E}_{Y|{\bf x}\sim P} \left[ \ell_{f}({\bf x}, Y)\right]. \label{eq:consistency-general}
\end{align}
Here $\delta_f(\xvec, y)$ stands for the surrogate loss function value if the true label is $y$ and we make a prediction on $\xvec$ using the potential function $f(\xvec, y)$.
The loss $\ell_f$ has a similar meaning.

\subsubsection{Fisher Consistency for Potential-Based Prediction}

We consider Fisher consistency for standard multiclass classification where the prediction is done by taking the argmax of the potentials, i.e., $\argmax_y f(\xvec, y)$. This usually applies to the setting where the predictor and ground truth class labels are chosen from the same set of labels, i.e., $y^* \in \Ycal$, and $y \in \Ycal \triangleq [k]$. 
Given that prediction is based on the $\argmax$ of the potentials, the right-hand side of Eq. \eqref{eq:consistency-general} is equivalent to:
\begin{align}
\mathbb{E}_{Y|{\bf x}\sim P} \left[ \ell(\argmax_{y'} {f}^*({\bf x}, y'), Y)\right] &= \min_{f} \mathbb{E}_{Y|{\bf x}\sim P} \left[ \ell(\argmax_{y'} {f}({\bf x}, y'), Y)\right].
\end{align}
Since $f$ is optimized over all measurable functions, the condition in Eq. \eqref{eq:consistency-general} can be further simplified as
\begin{align}
& f^* \in \argmin_{f} \mathbb{E}_{Y|{\bf x}\sim P} \left[ \delta_{f}({\bf x}, Y)\right]\\
\Rightarrow \ &
\argmax_{y'} {f}^*({\bf x}, y') \subseteq \argmin_{y'} \mathbb{E}_{Y|{\bf x}\sim P} \left[ \ell(y', Y)\right], \qquad \forall \xvec \in \Xcal.
\end{align}

Using the potential scoring function notation $f(\xvec, y)$, the adversarial surrogate loss in Eq. \eqref{eq:al-hat-check} can be equivalently written as:
\begin{align}
    AL_f({\bf x}, y) = 
    \max_{\check{P}(\check{Y}|{\bf x})}
     \min_{\hat{P}(\hat{Y}|{\bf x})}
    \mathbb{E}_{\hat{Y}|{\bf x}\sim\hat{P};
    \check{Y}|{\bf x}\sim{\check{P}}}\left[ \text{loss}(\hat{Y},\check{Y})
    + f({\bf x},\check{Y}) - f({\bf x},{y}) 
    \right].
\end{align}
Then, the Fisher consistency condition for the adversarial surrogate loss AL$_{f}$ becomes: 
\begin{align}
\label{eq:def_consistency}
& f^* \in \Fcal^* \triangleq \argmin_{f} \mathbb{E}_{Y|\xvec \sim P}
\sbr{\text{AL}_{f} (\xvec, Y)}
\\ \Rightarrow \ 
&\argmax_y f^*(\xvec,y) \subseteq \Ycal^\diamond \triangleq \argmin_{y'} \mathbb{E}_{Y | \xvec \sim P} [\text{loss}(y', Y)]. \notag \end{align}
%Note that in Eq. \eqref{eq:def_consistency} we allow $f$ to be optimized over the set of all measurable functions on the input space $(\xvec, y)$.
In the sequel, we will show that the condition in Eq. \eqref{eq:def_consistency} holds for our adversarial surrogate AL for any loss metrics satisfying a natural requirement that the correct prediction must suffer a loss that is strictly less than incorrect predictions. 
We start in Theorem \ref{thm:singlemin} by establishing Fisher consistency when the Bayes optimal label under the true distribution is unique (i.e., $\Ycal^\diamond$ is a singleton), and then proceed to more general cases in Theorem \ref{thm:multimin}. 

\vspace{1.2mm}
\begin{theorem}
\label{thm:singlemin}
% simplified version
	In the standard multiclass classification setting, suppose we have a loss metric that satisfies the natural requirement: $\text{loss}(y, y) < \text{loss}(y, y')$ for all $y' \neq y$.
	Then the adversarial surrogate loss AL$_{f}$ is Fisher consistent if $f$ is optimized over all measurable functions
	and $\Ycal^\diamond$ is a singleton.
\end{theorem}

\begin{proof}
	Let $\pvec$ be the probability mass given by the predictor player $\Phat(\hat{Y}|\xvec)$,
	$\qvec$ be the probability mass given by the adversary player $\Pcheck(\check{Y}|\xvec)$, and $\dvec$ be the probability mass of the true distribution $P(Y|\xvec)$.
	So, all $\pvec$, $\qvec$, and $\dvec$ lie in the $k$ dimensional probability simplex $\Delta$, where $k$ is the number of classes.
	Let $\Lbf$ be a $k$-by-$k$ loss matrix whose $(y, y')$-th entry is $\text{loss}(y, y')$.
	Let $\fvec \in \RR^{k}$ be the vector encoding of the value of $f$ at all classes.
	The definition of $f^*$ in Eq. \eqref{eq:def_consistency}   
	now becomes:
	\begin{align}
		\fvec^* &\in \argmin_\fvec \max_{\qvec \in \Delta} \min_{\pvec \in \Delta} \cbr{\fvec^\intercal \qvec + \pvec^\intercal \Lbf \qvec - \dvec^\intercal \fvec}
		\label{eq:proof_consistency_f}
		=\argmin_\fvec \max_{\qvec \in \Delta} \cbr{\fvec^\intercal \qvec + \min_{y} (\Lbf \qvec)_{y} - \dvec^\intercal \fvec}.\ \ \ 
	\end{align}
	%So the definition of $\pi^*$ is equivalent to  
	%$\pi^* = \argmax_\pi f^*_\pi$. % and $\pi^\diamond = \argmin_\pi (C \dvec)_\pi$.
	
% 	Let us assume that $\Pi^\diamond = \argmin_\pi \mathbb{E}_{\pibar | x \sim P} [\text{loss}(\pi, \pibar)]$ (or equivalently $ \argmin_\pi (C \dvec)_\pi$) has a unique solution which we denote as $\pi^\diamond$ (we will drop this assumption later). 
	Since $\Ycal^\diamond \triangleq \argmin_{y} \mathbb{E}_{Y | \xvec \sim P} [\text{loss}(y, Y)]$ (or equivalently $ \argmin_{y} (\Lbf \dvec)_{y}$) contains only a singleton,
	we denote it as $y^\diamond$. 
	We are to show that $\argmax_y f^*(\xvec,y)$ is a singleton,
	and its only element is exactly $y^\diamond$.
	Since $\fvec^*$ is an optimal solution,
	the objective function must have a zero subgradient at $\fvec^*$.
	That means $\zero = \qvec^* - \dvec$,
	where $\qvec^*$ is an optimal solution in Eq. \eqref{eq:proof_consistency_f} under $\fvec^*$.
	As a result:
	\begin{align}
	    \label{eq:d_in_argmax}
		\dvec \in \argmax_{\qvec \in \Delta} \cbr{\qvec^\intercal \fvec^* + \min_{y} (\Lbf\qvec)_{y}}.
	\end{align}
	
	By the first order optimality condition of constrained convex optimization 
	(see Eq. (4.21) of \citet{boyd2004convex}),
	this means:
	\begin{align}
	\label{eq:1st_optimality}
	  \left(\fvec^* + {\Lbf_{(y^\diamond,:)}}^\intercal\right)^\intercal (\uvec - \dvec) \le 0 \quad \forall \uvec \in \Delta,
	\end{align}
	where $\Lbf_{(y^\diamond,:)}$ is the $y^\diamond$-th row of $\Lbf$, $\fvec^* + {\Lbf_{(y^\diamond,:)}}^\intercal$ is the gradient of the objective in Eq. \eqref{eq:d_in_argmax} with respect to $\qvec$ evaluated at $\qvec = \dvec$.
	Here we used the definition of $y^\diamond$.
	However, this inequality can hold for some $\dvec \in \Delta_k \cap \RR_{++}^{k}$ 
	only if 
	$\fvec^* + {\Lbf_{(y^\diamond,:)}}^\intercal$ is a uniform vector,
	\ie, $f^*_y + \text{loss}(y^\diamond, y)$ is constant in $y$.
	To see this, let us assume the contrary that $\vvec \triangleq \fvec^* + {\Lbf_{(y^\diamond,:)}}^\intercal$ is not a uniform vector, and let $i$ be the index of its maximum element. Setting $\uvec = \evec_i$, it is clear that for any $\dvec \in \Delta_k \cap \RR_{++}^{k}$, $\vvec^\intercal \uvec > \vvec^\intercal \dvec$ and hence $\left(\fvec^* + {\Lbf_{(y^\diamond,:)}}^\intercal\right)^\intercal (\uvec - \dvec) > 0$, which violates the optimality condition.
	
	Finally, using the assumption that $\text{loss}(y, y) < \text{loss}(y, y')$ for all $y' \neq y$,
	it follows that $\argmax_y f^*(\xvec,y) = \argmin_y \Lbf_{(y^\diamond,y)} = \{y^\diamond\}$.
% 	$\argmax_y f^*(\xvec,y) = \argmin_y (\Lbf_{(y^\diamond,:)})_y$.
\end{proof}

The assumption of loss function in the above theorem is quite mild, requiring only that the incorrect predictions suffer higher loss than the correct one. 
We do not even require symmetry in its two arguments.
The key to the proofs is the observation that for the optimal potential function $f^*$,
$f^*(\xvec,y) + \text{loss}(y^\diamond, y)$ is invariant to $y$ when $\Ycal^\diamond = \{ y^\diamond \}$.
We refer to this as the \emph{loss reflective} property of the minimizer. In the next theorem, we generalize Theorem \ref{thm:singlemin} to the case where the Bayes optimal prediction under the true distribution may have ties.

\vspace{1.2mm}
\begin{theorem}
\label{thm:multimin}
	In the standard multiclass classification setting, suppose we have a loss metric that satisfies the natural requirement:  $\text{loss}(y, y) < \text{loss}(y, y')$ for all $y' \neq y$.
	Furthermore, if $f$ is optimized over all measurable functions,
	then:
	\begin{enumerate}[label=(\alph*), itemsep=0.5pt, topsep=0.5pt]
	    \item there exists $f^* \in \Fcal^*$ 
	    such that 
	    $\argmax_y f^*(\xvec,y) \subseteq \Ycal^\diamond$ (i.e., satisfies the Fisher consistency requirement).
	    In fact, all elements in $\Ycal^\diamond$ can be recovered by some $f^* \in \Fcal^*$.
	    %
	   % \item if 
	   % $\argmin_y \sum_{y' \in \Ycal^\diamond} \alpha_{y'} \text{loss}(y', y) \subseteq \Ycal^\diamond$ for all $\alpha_{(\cdot)} \ge 0$; $\sum_{y' \in \Ycal^\diamond} \alpha_{y'} = 1$,
	   % then  $\argmax_y f^*(\xvec,y) \subseteq \Ycal^\diamond$ for \textbf{all} $f^* \in \Fcal^*$. In this case, all $f^* \in \Fcal^*$ satisfy the Fisher consistency requirement.
	    \item if the loss satisfies
	    $\argmin_{y'} \sum_{y \in \Ycal^\diamond} \alpha_{y} \text{loss}(y, y') \subseteq \Ycal^\diamond$ for all $\alpha_{(\cdot)} \ge 0$ and $\sum_{y \in \Ycal^\diamond} \alpha_{y} = 1$,
	    then  $\argmax_y f^*(\xvec,y) \subseteq \Ycal^\diamond$ for \textbf{all} $f^* \in \Fcal^*$. 
	    In this case, all $f^* \in \Fcal^*$ satisfies the Fisher consistency requirement.
	\end{enumerate}
\end{theorem}

\begin{proof}
	Let $\pvec$, $\qvec$, and $\dvec$ have the same meaning as in the proof of Theorem \ref{thm:singlemin}. 
	Let $\Ycal^\diamond \triangleq \argmin_{y} (\Lbf \dvec)_{y}$
	which is not necessarily a singleton.
	%contain all of the solution, i.e., $\Ycal^\diamond = \{ y^\diamond \mid (\Lbf \dvec)_{y^\diamond} = \min_{y'} (\Lbf \dvec)_{y'} \}$. 
	The analysis in the proof of Theorem \ref{thm:singlemin} carries over to this case, except for Eq. \eqref{eq:1st_optimality}. Denote $h(\qvec) \triangleq \qvec^\intercal \fvec^* + \min_{y} (\Lbf \qvec)_{y}$. The subdifferential of $-h(\qvec)$ evaluated at $\qvec = \dvec$ is the set:
	\begin{align}
	  \label{eq:subdiff}
	  \partial (-h)(\dvec) = \{-\fvec^* - \vvec \mid \vvec \in \text{\bf conv}\{{\Lbf_{(y^\diamond,:)}}^\intercal \mid y^\diamond \in \Ycal^\diamond \} \},  
	\end{align}
	where $\text{\bf conv}$ denotes the convex hull.
	By extending the first order optimality condition to the subgradient case, this means that there is a subgradient $\gvec \in \partial (-h)(\dvec) $ such that:
	\begin{align}
	\label{eq:1st_optimality_subgrad}
	  \gvec^\intercal (\uvec - \dvec) \ge 0 \quad \forall \uvec \in \Delta.
	\end{align}
	
	Similar to the singleton $\Ycal^\diamond$ case, this inequality can hold for some $\dvec \in \Delta \cap \RR_{++}^{k}$ 
	only if 
	$\gvec$ is a uniform vector.
	Based on Eq. \eqref{eq:subdiff}, $-\gvec - \fvec^*$ can be written as a convex combination of $\{{\Lbf_{(y^\diamond,:)}}^\intercal \mid y^\diamond \in \Ycal^\diamond \}$, 
	and the ``if and only if'' relationship in the above derivation leads to a full characterization of the optimal potential function set $\Fcal^*_\xvec$ for a given $\xvec$ (c.f. Eq. \eqref{eq:def_consistency}):
	\begin{align}
	\label{eq:f_star_multi}
	 \Fcal^*_\xvec = \cbr{\fvec^* = c \one - \sum_{y \in \Ycal^\diamond} \alpha_{y} {\Lbf_{(y,:)}}^\intercal \ \middle\vert \ 
	 \alpha_{(\cdot)} \ge 0, \ \sum_{y \in \Ycal^\diamond} \alpha_{y} = 1,\ c \in \RR}.
	\end{align}
	This means that multiple solutions of $\fvec^*$ are possible. 
	%Let us denote the set of containing all solutions as $\Fcal^*$.
	For each element $y^\diamond$ in $\Ycal^\diamond$, we can recover a $f^*_{y^\diamond}$ in which the $\argmax_y f^*_{y^\diamond}(\xvec,y)$ contains a singleton element $y^\diamond$ by using Eq. \eqref{eq:f_star_multi} with  $\alpha_{y^\diamond} = 1$ and $\alpha_{y\in \{\Ycal^\diamond\setminus y^\diamond\}} = 0$. 
	This is implied by our loss assumption that $\text{loss}(y, y) < \text{loss}(y, y')$ for all $y' \neq y$, and hence $\argmax_y f^*_{y^\diamond}(\xvec,y) = \argmin_y \Lbf_{(y^\diamond,y)}$. 
	So (a) is proved.
	
	We next prove (b).
	If we assume $\argmin_{y'} \sum_{y \in \Ycal^\diamond} \alpha_{y} \text{loss}(y, y') \subseteq \Ycal^\diamond$ for all $\alpha_{(\cdot)} \ge 0$ and $\sum_{y \in \Ycal^\diamond} \alpha_{y} = 1$, 
	then it follows trivially that $\argmax_y f^*(\xvec,y) \subseteq \Ycal^\diamond$ for all $f^* \in \Fcal^*_\xvec$. 
	%since for any loss function that satisfy the assumption, $\argmin_y \sum_{y' \in \Ycal^\diamond} a_{y'} {\Lbf_{(y',y)}}  \subseteq \Ycal^\diamond$ for all  $\alpha_{(\cdot)} \ge 0$, $\sum_{y' \in \Ycal^\diamond} \alpha_{y'} = 1$. 
\end{proof}

Note that in theorem above, when the Bayes optimal label is not unique, for general loss metric (point (a)), not all $f^* \in \Fcal^*$ satisfy the requirement of $\argmax_y f^*(\xvec,y) \subseteq \Ycal^\diamond$. However, all elements in $\Ycal^\diamond$ can be recovered by some $f^* \in \Fcal^*$. This is expected  in the case of non-unique optimal labels, since the prediction rule under the potential based prediction, $\argmax_y f(\xvec,y)$  is not well defined in the case where $f(\xvec,y)$ contains ties.

\subsubsection{Consistency for Prediction Based on the Predictor Player's Probability}

For a prediction task where the set of options a predictor can choose is different from the set of ground truth labels (e.g., the classification task with abstention task in Section \ref{sec:pred_abstain}),  the analysis in the previous subsection cannot be applied. In this subsection we will establish consistency properties of the adversarial prediction framework for a general loss matrix where the prediction is based on the predictor player's optimal probability. 

\begin{theorem}
The predictor's optimal probability in the adversarial prediction
framework given the true distribution $P(Y|\xvec)$ and a loss matrix $\Lbf$ reaches the Bayes optimal risk, assuming that $f$ is allowed to be optimized over all measurable function.

\end{theorem}

\begin{proof}
	Since the predictor can choose from $l$ options which could be different than the $k$ number of classes in the ground truth,
	$\dvec$ and $\qvec$ lie in the $k$ dimensional probability simplex $\Delta^k$, while the predictor's probability mass $\pvec$ lies in the $l$ dimensional probability simplex $\Delta^l$.
	Let $\fvec \in \RR^{k}$ the vector encoding of the value of $f$ at all classes.
	The potential function minimizer $f^*$ can now be written as:
	\begin{align}
		\fvec^* &\in \argmin_\fvec \max_{\qvec \in \Delta^k} \min_{\pvec \in \Delta^l} \cbr{\fvec^\intercal \qvec + \pvec^\intercal \Lbf \qvec - \dvec^\intercal \fvec}. \label{eq:proof_consistency_f2}
	\end{align}
	
	As noted in our previous analysis, since $\fvec^*$ is an optimal solution,
	the objective function must have a zero subgradient at $\fvec^*$.
	That means $\zero = \qvec^* - \dvec$,
	where $\qvec^*$ is an optimal solution in Eq. \eqref{eq:proof_consistency_f2} under $\fvec^*$. 
	
	Here we use the probabilistic prediction scheme as mentioned in Eq. \eqref{eq:predict-vector}. The consistency condition in Eq. \eqref{eq:consistency-general} requires  that the loss of this prediction scheme under the optimal potential $\fvec^*$ and the true probability $\dvec$ reaches the Bayes optimal risk, i.e.,
	\begin{align}
    &{\pvec^\diamond}^\intercal \Lbf \dvec = \min_{y} (\Lbf \dvec)_{y},
    \quad \text{where } \quad \pvec^\diamond = \argmin_{\pvec \in \Delta^l}
    \max_{\qvec \in \Delta^k}
    \pvec^\intercal \Lbf \qvec 
    + {\fvec^*}^\intercal \qvec.
    \end{align}
    
    Since the maximization over $\qvec$ in Eq. \eqref{eq:proof_consistency_f2} does not depend on $\dvec^\intercal \fvec$, 
    we know that $\dvec$ is also an optimal solution of $
    \argmax_{\qvec \in \Delta^k}
    \min_{\pvec \in \Delta^l} \pvec^\intercal \Lbf \qvec + {\fvec^*}^\intercal \qvec
    $.
    Then, based on the minimax duality theorem \citep{von1945theory}, we know that:
    \begin{align}
    {\pvec^\diamond}^\intercal \Lbf \dvec
    + {\fvec^*}^\intercal \dvec
    = 
    \min_{\pvec \in \Delta^l}
    \pvec^\intercal \Lbf \dvec
    + {\fvec^*}^\intercal \dvec.
    \end{align}
    This implies that:
    %\begin{align}
${\pvec^\diamond}^\intercal \Lbf \dvec
    = 
    \min_{\pvec \in \Delta^l}
    \pvec^\intercal \Lbf \dvec = \min_{y} (\Lbf \dvec)_{y},
$
%\end{align}
    which concludes our proof.
\end{proof}

\subsection{Generalization Guarantee}

We now analyze the generalization guarantee of the adversarial prediction framework. Let us write the constraint set of the adversarial prediction formulation in Eq. \eqref{eq:def} as the set $\Xi$,
\begin{align}
    \Xi \triangleq \left\{ \check{P}(\check{Y}|\Xbf) \mid \mathbb{E}_{{\bf X} \sim 
    \tilde{P};
    	\check{Y}|{\bf X}\sim \check{P}}[\phi({\bf X},\check{Y})]
    = \mathbb{E}_{{\bf X},{Y} \sim \tilde{P}}\left[\phi({\bf X},{Y}) \right] \right\}.
\end{align}
Since the predictor player in our formulation robustly minimizes the loss metric against an adversary player, the adversarial loss upper-bounds the generalization loss, so long as the evaluation distribution is similar to training data properties, as described in Theorem \ref{thm:generalization}.

\begin{theorem}
\label{thm:generalization}
Let $\Phat^*(\Yhat|\Xbf)$ and $\check{P}^*(\check{Y}|\Xbf)$ be the predictor player's solution and adversarial player's solutions of the adversarial prediction formulation (Eq. \eqref{eq:def}) respectively.
Given the underlying true distribution $P(Y|\Xbf)$ resides in the uncertainty set $\Xi$, the generalization loss is upper-bounded by the adversarial loss:
\begin{align}
    \mathbb{E}_{{\bf X} \sim \tilde{P}; \hat{Y}|{\bf X}\sim\hat{P}^*;
    	Y|{\bf X}\sim P} \left[\text{loss}(\hat{Y}, Y) \right] \leq
    \mathbb{E}_{{\bf X} \sim \tilde{P}; \hat{Y}|{\bf X}\sim\hat{P}^*;
    	\check{Y}|{\bf X}\sim \check{P}^*} \left[\text{loss}(\hat{Y}, \check{Y}) \right].
\end{align}
\end{theorem}

\begin{proof}
By definition, the solution of the adversarial conditional label distribution, $\check{P}^*(\check{Y}|\Xbf)$, is a Nash equilibrium and it provides
the worst possible loss for the estimator of all conditional
label distributions from set $\Xi$. So long as the underlying true label
distribution used for evaluation, $P(Y|\Xbf)$, is similar to training data properties (i.e, a member of $\Xi$), 
then
the generalization loss resulted from evaluation distribution  $P(Y|\Xbf)$ must be 
no worse than $\check{P}^*(\check{Y}|\Xbf)$. Otherwise,  $P(Y|\Xbf)$ is  a
better choice from $\Xi$ than $\check{P}^*(\check{Y}|\Xbf)$ for maximizing the predictor’s loss, a contradiction.
\end{proof}
\section{Optimization}

The goal of a learning algorithm in the adversarial prediction framework is to obtain the optimal Lagrange dual variable $\theta$ that enforces the adversary's probability distribution to reside within the moment matching constraints in Eq \eqref{eq:def}. In the risk minimization perspective (Eq. \eqref{eq:dual}), it is equivalent to finding the parameter $\theta$ that minimizes the adversarial surrogate loss (AL) in Eq. \eqref{eq:al-hat-check}. To find the optimal $\theta$, we employ (sub)-gradient methods to optimize our convex objective.

\subsection{Subgradient-Based Convex Optimization}

The risk minimization perspective of adversarial prediction framework  (Eq. \eqref{eq:dual}) can be written as:
\begin{align}
\min_{\theta} & \; \mathbb{E}_{{\bf X},{Y}\sim \tilde{P}} \; \left[ AL({\bf X}, Y, \theta) \right]\\
\text{where:  } & AL({\bf x}, y, \theta) = \max_{\qvec \in \Delta}
    \min_{\pvec \in \Delta}
    \pvec^\intercal \Lbf \qvec 
    + \theta^\intercal \left[ \textstyle\sum_{j} q_{j} \phi(\xvec, j)
    - \phi(\xvec, y) \right].
\end{align}
The subdifferential of the expected adversarial loss in the objective above is equal to the expected subdifferential of the loss for each sample \cite[Corollary 23.8,][]{Rockafellar70}:
\begin{align}
\partial_{\theta} \; \mathbb{E}_{{\bf X},{Y}\sim \tilde{P}} \; \left[ AL({\bf X}, Y, \theta) \right] 
     = \mathbb{E}_{{\bf X},{Y}\sim \tilde{P}} \; \left[ \partial_{\theta} \; AL({\bf X}, Y, \theta) \right].
\end{align}
Theorem \ref{thm:subgradient-lp} describes the subgradient of the adversarial surrogate loss with respect to $\theta$.

\begin{theorem}
\label{thm:subgradient-lp}
Given $\theta$, suppose the set of optimal $\qvec$ for the maximin inside the AL is $Q^*$:
\[
Q^*  = \argmax_{\qvec \in \Delta}
    \min_{\pvec \in \Delta}
    \cbr{\pvec^\intercal \Lbf \qvec 
    + \theta^\intercal \left[ \textstyle\sum_{j} q_{j} \phi(\xvec, j)
    - \phi(\xvec, y) \right]}.
\]
Then the subdifferential of the adversarial loss AL($\xvec, y, \theta$) with respect to the parameter $\theta$ can be fully characterized by
\[
\partial_{\theta} \; AL({\bf x}, y, \theta) =
 \text{\bf conv} \cbr{\textstyle\sum_{j} q^*_{j} \phi(\xvec, j)
    - \phi(\xvec, y) \ \middle\vert \  \qvec \in Q^*}.
\]
\end{theorem}

\begin{proof}
Denote $\varphi(\theta, \qvec) \triangleq \min_{\pvec \in \Delta} \cbr{\pvec^\intercal \Lbf \qvec 
    + \theta^\intercal \left[ \textstyle\sum_{j} q_{j} \phi(\xvec, j)
    - \phi(\xvec, y)\right]} $.
Then for any fixed $\qvec$,
$\varphi(\theta,\qvec)$
is a closed proper convex function in $\theta$.
Denote $g(\theta) \triangleq \max_{\qvec \in \Delta} \varphi(\theta, \qvec)$.
Then the interior of its domain $\text{int}(\dom g)$ is the entire Euclidean space of $\theta$,
and $\varphi$ is continuous on $\text{int}(\dom g) \times \Delta$.
Using the obvious fact that $\partial_\theta \varphi(\theta,\qvec) = \cbr{\sum_{j} q_{j} \phi(\xvec, j) - \phi(\xvec, y)}$,
the desired conclusion follows directly from Proposition A.22 of \cite{Bertsekas71}.
\end{proof}

The runtime complexity to calculate the subgradient of AL for one example above is $\Ocal(k^{3.5})$ due to the need to solve the inner minimiax using linear program (Karmarkar's algorithm). 
For the loss metrics that we have studied in Section 3 we construct faster ways to compute the subgradient as follows.

\begin{corollary}
The subdifferential of AL$^\text{0-1}$($\xvec, y, \theta$) with respect to $\theta$ includes:
\[
 \partial_{\theta} \; AL^\text{0-1}({\bf x}, y, \theta) \ni \tfrac{1}{|S^*|} \textstyle\sum_{j \in S^*} \phi(\xvec, j)
    - \phi(\xvec, y),
\]
where $S^*$ is an optimal solution set of the maximization inside the AL$^\text{0-1}$, i.e.:
\[
S^* \in 
\argmax_{
    {S \subseteq
			[k], %\\ 
            \; S \neq \emptyset } }
	\frac{\sum_{j \in S}  \theta^\intercal \phi(\xvec, j)
		+|S|-1}{|S|}.
\]
\end{corollary}

\begin{corollary}
The subdifferential of AL$^\text{ord}$($\xvec, y, \theta$) with respect to $\theta$ includes:
\[
 \partial_{\theta} \; AL^\text{ord}({\bf x}, y, \theta) \ni \tfrac12 \left( \phi(\xvec, i^*) + \phi(\xvec, j^*) \right)
    - \phi(\xvec, y),
\]
where $i^*, j^*$ is the solution of:
\[
(i^*, j^*) \in \argmax_{i,j \in [k] }  \frac{\theta^\intercal \phi(\xvec, i) + \theta^\intercal \phi(\xvec, j) + j - i}{2} .
\]
\end{corollary}

\begin{corollary}
The subdifferential of AL$^\text{abstain}$($\xvec, y, \theta, \alpha$) where $0 \le \alpha \le \frac12$ with respect to $\theta$ includes:
\[
 \partial_{\theta} \; AL^\text{abstain}({\bf x}, y, \theta, \alpha) \ni 
 \begin{cases} 
      (1-\alpha) \phi(\xvec, i^*) + \alpha \phi(\xvec, j^*) - \phi(\xvec, y) & 
      g({\bf x}, y, \theta, \alpha) > h({\bf x}, y, \theta, \alpha) \\
      \phi(\xvec, l^*) - \phi(\xvec, y) & \text{otherwise},
   \end{cases}
\]
where:
\begin{align}
g({\bf x}, y, \theta, \alpha) &= \max_{i,j \in [k], i \neq j }  \left(1 - \alpha \right) f_i + \alpha f_j + \alpha, \quad
h({\bf x}, y, \theta, \alpha) = \max_l f_l,
\\
(i^*, j^*) &\in \argmax_{i,j \in [k], i \neq j }  \left(1 - \alpha \right) f_i + \alpha f_j + \alpha, \quad
l^* = \argmax_l f_l,
\end{align}
and the potential $f_i$ is defined as $f_i = \theta^\intercal \phi({\xvec, i})$.
\end{corollary}

The runtime of the subgradient computation algorithms above are the same as the runtime of computing the adversarial surrogate losses, i.e., $\Ocal(k \log k)$ for AL$^\text{0-1}$, $\Ocal(k)$ for AL$^\text{ord}$, and $\Ocal(k)$ for AL$^\text{abstain}$. This is a significant speed-up compared to the technique that uses a linear program solver.

Since we already have algorithms for computing the subgradient of AL, any subgradient based optimization techniques can be used to optimize $\theta$ including some stochastic (sub)-gradient techniques like SGD, AdaGrad, and ADAM or batch (sub)-gradient techniques like L-BFGS. Some regularization techniques such as L1 and L2 regularizations, can also be added to the objective function.
The optimization is guaranteed to converge to the global optimum as the objective is convex.

\subsection{Incorporating Rich Feature Spaces via the Kernel Trick}

Considering large feature spaces is important for developing an expressive
classifier that can learn from large amounts of training data.
Indeed, Fisher consistency requires such feature spaces for its guarantees to
be meaningful.
However, na\"ively projecting from the original feature space, $\phi(\xvec, y)$,
to a richer (or possibly infinite) feature space $\omega(\phi(\xvec, y))$,
can be computationally burdensome.
Kernel 
methods enable this feature expansion by allowing the dot products of certain feature functions to be computed implicitly, i.e., $K(\phi(\xvec_i, y_i) , \phi(\xvec_j, y_j)) = \omega(\phi(\xvec_i, y_i)) \cdot \omega(\phi(\xvec_j, y_j)) $. 

To formulate a learning algorithm for adversarial surrogate losses that can incorporate richer feature spaces via kernel trick, we apply the PEGASOS algorithm \citep{shalev2011pegasos} to our losses. Instead of optimizing the problem in the dual formulation as in many kernel trick algorithms, PEGASOS allows us to incorporate the kernel trick into its primal stochastic subgradient optimization technique. The algorithm works on L2 penalized risk minimization, 
\begin{align}
\min_{\theta} & \; \mathbb{E}_{{\bf X},{Y}\sim \tilde{P}} \; \frac{\lambda}{2} \|\theta \|^2 + AL({\bf X}, Y, \theta) ,
\end{align}
where $\lambda$ is the regularization penalty parameter.
Since we want to perform stochastic optimization, we replace the objective above with an approximation based on a single training example:
\begin{align}
\frac{\lambda}{2} \|\theta \|^2 + AL({\bf x}_{i_t}, y_{i_t}, \theta) ,
\end{align}
where $i_t$ indicates the index of the example randomly selected at iteration $t$. Therefore at iteration $t$, the subgradient of our objective function with respect to the parameter $\theta$ is:
\begin{align}
\partial_\theta^{(t)} =& \lambda \theta^{(t)} + \textstyle\sum_{j} {q^*_{j}}^{(t)} \phi(\xvec_{i_t}, j)
    - \phi(\xvec_{i_t}, y_{i_t}) , \label{eq:grad-pegasos} \\
    & \text{where: }  {\qvec^*}^{(t)}  = \argmax_{\qvec \in \Delta}
    \min_{\pvec \in \Delta}
    \pvec^\intercal \Lbf \qvec 
    + {\fvec^{(t)}}^\intercal \qvec 
    - f_{y_{i_t}}^{(t)}, \\
    & \qquad \quad f_{j}^{(t)} = {\theta^{(t)}}^\intercal \phi(\xvec_{i_t},j).
\end{align}

The algorithm starts with zero initialization, i.e., $\theta^{(1)} = {\bf 0}$ and uses a pre-determined learning rate scheme $\eta^{(t)} = \frac{1}{\lambda t}$ to take optimization steps, 
\begin{align}
\theta^{(t+1)} = \theta^{(t)} - \eta^{(t)} \partial_\theta^{(t)} = \theta^{(t)} - \tfrac{1}{\lambda t} \partial_\theta^{(t)}.
\end{align}
Let us denote $\gvec^{(t)} = \sum_{j} {q^*_{j}}^{(t)} \phi(\xvec_{i_t}, j)
    - \phi(\xvec_{i_t}, y_{i_t})$ from Eq. \eqref{eq:grad-pegasos}, then the update steps can be written as:
\begin{align}
\theta^{(t+1)} = (1 - \tfrac{1}{t}) \theta^{(t)} - \tfrac{1}{\lambda t} \gvec^{(t)}.
\end{align}
By accumulating the weighted contribution of $\gvec$ for each step, the value of $\theta$ at iteration $t+1$ is:
\begin{align}
\theta^{(t+1)} = -\frac{1}{\lambda t} \sum_{l=1}^{t} \gvec^{(l)},
\end{align}
which can be expanded to the original formulation of our subgradient:
\begin{align}
\theta^{(t+1)} &= -\frac{1}{\lambda t} \sum_{l=1}^{t} \sum_{j=1}^k {q^*_{j}}^{(l)} \phi(\xvec_{i_l}, j)
    - \phi(\xvec_{i_l}, y_{i_l}) ,  \label{eq:theta-pegasos}  \\
    \where  {\qvec^*}^{(l)} &= \argmax_{\qvec \in \Delta}
    \min_{\pvec \in \Delta}
    \pvec^\intercal \Lbf \qvec 
    + {\fvec^{(l)}}^\intercal \qvec 
    - f^{(l)}_{y_{i_l}}, \\
    f_{j}^{(l)} &= {\theta^{(l)}}^\intercal \phi(\xvec_{i_l},j).
\end{align}

Let $\zvec$ be the one-hot vector representation of the ground truth label $y$ where its elements are $z_y = 1$, and $z_j = 0$ for all $j \neq y$. From the definition of $\gvec^{(t)}$, let us denote $\rvec^{(t)} = {\qvec^*}^{(t)} - \zvec_{i_t}$, then $\gvec^{(t)}$ can be equivalently written as $\gvec^{(t)} = \sum_{j} {r_{j}}^{(t)} \phi(\xvec_{i_t}, j)$. 
We denote $\alphavec_{i}^{(t+1)}$ as a vector that accumulates the value of $\rvec$ for the $i$-th example each time it is selected until iteration $t$. Then, the value of $\theta^{(t+1)}$ in Eq. \eqref{eq:theta-pegasos} can be equivalently written as:
\begin{align}
\theta^{(t+1)} =& -\frac{1}{\lambda t} \sum_{i=1}^{n} \sum_{j=1}^k \alpha_{(i,j)}^{(t+1)} \phi(\xvec_{i}, j),    
\end{align}
where $\alpha_{(i,j)}^{(t+1)}$ indicates the $j$-th element of the vector $\alphavec_{i}^{(t+1)}$. Using this notation, the potentials $\fvec^{(t)}$ used to calculate the adversarial loss can be computed as:
\begin{align}
f_j^{(t)} = {\theta^{(t)}}^\intercal \phi(\xvec_{i_t},j) =& -\frac{1}{\lambda t} \sum_{i'}^n \sum_{j'}^k \alpha_{(i',j')}^{(t)} \; \phi(\xvec_{i'},j') \cdot \phi(\xvec_{i_t},j). 
\end{align}
Note that the computation of the potentials above only depends on the dot product between the feature functions weighted by the $\alphavec$ variables.

Since the algorithm only depends on the dot products, to incorporate a richer feature spaces $\omega(\phi(\xvec, y))$, we can directly apply kernel function in the computation of the potentials, 
\begin{align}
f_j^{(t)} = {\theta^{(t)}}^\intercal \omega(\phi(\xvec_{i_t},j)) =& -\frac{1}{\lambda t} \sum_{i'}^n \sum_{j'}^k \alpha_{(i',j')}^{(t)} \omega(\phi(\xvec_{i'},j')) \cdot \omega(\phi(\xvec_{i_t},j)) \\
=& -\frac{1}{\lambda t} \sum_{i'}^n \sum_{j'}^k \alpha_{(i',j')}^{(t)} K(\phi(\xvec_{i'},j'), \phi(\xvec_{i_t},j)).
\end{align}
The detailed algorithm for our adversarial surrogate loss is described in Algorithm \ref{alg:pegasos-adv}.

\begin{algorithm}[ht]
	\caption{PEGASOS algorithm for adversarial surrogate losses with kernel trick}
	\label{alg:pegasos-adv}
	\begin{algorithmic}[1]
		\STATE {\bfseries Input:} Training data $(\mathbf{x}_1, y_1), \ldots (\mathbf{x}_n, y_n)$, $\Lbf, \lambda, T, k$
		\STATE $\alphavec_i^{(1)} \gets {\bf 0}, \forall i \in \{ 1, \hdots, n \} $ 
		\STATE Let $\zvec_{i}$ be the one-hot encoding of $y_i$ for all $i \in \{ 1, \hdots, n \}$
		\FOR{$t \gets 1,2,\ldots, T$}
		\STATE Choose $i_t \in \{ 1, \hdots, n \} $  uniformly at random
		\STATE Compute $\fvec^{(t)}$, where $f_j^{(t)} \gets -\frac{1}{\lambda t} \sum_{i'}^n \sum_{j'}^k \alpha_{(i',j')}^{(t)} K(\phi(\xvec_{i'},j'), \phi(\xvec_{i_t},j))$ 
		\STATE ${\qvec^*}^{(t)} \gets \argmax_{\qvec \in \Delta}
            \min_{\pvec \in \Delta}
            \pvec^\intercal \Lbf \qvec 
            + {\fvec^{(t)}}^\intercal \qvec 
             - f_{y_{i_t}}^{(t)}$ 
		\STATE $\alphavec_{i_t}^{(t+1)} \gets \alphavec_{i_t}^{(t)} + {\qvec^*}^{(t)} - \zvec_{i_t}$ 
		\ENDFOR		
		\RETURN $\alphavec_i^{(t+1)}, \; \forall i \in \{ 1, \hdots, n \}$
	\end{algorithmic}
\end{algorithm}

\section{Experiments}

We conduct experiments on real data to investigate the empirical performance of the adversarial surrogate losses in several prediction tasks.

\subsection{Experiments for Multiclass Zero-One Loss Metric}

We evaluate the performance of the AL$^{\text{0-1}}$ classifier and 
compare it with the three most popular multiclass SVM formulations: WW \citep{weston1999support}, CS \citep{crammer2002algorithmic}, and LLW \citep{lee2004multicategory}. We use 12 datasets from the UCI machine learning repository \citep{lichman2013UCIML} with various sizes and numbers of classes (details in Table \ref{tbl:zo-dataset}). 
For each dataset, we consider the methods using the original feature space
(linear kernel) and a kernelized feature space using the 
Gaussian radial basis function kernel.
We also aim to highlight the sub-optimal performance of the LLW's formulation for multiclass SVM in the case of datasets with low dimensional features \citep{dogan2016unified}, and how our method performs in these datasets. Therefore, most of the datasets we selected have this property.

\begin{table}[ht]
    \small
	\centering
	\begin{tabular}{@{}l@{}rrrrr@{}}
		\toprule
		\multicolumn{1}{c}{\multirow{2}{*}{Dataset}} & \multicolumn{4}{c}{Properties}                                 
\\ \cmidrule(r){2-5}  
		\multicolumn{1}{c}{}                         & \multicolumn{1}{c}{\#class} & \multicolumn{1}{c}{\#train} & \multicolumn{1}{c}{\# test} & \multicolumn{1}{c}{\#feature
}  
\\ \midrule \midrule
		(1) \, iris                                          & 3                              & 105                             & 45                             & 4                               
		                      \\
		(2) \, glass                                         & 6                              & 149                             & 65                             & 9                               
		             \\
		(3) \, redwine                                       & 10                             & 1119                            & 480                            & 11                             
		                     \\
		(4) \, ecoli                                     & 8                             & 235                            & 101                           & 7                              
		                       \\
		(5) \, vehicle                                       & 4                              & 592                             & 254                            & 18                             
		                    \\
		(6) \, segment                                       & 7                              & 1617                            & 693                            & 19                              
		                    \\
		(7) \, sat                                           & 7                              & 4435                            & 2000                           & 36                              
		                    \\
		(8) \, optdigits                                     & 10                             & 3823                            & 1797                           & 64                              
		                          \\
		(9) \, pageblocks                                    & 5                              & 3831                            & 1642                           & 10                              
		              \\
		(10) libras                                        & 15                             & 252                             & 108                            & 90                              
		                     \\
		(11) vertebral                                     & 3                              & 217                             & 93                             & 6                               
		                              \\
		(12) breasttissue                                  & 6                              & 74                              & 32                             & 9                               
		                                                            \\ \bottomrule
	\end{tabular}
	\caption{Properties of the datasets for the zero-one loss metric experiments.}
	\label{tbl:zo-dataset}
\end{table} 

For our experimental methodology, 
we follow the experiment constructions in \citet{dogan2016unified}.
We first perform two-stages parameter selections using five-fold cross validation on the training set of a random split of the dataset to tune each model's parameter $C$ and the kernel parameter $\gamma$ under the kernelized formulation. 
 In the first stage, the values for $C$ are $2^i, i=\{0,3,6,9,12\}$ and the values for $\gamma$ are $2^i, i=\{-12,-9,-6,-3,0\}$. 
We select final values for $C$ from $2^i C_0, i=\{-2,-1,0,1,2\}$ and values for $\gamma$ from $2^i \gamma_0 , i=\{-2,-1,0,1,2\}$ in the second stage, where $C_0$ and $\gamma_0$ 
are the best parameters obtained in the first stage.
%Let $C_0$ and $\gamma_0$ be the best parameters in the first stage. In the second stage, the values for $C$ are $2^i C_0, i=\{-2,-1,0,1,2\}$ and the values for $\gamma$ are $2^i \gamma_0 , i=\{-2,-1,0,1,2\}$. 
We then create 20 independent random splits of each dataset into training and testing sets.
Using the selected parameters, we train each model on the training sets and evaluate the performance on the corresponding testing set. 
We use the Shark machine learning library \citep{shark08} for the implementation of the three multiclass SVM formulations.

\begin{table}[ht]
        \footnotesize
	\centering

	\begin{tabular}{@{}lrrrrrrrr@{}}
		\toprule 
		\multicolumn{1}{c}{\multirow{2}{*}{D}} & \multicolumn{4}{c}{Linear Kernel}                                                                       & \multicolumn{4}{c}{Gaussian Kernel}                                                                 \\ \cmidrule(r){2-5} \cmidrule(l){6-9}
		\multicolumn{1}{c}{}                         & \multicolumn{1}{c}{AL$^{\text{0-1}}$} & \multicolumn{1}{c}{WW} & \multicolumn{1}{c}{CS} & \multicolumn{1}{c}{LLW} & \multicolumn{1}{c}{AL$^{\text{0-1}}$} & \multicolumn{1}{c}{WW} & \multicolumn{1}{c}{CS} & \multicolumn{1}{c}{LLW} \\ \midrule \midrule
		(1) %iris                                          
		& \textbf{96.3} (3.1)              & \textbf{96.0} (2.6)             & \textbf{96.3} (2.4)             & 79.7 (5.5)               & \textbf{96.7} (2.4)              & \textbf{96.4} (2.4)             & \textbf{96.2} (2.3)             & 95.4 (2.1)               \\
		(2) %glass                                         
		& \textbf{62.5} (6.0)              & \textbf{62.2} (3.6)             & \textbf{62.5} (3.9)             & 52.8 (4.6)
		              & \textbf{69.5} (4.2)              & 66.8 (4.3)             & \textbf{69.4} (4.8)             & \textbf{69.2} (4.4)               \\
		(3) %redwine                                       
		& \textbf{58.8} (2.0)              & \textbf{59.1} (1.9)             & 56.6 (2.0)             & 57.7 (1.7)               & 63.3 (1.8)              & 64.2 (2.0)             & 64.2 (1.9)             &  \textbf{64.7} (2.1)        \\
		(4) %ecoli                                     
		& \textbf{86.2} (2.2)              & 85.7 (2.5)             & \textbf{85.8} (2.3)             & 74.1 (3.3)               & \textbf{86.0} (2.7)              & 84.9 (2.4)             & \textbf{85.6} (2.4)              & \textbf{86.0} (2.5)  \\
		(5) %vehicle                                       
		& \textbf{78.8} (2.2)              & \textbf{78.8} (1.7)             & \textbf{78.4} (2.3)             & 69.8 (3.7)               & \textbf{84.3} (2.5)              & \textbf{84.4} (2.6)             & 83.8 (2.3)             & \textbf{84.4} (2.6)               \\
		(6) %segment                                       
		& 94.9 (0.7)              & 94.9 (0.8)             & \textbf{95.2} (0.8)             & 75.8 (1.5)               & \textbf{96.5} (0.6)              & \textbf{96.6} (0.5)             & 96.3 (0.6)             &  \textbf{96.4} (0.5)               \\
		(7) %sat                                           
		& 84.9 (0.7)              & \textbf{85.4} (0.7)             & 84.7 (0.7)             & 74.9 (0.9)               & \textbf{91.9} (0.5)              & \textbf{92.0} (0.6)             & \textbf{91.9} (0.5)             & \textbf{91.9} (0.4)               \\
		(8) %optdigits                                     
		& \textbf{96.6} (0.6)              & \textbf{96.5} (0.7)             & 96.3 (0.6)             & 76.2 (2.2)               & 98.7 (0.4)              & \textbf{98.8} (0.4)             & \textbf{98.8} (0.3)             &  \textbf{98.9} (0.3)               \\
		(9) %pageblocks                                    
		& 96.0 (0.5)              & 96.1 (0.5)             & \textbf{96.3} (0.5)             & 92.5 (0.8)               & \textbf{96.8} (0.5)              & 96.6 (0.4)             & 96.7 (0.4)             & 96.6 (0.4)               \\
		(10) %libras                                        
		& \textbf{74.1} (3.3)              & 72.0 (3.8)             & 71.3 (4.3)             & 34.0 (6.4)               & 83.6 (3.8)              & 83.8 (3.4)             & \textbf{85.0} (3.9)             & 83.2 (4.2)               \\
		(11) %vertebral                                     
		& \textbf{85.5} (2.9)              & \textbf{85.9} (2.7)             & \textbf{85.4} (3.3)             & 79.8 (5.6)               & \textbf{86.0} (3.1)              & \textbf{85.3} (2.9)             & 85.5 (3.3)             & 84.4 (2.7)               \\
		(12) %breasttissue                                  
		& \textbf{64.4} (7.1)              & 59.7 (7.8)             & \textbf{66.3} (6.9)             & 58.3 (8.1)               & \textbf{68.4} (8.6)              & \textbf{68.1} (6.5)             & \textbf{66.6} (8.9)             & \textbf{68.0} (7.2)               \\ \midrule \midrule
		{\footnotesize avg}
		& {\small 81.59} \hspace{0.2cm}             & {\small 81.02} \hspace{0.2cm}            & {\small 81.25} \hspace{0.2cm}            & {\small 68.80} \hspace{0.2cm}              & {\small 85.14} \hspace{0.2cm}             & {\small 84.82} \hspace{0.2cm}            & {\small 85.00} \hspace{0.2cm}            & {\small 84.93} \hspace{0.2cm}              \\ 
		{\footnotesize \#b} 
		& {\small 9} \hspace{0.5cm}             & {\small 7} \hspace{0.5cm}            & {\small 8} \hspace{0.5cm}            & {\small 0} \hspace{0.5cm}              & {\small 9} \hspace{0.5cm}             & {\small 7} \hspace{0.5cm}            & {\small 7} \hspace{0.5cm}            & {\small 8} \hspace{0.5cm}                \\ \bottomrule
		
	\end{tabular}
    \caption{The mean and (in parentheses) standard deviation of the accuracy for each model with linear kernel and Gaussian kernel feature representations. Bold numbers in each case indicate that the result is the best or not significantly worse than the best (Wilcoxon signed-rank test with $\alpha = 0.05$).}
	\label{tbl:zo-result}
\end{table}

We report the accuracy of each method averaged over the $20$ dataset splits 
for both linear feature representations and Gaussian kernel feature 
representations in Table \ref{tbl:zo-result}.
%, with the standard deviation shown in the parentheses. 
We denote the results that are either the best of all four methods or not 
worse than the best with statistical significance (under the non-parametric Wilcoxon signed-rank test with $\alpha = 0.05$) using bold font. 
We also show the accuracy averaged over all of the datasets for each method 
and the number of dataset for which each method is 
``indistinguishably best'' (bold numbers) in the last row. 
As we can see from the table, the only alternative model that is Fisher consistent---the LLW model---performs poorly on all datasets 
when only linear features are employed. This matches with previous experimental results conducted by \citet{dogan2016unified} and demonstrates a weakness of
using an absolute margin for the loss function (rather than the relative
margins of all other methods). 
The AL$^{\text{0-1}}$ classifier performs competitively with the WW and CS models 
with a 
slight advantages on overall average accuracy and a larger number of 
``indistinguishably best'' performances on datasets---or, equivalently,
fewer statistically significant losses to any other method. 

The kernel trick in the Gaussian kernel case 
provides access to much richer feature spaces, improving the performance of 
all models, and the LLW model especially. 
%The LLW model gains a lot of improvements compared with the result in linear case. 
In general, all models provide competitive results in the Gaussian kernel case.
The AL$^{\text{0-1}}$ classifier maintains a similarly slight advantage 
%from
%the linear feature representation setting 
and only provides performance
that is sub-optimal (with statistical significance) in three of the twelve
datasets versus six of twelve and five of twelve for the other methods.
We conclude that the multiclass adversarial method performs well in both low and high dimensional feature spaces. Recalling the theoretical analysis of the adversarial method, it is a well-motivated (from the adversarial zero-one loss minimization) multiclass classifier that enjoys both strong theoretical properties (Fisher consistency) and empirical performance.

\subsection{Experiments for Multiclass Ordinal Classification}

% \subsubsection{Experiment Setup}

We conduct our ordinal classification experiments on a benchmark dataset for ordinal regression \citep{chu2005gaussian}, evaluate the performance using mean absolute error (MAE), and perform statistical tests on the results of different hinge loss surrogate 
methods. The benchmark contains %12 
datasets taken from the UCI machine learning repository \citep{lichman2013UCIML}, which range from relatively small to relatively large datasets. The characteristic of the datasets, i.e., the number of classes, the training set size, the testing set size, and the number of features is described in Table \ref{table:dataset}. 

\begin{table}[ht]
% \vspace{-2mm
%\vskip 0.15in
\small
\centering
\begin{tabular}{lrrrr}
\toprule
% \abovespace\belowspace
Dataset & \!\!\!\!\!\!\!\!\!\#class & \!\!\!\#train & \!\!\!\#test & \!\!\!\#features \\
\midrule \midrule
% \abovespace
diabetes & 5        & 30       & 13      & 2          \\
pyrimidines & 5        & 51       & 23      & 27         \\
triazines & 5        & 130      & 56      & 60         \\
wisconsin & 5        & 135      & 59      & 32         \\
machinecpu & 10       & 146      & 63      & 6          \\
autompg & 10       & 274      & 118     & 7          \\
boston & 5        & 354      & 152     & 13         \\
stocks & 5        & 665      & 285     & 9          \\
abalone & 10       & 2923     & 1254    & 10         \\
bank & 10       & 5734     & 2458    & 8          \\
% bank2 (ba2) & 10       & 5734     & 2458    & 32         \\
% computer1 (co1) & 10       & 5734     & 2458    & 12         \\
computer  & 10       & 5734     & 2458    & 21         \\
calhousing  & 10       & 14447    & 6193    & 8          \\
% census1 (ce1) & 10       & 15948    & 6836    & 8          \\
% \belowspace
% census2 (ce2) & 10       & 15948    & 6836    & 16        \\
\bottomrule
\end{tabular}
\caption{Properties of the datasets for the ordinal classification experiments.}
\label{table:dataset}
\end{table}

In the experiment, we consider the methods using the original feature space and using a Gaussian radial basis function kernel feature space.
The methods that we compare include two variations of our approach, the threshold based ({$\text{AL}^{\text{ord-th}}$}), and the multiclass-based ({$\text{AL}^{\text{ord-mc}}$}). The baselines we use for the threshold-based models include an SVM-based reduction framework algorithm ({$\text{RED}^{\text{th}}$}) \citep{li2007ordinal}, the \emph{all threshold} method with hinge loss ({AT}) \citep{shashua2003ranking,chu2005new}, and the \emph{immediate threshold} method with hinge loss ({IT}) \citep{shashua2003ranking,chu2005new}. For the multiclass-based models, we compare our method with an SVM-based reduction framework algorithm using multiclass features ({$\text{RED}^{\text{mc}}$})  \citep{li2007ordinal}, cost-sensitive one-sided support vector regression ({CSOSR}) \citep{tu2010one}, cost-sensitive one-versus-one SVM ({CSOVO}) \citep{lin2014reduction}, and cost-sensitive one-versus-all SVM ({CSOVA}) \citep{lin2008ordinal}. For our Gaussian kernel experiment, we compare our threshold-based model ({$\text{AL}^{\text{ord-th}}$}) with {SVORIM} and {SVOREX} \citep{chu2005new}.

In our experiments, 
we performed two stages of five-fold cross validation on the training set of a random split of the dataset to tune each model's regularization constant $\lambda$. 
In the first stage, the possible values for $\lambda$ are $2^{-i}, i=\{1,3,5,7,9,11,13\}$. Using the best $\lambda$ in the first stage, we set the possible values for $\lambda$ in the second stage as  $2^{\frac{i}{2}} \lambda_0, i=\{-3,-2,-1,0,1,2,3\}$, where $\lambda_0$ 
is the best parameter obtained in the first stage.
We then create 20 independent random splits of each dataset into training and testing sets.
Using the selected parameter, we train each model on the 20 training sets and evaluate the MAE performance on the corresponding testing set. We then perform a statistical test to find whether the performance of a model is different with statistical significance from other models.
Similarly, we perform the Gaussian kernel experiments with the same model parameter settings as in the multiclass zero-one experiments.

We report the mean absolute error (MAE) averaged over the dataset splits as shown in Table \ref{table:result-linear} and Table \ref{table:result-kernel}. We highlight the results that are either the best or not worse than the best with statistical significance (under the non-parametric Wilcoxon signed-rank test with $\alpha = 0.05$) in boldface font.
We also provide the summary for each model in terms of the averaged MAE over all datasets and the number of datasets for which each model marked with boldface font in the bottom of the table.

\begin{table}[ht]
\centering
\footnotesize
\begin{tabular}{@{}lccccccccc@{}}
\toprule
\multicolumn{1}{c}{\multirow{2}{*}{Dataset}} & \multicolumn{4}{c}{Threshold-based models}                                                                       & \multicolumn{5}{c}{Multiclass-based models}                                                                 \\ \cmidrule(r){2-5} \cmidrule(l){6-10}
\multicolumn{1}{c}{}                         & \multicolumn{1}{c}{$\text{AL}^{\text{ord-th}}$} & \multicolumn{1}{c}{$\text{RED}^{\text{th}}$} & \multicolumn{1}{c}{AT} & \multicolumn{1}{c}{IT} & \multicolumn{1}{c}{$\text{AL}^{\text{ord-mc}}$} & \multicolumn{1}{c}{$\text{RED}^{\text{mc}}$} & \multicolumn{1}{c}{CSOSR} & \multicolumn{1}{c}{CSOVO} & \multicolumn{1}{c}{CSOVA}\\ \midrule \midrule
% \hline
% dataset           &  $\text{AL}^{\text{ord-th}}$         & Ext TH         & AT hinge & IT hinge & $\text{AL}^{\text{ord-mc}}$        & Ext MC         & CSOSR          & CSOVO          \\ \hline
diabetes         & \makecell{\textbf{0.696} \\ (0.13) } & \makecell{\textbf{0.715} \\ (0.19)} & \makecell{0.731 \\ (0.15)}    & \makecell{0.827 \\ (0.28)}    & \makecell{\textbf{0.692} \\ (0.14)} & \makecell{\textbf{0.700} \\ (0.15)} & \makecell{\textbf{0.715} \\ (0.19)} & \makecell{0.738 \\ (0.16)} & \makecell{\textbf{0.762} \\ (0.19)}         \\ 
pyrimidines        & \makecell{0.654 \\ (0.12)}          & \makecell{0.678 \\ (0.15)}         & \makecell{0.615 \\ (0.3)}    & \makecell{0.626 \\ (0.14)}    & \makecell{\textbf{0.509} \\ (0.12)} & \makecell{0.565 \\ (0.13)}          & \makecell{\textbf{0.520} \\ (0.13)} & \makecell{0.576 \\ (0.16)} & \makecell{\textbf{0.526} \\ (0.16)}          \\
triazines  & \makecell{\textbf{0.607} \\ (0.09)} & \makecell{0.683\\(0.11)}          & \makecell{0.649\\(0.11)}    & \makecell{0.654\\(0.12)}    & \makecell{0.670\\(0.09)}          & \makecell{0.673\\(0.11)}          & \makecell{0.677\\(0.10)}          & \makecell{0.738\\(0.10)} & \makecell{0.732\\(0.10)}         \\
wisconsin         & \makecell{\textbf{1.077}\\(0.11)} & \makecell{\textbf{1.067}\\(0.12)} & \makecell{\textbf{1.097}\\(0.11)}    & \makecell{1.175\\(0.14)}    & \makecell{1.136\\(0.11)}          & \makecell{1.141\\(0.10)}          & \makecell{1.208\\(0.12)}          & \makecell{1.275\\(0.15)} & \makecell{1.338\\(0.11)}         \\
machinecpu        & \makecell{\textbf{0.449}\\(0.09)} & \makecell{\textbf{0.456}\\(0.09)} & \makecell{\textbf{0.458}\\(0.09)}    & \makecell{\textbf{0.467}\\(0.10)}    & \makecell{0.518\\(0.11)}          & \makecell{0.515\\(0.10)}          & \makecell{0.646\\(0.10)}          & \makecell{0.602\\(0.09)} & \makecell{0.702\\(0.14)}         \\
autompg           & \makecell{\textbf{0.551}\\(0.06)} & \makecell{\textbf{0.550}\\(0.06)} & \makecell{\textbf{0.550}\\(0.06)}    & \makecell{0.617\\(0.07)}    & \makecell{0.599\\(0.06)}          & \makecell{0.602\\(0.06)}          & \makecell{0.741\\(0.07)}          & \makecell{0.598\\(0.06)}   & \makecell{0.731\\(0.07)}       \\
boston     & \makecell{0.316\\(0.03)}          & \makecell{\textbf{0.304}\\(0.03)} & \makecell{\textbf{0.306}\\(0.03)}    & \makecell{\textbf{0.298}\\(0.04)}    & \makecell{0.311\\(0.03)} & \makecell{0.311\\(0.04)}          & \makecell{0.353\\(0.05)}          & \makecell{\textbf{0.294}\\(0.04)} & \makecell{0.363\\(0.04)} \\
stocks            & \makecell{0.324\\(0.02)}          & \makecell{0.317\\(0.02)}          & \makecell{0.315\\(0.02)}    & \makecell{0.324\\(0.02)}    & \makecell{0.168\\(0.02)} & \makecell{0.175\\(0.03)} & \makecell{0.204\\(0.02)}          & \makecell{\textbf{0.147}\\(0.02)} & \makecell{0.213\\(0.02)} \\
abalone           & \makecell{0.551\\(0.02)}          & \makecell{0.547\\(0.02)}          & \makecell{0.546\\(0.02)}    & \makecell{0.571\\(0.02)}    & \makecell{\textbf{0.521}\\(0.02)} & \makecell{\textbf{0.520}\\(0.02)} & \makecell{0.545\\(0.02)}          & \makecell{0.558\\(0.02)}    & \makecell{0.556\\(0.02)}      \\
bank      & \makecell{0.461\\(0.01)}          & \makecell{0.460\\(0.01)}          & \makecell{0.461\\(0.01)}    & \makecell{0.461\\(0.01)}    & \makecell{\textbf{0.445}\\(0.01)}          & \makecell{\textbf{0.446}\\(0.01)} & \makecell{0.732\\(0.02)}          & \makecell{0.448\\(0.01)}   & \makecell{0.989\\(0.02)}       \\
computer & \makecell{0.640\\(0.02)}          & \makecell{0.635\\(0.02)}          & \makecell{0.633\\(0.02)}    & \makecell{0.683\\(0.02)}    & \makecell{\textbf{0.625}\\(0.01)}          & \makecell{\textbf{0.624}\\(0.02)}          & \makecell{0.889\\(0.02)}          & \makecell{0.649\\(0.02)}    & \makecell{1.055\\(0.02)}      \\
calhousing & \makecell{1.190\\(0.01)}          & \makecell{1.183\\(0.01)}          & \makecell{1.182\\(0.01)}    & \makecell{1.225\\(0.01)}    & \makecell{1.164\\(0.01)}          & \makecell{\textbf{1.144}\\(0.01)}          & \makecell{1.237\\(0.01)}          & \makecell{1.202\\(0.01)}       & \makecell{1.601\\(0.02)}   \\ \midrule \midrule
{\small average}           & {\small 0.626}          & {\small 0.633}          & {\small 0.629}    & {\small 0.661}    & {\small 0.613}          & {\small 0.618}          & {\small 0.706}          & {\small 0.652}     & {\small 0.797}    \\
\# bold               & 5              & 5              & 4        & 2        & 5              & 5              & 2              & 2      & 2        \\ \bottomrule
\end{tabular}
\caption{The average and (in parenthesis) standard deviation of the mean absolute error (MAE) for each model. Bold numbers in each case indicate that the result is the best or not significantly worse than the best (Wilcoxon signed-rank test with $\alpha = 0.05$).}
\label{table:result-linear}

\end{table}

As we can see from Table \ref{table:result-linear}, in the experiment with the original feature space, threshold-based models perform well on relatively small datasets, whereas multiclass-based models perform well on relatively large datasets. A possible explanation for this result is that multiclass-based models have more flexibility in creating decision boundaries, hence perform better if the training data size is sufficient. However, since multiclass-based models have many more parameters than threshold-based models ($m k$ parameters rather than $m + k-1$ parameters), multiclass methods may need more data, and hence, may not perform well on relatively small datasets.

In the threshold-based models comparison, {$\text{AL}^{\text{ord-th}}$}, {$\text{RED}^{\text{th}}$}, and {AT} perform competitively on relatively small datasets like \texttt{triazines}, \texttt{wisconsin}, \texttt{machinecpu}, and \texttt{autompg}. {$\text{AL}^{\text{ord-th}}$} has a slight advantage over {$\text{RED}^{\text{th}}$} on the overall accuracy, and a slight advantage over {AT} on the number of ``indistinguishably best'' performance on all datasets. We can also see that {AT} is superior to {IT} in the experiments under the original feature space. 
Among the multiclass-based models, {$\text{AL}^{\text{ord-mc}}$} and {$\text{RED}^{\text{mc}}$} perform competitively on datasets like \texttt{abalone}, \texttt{bank}, and \texttt{computer}, with a slight advantage of {$\text{AL}^{\text{ord-mc}}$} model on the overall accuracy. In general, the cost-sensitive models perform poorly compared with {$\text{AL}^{\text{ord-mc}}$} and {$\text{RED}^{\text{mc}}$}. A notable exception is the {CSOVO} model which perform very well on the \texttt{stocks}, and \texttt{boston} datasets.

\begin{table}[ht]
% \vspace{-5mm}
\centering
\small
\begin{tabular}{lccc}
\toprule
Dataset     & $\text{AL}^{\text{ord-th}}$         & SVORIM         & SVOREX         \\ \midrule \midrule
diabetes    & \textbf{0.696} (0.13) & \textbf{0.665} (0.14) & \textbf{0.688} (0.18)  \\
pyrimidines & \textbf{0.478} (0.11) & 0.539 (0.11)          & 0.550 (0.11)          \\
triazines   & \textbf{0.608} (0.08) & \textbf{0.612} (0.09) & \textbf{0.604} (0.08) \\
wisconsin   & \textbf{1.090} (0.10)          & 1.113 (0.12)         & \textbf{1.049} (0.09) \\
machinecpu  & \textbf{0.452} (0.09) & 0.652 (0.12)         & 0.628 (0.13)         \\
autompg     & \textbf{0.529} (0.04) & 0.589 (0.05)         & 0.593 (0.05)         \\
boston      & \textbf{0.278} (0.04) & 0.324 (0.03)         & 0.316 (0.03)          \\
stocks      & \textbf{0.103} (0.02) & \textbf{0.099} (0.01) & \textbf{0.100} (0.02) \\ \midrule \midrule
average         & 0.531          & 0.574          & 0.566          \\
\# bold     & 8              & 3              & 4              \\ \bottomrule
\end{tabular}
\caption{The mean and (in parenthesis) standard deviation of the MAE for models with Gaussian kernel. Bold numbers in each case indicate that the result is the best or not significantly worse than the best (Wilcoxon signed-rank test with $\alpha = 0.05$).}
\label{table:result-kernel}
\end{table}

In the Gaussian kernel experiment, we can see from Table \ref{table:result-kernel} that the kernelized version of {$\text{AL}^{\text{ord-th}}$} performs significantly better than the threshold-based models SVORIM and SVOREX in terms of both the overall accuracy and the number of ``indistinguishably best'' performance on all datasets. We also note that immediate-threshold-based model (SVOREX) performs better than all-threshold-based model (SVORIM) in our experiment using Gaussian kernel.
We can conclude that our proposed adversarial losses for ordinal regression perform competitively compared to the state-of-the-art ordinal regression models using both original feature spaces and kernel feature spaces with a significant performance improvement in the Gaussian kernel experiments.

\subsection{Experiments for the Classification with Abstention}

We conduct experiments for classification with abstention tasks using the same dataset as in the multiclass zero-one experiments (Table \ref{tbl:zo-dataset}). 
We compare the performance of our adversarial surrogate loss (AL$^{\text{abstain}}$) with the SVM's one-vs-all (OVA) and Crammer \& Singer (CS) formulations for classification with abstention \citep{ramaswamy2018consistent}.
We evaluate the prediction performance for a $k$-class classification using the abstention loss: 
\begin{align}
    \text{loss}(\hat{y},y) = 
    \begin{cases} 
    \alpha & \hat{y} = k+1 \\
    I(\hat{y} \neq y) & \text{otherwise},
    \end{cases}
\end{align}
where $\hat{y} = k+1$ indicates an abstain prediction, and $\alpha$ is a fixed value for the penalty for making abstain prediction. Throughout the experiments, we use the standard value of $\alpha = \frac12$.

Similar to the setup in the previous experiments, we perform two-stage, five-fold cross validation on the training set of a random split of the dataset to tune each model's parameter ($C$ or $\lambda$) and the kernel parameter $\gamma$ under the kernelized formulation. 
We then make 20 random splits of each dataset into training and testing sets. 
Using the selected parameters, we train each model on the 20 training sets and evaluate the performance on the corresponding testing set. 
In the prediction step, we use a non-probabilistic prediction scheme for AL$^{\text{abstain}}$ as presented in Corollary \ref{col:abstain}.
For the baseline methods, we use a threshold base prediction scheme as presented in \citep{ramaswamy2018consistent} with the default value of the threshold $\tau$ for each model  ($\tau = 0.5$ for the SVM-CS, and $\tau = 0$ for the SVM-OVA).

\begin{table}[ht]
        \footnotesize
	\centering

	\begin{tabular}{@{}lrrrrrrrr@{}}
		\toprule 
		\multicolumn{1}{c}{\multirow{2}{*}{Dataset}} & \multicolumn{3}{c}{Linear Kernel}                                                                       & \multicolumn{3}{c}{Gaussian Kernel}                                                                 \\ \cmidrule(r){2-4} \cmidrule(l){5-7}
		\multicolumn{1}{c}{}                         & \multicolumn{1}{c}{AL$^{\text{abstain}}$} & \multicolumn{1}{c}{OVA} & \multicolumn{1}{c}{CS} & \multicolumn{1}{c}{AL$^{\text{abstain}}$} & \multicolumn{1}{c}{OVA} & \multicolumn{1}{c}{CS}  \\ \midrule \midrule
		iris                                          
		& \makecell{\textbf{0.037} (0.02) \\ {[7\%]} }              & \makecell{0.122 (0.04) \\ {[13\%]}}            & \makecell{\textbf{0.038} (0.02) \\ {[6\%]}}            & \makecell{\textbf{0.051} (0.03) \\ {[6\%]} }              & \makecell{0.120 (0.04) \\ {[14\%]} }              & \makecell{\textbf{0.043} (0.03) \\ {[1\%]} }                \\
		glass                                         
		& \makecell{\textbf{0.380} (0.04) \\ {[40\%]}}             & \makecell{\textbf{0.393} (0.04) \\ {[27\%]}}            & \makecell{\textbf{0.379} (0.04) \\ {[38\%]}}             & \makecell{\textbf{0.302} (0.03) \\ {[37\%]}}
		              & \makecell{0.393 (0.04) \\ {[35\%]}}              & \makecell{\textbf{0.317} (0.03) \\ {[25\%]} }           \\
		redwine                                       
		& \makecell{\textbf{0.418} (0.01) \\ {[58\%]} }              & \makecell{0.742 (0.04) \\ {[50\%]}}             & \makecell{0.423 (0.01) \\ {[54\%]}}     & \makecell{\textbf{0.373} (0.01) \\ {[42\%]}}         & \makecell{0.742 (0.04) \\ {[50\%]}} & \makecell{0.391 (0.01) \\ {[58\%]}}   \\
		ecoli                                     
		& \makecell{\textbf{0.165} (0.02)\\ {[17\%]} }              & \makecell{0.222 (0.10) \\ {[11\%]} }            & \makecell{\textbf{0.213} (0.10) \\ {[15\%]} }      & \makecell{0.160 (0.03) \\ {[17\%]}}      & \makecell{0.221 (0.10) \\ {[11\%]}} & \makecell{\textbf{0.144} (0.02) \\ {[5\%]}} \\
		vehicle                                       
		& \makecell{\textbf{0.214} (0.02) \\ {[23\%]} }             & \makecell{0.231 (0.02) \\ {[17\%]} }            & \makecell{\textbf{0.216} (0.02) \\ {[20\%]} }      & \makecell{\textbf{0.206} (0.03) \\ {[20\%]}}     & \makecell{0.226 (0.03) \\ {[15\%]}} & \makecell{0.300 (0.02) \\ {[31\%]}} \\
		segment                                       
		& \makecell{0.061 (0.01) \\ {[7\%]}}              & \makecell{0.082 (0.01) \\ {[11\%]}}             & \makecell{\textbf{0.052} (0.01) \\ {[6\%]} }         & \makecell{\textbf{0.042} (0.01) \\ {[5\%]}}       & \makecell{0.084 (0.01) \\ {[11\%]}} & \makecell{0.102 (0.01) \\ {[13\%]}} \\
		sat                                           
		& \makecell{\textbf{0.147} (0.01) \\ {[14\%]}}             & \makecell{0.356 (0.01) \\ {[20\%]} }              & \makecell{0.337 (0.01) \\ {[14\%]} }       & \makecell{\textbf{0.094} (0.01) \\ {[9\%]}}       & \makecell{0.356 (0.01) \\ {[20\%]}} & \makecell{0.181 (0.01) \\ {[4\%]}} \\
		optdigits                                     
		& \makecell{\textbf{0.037} (0.01) \\ {[4\%]} }              & \makecell{0.045 (0.01) \\ {[5\%]} }             & \makecell{0.038 (0.01) \\ {5\%} }      & \makecell{0.062 (0.01) \\ {[12\%]}}      & \makecell{\textbf{0.051} (0.01) \\ {[5\%]}}  & \makecell{0.072 (0.01) \\ {[8\%]}} \\
		pageblocks                                    
		& \makecell{\textbf{0.040} (0.01) \\ {[3\%]} }              & \makecell{0.042 (0.01)  \\ {[1\%]} }            & \makecell{\textbf{0.045} (0.02) \\ {[4\%]} }      & \makecell{\textbf{0.037} (0.01) \\ {[4\%]}}      & \makecell{0.042 (0.01) \\ {[1\%]}}  & \makecell{0.060 (0.01) \\ {[4\%]}} \\
		libras                                        
		& \makecell{\textbf{0.260} (0.03) \\ {[36\%]} }             & \makecell{\textbf{0.253} (0.02) \\ {[36\%]} }            & \makecell{\textbf{0.253} (0.02) \\ {[36\%]} }     & \makecell{0.263 (0.02) \\ {[50\%]}}      & \makecell{0.362 (0.04) \\ {[4\%]}} & \makecell{\textbf{0.207} (0.03) \\ {[14\%]}}     \\
		vertebral                                     
		& \makecell{0.154 (0.02) \\ {[16\%]} }              & \makecell{\textbf{0.147} (0.02) \\ {[7\%]} }            & \makecell{0.159 (0.02) \\ {[14\%]} }       & \makecell{0.181 (0.02) \\ {[22\%]}}      & \makecell{\textbf{0.147} (0.03) \\ {[7\%]}} & \makecell{0.220 (0.04) \\ {[4\%]}}   \\
		breasttissue                                  
		& \makecell{\textbf{0.315} (0.04) \\ {[51\%]} }              & \makecell{\textbf{0.316} (0.05) \\ {[37\%]} }             & \makecell{\textbf{0.326} (0.06) \\ {[32\%]} }      & \makecell{\textbf{0.330} (0.04) \\ {[54\%]}}      & \makecell{\textbf{0.313} (0.06) \\ {[32\%]}} & \makecell{0.367 (0.03) \\ {[67\%]}}   \\ \midrule \midrule
		{\footnotesize average}
		& {\small 0.186} \hspace{4mm}             & {\small 0.246} \hspace{4mm}            & {\small 0.207} \hspace{4mm} & {\small 0.175} \hspace{4mm}    & {\small 0.255} \hspace{4mm}  & {\small 0.200} \hspace{4mm}              \\ 
		{\footnotesize \# bold} 
		& {\small 10} \hspace{6mm}             & {\small 4} \hspace{6mm}            & {\small 8} \hspace{6mm}  & {\small 8} \hspace{6mm}             & {\small 3} \hspace{6mm}            & {\small 4} \hspace{6mm}     \\ \bottomrule
		
	\end{tabular}
    \caption{The mean and (in parentheses) standard deviation of the abstention loss, and (in square bracket) the percentage of abstain predictions for each model with linear kernel and Gaussian kernel feature representations. Bold numbers in each case indicate that the result is the best or not significantly worse than the best (Wilcoxon signed-rank test with $\alpha = 0.05$).}
	\label{table-abstain}
\end{table}

We report the abstention loss averaged over the dataset splits as shown in Table \ref{table-abstain}. We highlight the results that are either the best or not worse than the best with statistical significance (under the non-parametric Wilcoxon signed-rank test with $\alpha = 0.05$) in boldface font.
We also report the average percentage of abstain predictions produced by each model in each dataset.
Finally, we provide the summary for each model in terms of the averaged abstention loss over all datasets and the number of datasets for which each model is marked with boldface font in the bottom of the table.

The results from Table \ref{table-abstain} indicates that all models output more abstain predictions in the case of the dataset with higher noise (i.e., bigger value of loss). 
The percentage of abstain predictions of AL$^{\text{abstain}}$, SVM-OVA, and SVM-CS are fairly similar. In some datasets like \texttt{segment} and \texttt{pageblocks}, all models output very rarely abstain, whereas in some datasets like \texttt{redwine} and \texttt{breasttissue}, some of the models abstain for more than 50\% of the total number of testing examples. The results show that this percentage does not depend on the number of classes. For example, both \texttt{redwine} and \texttt{optdigits} are 10-class classification problems. However, the percentage of abstain prediction for \texttt{optdigits} is far less than the one for \texttt{redwine}.

In the linear kernel experiments, the AL$^{\text{abstain}}$ performs best compared the baselines in terms of the overall abstention loss and the number of ``indistinguishably best'' performance, followed by SVM-CS and then SVM-OVA. 
The AL$^{\text{abstain}}$ has a slight advantage compared with the SVM-CS in most of the datasets in the linear kernel experiments except in few datasets that the AL$^{\text{abstain}}$ outperfoms the SVM-CS by significant margins. Overall, the SVM-OVA performs poorly on most datasets except in a few datasets (\texttt{libras}, \texttt{vertebral}, and \texttt{breasttissue}).

The introduction of non-linearity via the Gaussian kernel improves the performance of both AL$^{\text{abstain}}$ and SVM-CS as we see from Table \ref{table-abstain}. The AL$^{\text{abstain}}$ method maintains its advantages over the baselines in terms of the overall abstention loss and the number of ``indistinguishably best'' performances. 
We can conclude that AL$^{\text{abstain}}$ performs competitively compared to the baseline models using both original feature spaces and the Gaussian kernel feature spaces. We note that these competitive advantages do not have any drawbacks in terms of the computational cost compared to the baselines. As described in Section 3.5 and Section 4.3, the surrogate loss function and prediction rule are relatively simple and easy to compute.

\section{Conclusions}

In this paper, we proposed an adversarial prediction framework for general multiclass classification that seeks a predictor distribution that robustly optimizes non-convex and non-continuous multiclass loss metrics against the worst-case conditional label distributions (the adversarial distribution) constrained to (approximately) match the statistics of the training data.
The dual formulation of the framework 
resembles a risk minimization model with a convex surrogate loss we call \emph{the adversarial surrogate loss}.
These adversarial surrogate losses provide desirable properties of surrogate losses for multiclass classification. 
For example, in the case of multiclass zero-one classification, our surrogate loss fills the long-standing gap in multiclass classification by simultaneously: 
guaranteeing Fisher consistency, enabling computational efficiency via the kernel trick, and  providing competitive performance in practice. 
Our formulations for the ordinal classification problem provide novel consistent surrogate losses that have not previously been considered in the literature. 
Lastly, our surrogate loss for the classification with abstention problem provides a unique consistent method that is applicable to binary and multiclass problems, fast to compute, and also competitive in practice. 

In general, we showed that the adversarial surrogate losses for general multiclass classification problems enjoy the nice theoretical property of Fisher consistency. 
We also developed efficient algorithms for optimizing the surrogate losses and a way to incorporate rich feature representation via kernel tricks.
Finally, we demonstrated that the adversarial surrogate losses provide competitive performance in practice on several datasets taken from UCI machine learning repository.
We will investigate the adversarial prediction framework for more general loss metrics (e.g., multivariate loss metrics), and also for different prediction settings (e.g., active learning and multitask learning) in our future works.

% Acknowledgements should go at the end, before appendices and references

\acks{This research was supported as part of the Future of Life Institute (futureoflife.org) FLI-RFP-AI1 program, grant\#2016-158710 and by NSF grant RI-\#1526379.}

% \input{appendix.tex}

\vskip 0.2in
\setlength{\bibsep}{6pt}
\bibliography{biblio}

\end{document}